\documentclass{article}

\usepackage[final]{neurips_2022}

\usepackage[utf8]{inputenc} 
\usepackage[T1]{fontenc}    
\usepackage{url}            
\usepackage{booktabs}       
\usepackage{amsfonts}       
\usepackage{nicefrac}       
\usepackage{microtype}      
\usepackage{xcolor}         

\usepackage{pifont}

\usepackage{amsmath}
\usepackage{amssymb}
\usepackage{mathtools}
\usepackage{amsthm}
\usepackage{paralist}
\usepackage[para]{footmisc}
\usepackage{hyperref}

\definecolor{shadecolor}{RGB}{255,203,203}

\usepackage[capitalize,noabbrev]{cleveref}

\theoremstyle{plain}
\newtheorem{theorem}{Theorem}[section]
\newtheorem{proposition}[theorem]{Proposition}
\newtheorem{lemma}[theorem]{Lemma}

\theoremstyle{definition}
\newtheorem{definition}[theorem]{Definition}

\theoremstyle{remark}

\usepackage[textsize=tiny]{todonotes}

\usepackage{multirow}
\usepackage{subfig}
\usepackage{enumitem}
\usepackage{colortbl}
\usepackage{xcolor}
\usepackage{wrapfig}

\title{A Consistent and Differentiable \\$L_p$ Canonical Calibration Error Estimator}

\author{%
  Teodora Popordanoska\thanks{Equal contribution} \\
  ESAT-PSI, KU Leuven \\
  \texttt{teodora.popordanoska@kuleuven.be}
  \And
  Raphael Sayer\footnotemark[1]\protect\phantom{\footnotesize 1}\thanks{Most of this work was done while at KU Leuven}\\
  University of Tübingen \\
  \texttt{raphael.sayer@uni-tuebingen.de} 
  \And
  Matthew B. Blaschko \\
  ESAT-PSI, KU Leuven \\
  \texttt{matthew.blaschko@esat.kuleuven.be} \\
}

\begin{document}

\maketitle

\begin{abstract}
Calibrated probabilistic classifiers are models whose predicted probabilities can directly be interpreted as uncertainty estimates. It has been shown recently that deep neural networks are poorly calibrated and tend to output overconfident predictions. As a remedy, we propose a low-bias, trainable calibration error estimator based on Dirichlet kernel density estimates, which asymptotically converges to the true $L_p$ calibration error. This novel estimator enables us to tackle the strongest notion of multiclass calibration, called canonical (or distribution) calibration, while other common calibration methods are tractable only for top-label and marginal calibration. The computational complexity of our estimator is $\mathcal{O}(n^2)$, the convergence rate is $\mathcal{O}(n^{-1/2})$, and it is unbiased up to $\mathcal{O}(n^{-2})$, achieved by a geometric series debiasing scheme. In practice, this means that the estimator can be applied to small subsets of data, enabling efficient estimation and mini-batch updates. The proposed method has a natural choice of kernel, and can be used to generate consistent estimates of other quantities based on conditional expectation, such as the sharpness of a probabilistic classifier. 
Empirical results validate the correctness of our estimator, and demonstrate its utility in canonical calibration error estimation and calibration error regularized risk minimization.
\end{abstract}

\section{Introduction}

Deep neural networks have shown tremendous success in classification tasks, being regularly the best performing models in terms of accuracy. However, they are also known to make overconfident predictions \citep{guo2017}, which is particularly problematic in safety-critical applications, such as medical diagnosis \citep{esteva2017,esteva2019} or autonomous driving \citep{caesar2020,sun2020}. In many real world applications it is not only the predictive performance that is important, but also the trustworthiness of the prediction, i.e., we are interested in accurate predictions with robust uncertainty estimates. To this end, it is necessary that the models are uncertainty calibrated, which means that, for instance, among all cells that have been predicted with a probability of 0.8 to be cancerous, 80\% should indeed belong to a malignant tumor. 

The field of uncertainty calibration has been mostly focused on binary problems, often considering only the confidence score of the predicted class. However, this so called top-label (or confidence) calibration \citep{guo2017}) is often not sufficient in multiclass settings. A stronger notion of calibration is marginal (or class-wise) \citep{kull2019}, that splits up the multiclass problem into $K$ one-vs-all binary ones, and requires each to be calibrated according to the definition of binary calibration. The most strict notion of calibration, called canonical (or distribution) calibration \citep{brocker_2009, kull2015, vaicenavicius2019}, requires the whole probability vector to be calibrated. The curse of dimensionality makes estimation of this form of calibration difficult, and current estimators, such as the binned estimator $ECE^{bin}$ \citep{naeini2015}, MMCE \citep{kumar2018} and Mix-n-Match \citep{zhang2020}, have computational or statistical limitations that prevent them from being successfully applied in this important setting. Specifically, the binned estimator is sensitive to the binning scheme and is asymptotically inconsistent in many situations \cite{vaicenavicius2019}, MMCE is not a consistent estimator of $L_p$ calibration error and Mix-n-Match, although consistent, is intractable in high dimensions and the authors did not implement it in more than one dimension. 

We propose \textit{a tractable, differentiable, and consistent estimator of the expected $L_p$ canonical calibration error}. In particular, we use kernel density estimates (KDEs) with a Beta kernel in binary classification tasks and a Dirichlet kernel in the multiclass setting, as these kernels are the natural choices to model densities over a probability simplex. 
In Table \ref{tab:calibration_error_estimators}, we summarize and compare the properties of our $ECE^{KDE}$ estimator and other commonly used estimators. $ECE^{KDE}$ scales well to higher dimensions and it is able to capture canonical calibration with $\mathcal{O}(n^2)$ complexity.

\newcommand{\y}{\textcolor{green}{\ding{51}}}
\newcommand{\n}{\textcolor{red}{\ding{55} }}
\begin{table*}[t!]
    \caption{Properties of $ECE^{KDE}$ and other commonly used calibration error estimators.}
    \centering
    	\resizebox{1\textwidth}{!}{%
            \begin{tabular}{llclc}
             \bottomrule
                & \multicolumn{4}{c}{\textbf{Properties}}\\
                 &  Consistency &  Scalability &  De-biased  & Differentiable  \\
			\midrule
            \textbf{$ECE^{KDE}$ (Our)} & \y & \y  & \y  & \y \\
            $ECE^{bin}$ & \n  \citep{vaicenavicius2019} & \n & \y   \citep{pmlr-v151-roelofs22a} & \n  \\
            Mix-n-Match & \y  \citep{zhang2020} & \n & \n  & \y  \\
            MMCE & \n  \citep{kumar2018} & \y & \n  & \y   \\
            \bottomrule
            \end{tabular}
       }
    \label{tab:calibration_error_estimators}
\end{table*}

Our contributions can be summarized as follows: 
    1. We develop a tractable estimator of canonical $L_p$ calibration error that is consistent and differentiable.
    2. We demonstrate a natural choice of kernel. Due to the scaling properties of Dirichlet kernel density estimation, 
    evaluating \emph{canonical calibration} becomes feasible in cases that cannot be estimated using other methods. 
    3. We provide a second order debiasing scheme to further improve the convergence of the estimator. 
    4. We empirically evaluate the correctness of our estimator and demonstrate its utility in the task of calibration regularized risk minimization on variety of network architectures and several datasets.

\section{Related Work}
Calibration of probabilistic predictors has long been studied in many fields. 
This topic gained attention in the deep learning community since \citet{guo2017} observed that modern neural networks are poorly calibrated and tend to give overconfident predictions due to overfitting on the NLL loss. The surge of interest resulted in many calibration strategies that can be split in two general categories, which we discuss subsequently. 

\textbf{Post-hoc calibration strategies} learn a calibration map of the predictions from a trained predictor in a post-hoc manner, using a held-out calibration set. For instance, Platt scaling \citep{platt99} fits a logistic regression model on top of the logit outputs of the model. A special case of Platt scaling that fits a single scalar, called temperature, has been popularized by \citet{guo2017} as an accuracy-preserving, easy to implement and effective method to improve calibration. However, it has the undesired consequence that it clamps the high confidence scores of accurate predictions \citep{kumar2018}. 
Similar approaches for post-hoc calibration include histogram binning \citep{zadrozny2001}, isotonic regression \citep{zadrozny2002}, Bayesian binning into quantiles \citep{naeini2015b}, Beta \citep{kull2017} and Dirichlet calibration \citep{kull2019}.
Recently, \citet {gupta2021} proposed a binning-free calibration measure based on the Kolmogorov-Smirnov test. In this approach, the recalibration function is obtained via spline-fitting, rather than minimizing a loss function on a calibration set. \citet {ma2021} integrate ensamble-based and post-hoc calibration methods in an accuracy-perserving truth discovery framework. \citet{zhao2021} introduce a new notion of calibration, called decision calibration, however, they do not propose an estimator of calibration error with statistical guarnatees.

\textbf{Trainable calibration strategies} integrate a differentiable calibration measure into the training objective. One of the earliest approaches is regularization by penalizing low entropy predictions \citep{pereyra2017}. Similarly to temperature scaling, it has been shown that entropy regularization needlessly suppresses high confidence scores of correct predictions \citep{kumar2018}. Another popular strategy is MMCE (Maximum Mean Calibration Error) \citep{kumar2018}, where the entropy regularizer is replaced by a kernel-based surrogate for the calibration error that can be optimized alongside NLL. It has been shown that label smoothing \citep{szegedy2015, muller2020}, i.e. training models with a weighted mixture of the labels instead of one-hot vectors, also improves model calibration. \citet{liang2020} propose to add the difference between predicted confidence and accuracy as auxiliary term to the cross-entropy loss. Focal loss \citep{mukhoti2020, lin2018} has recently been \emph{empirically} shown to produce better calibrated models than many of the alternatives, but does not estimate a clear quantity related to calibration error. \citet{bohdal2021} derive a differentiable approximation to the commonly-used binned estimator of calibration error by computing differentiable approximations to the 0/1 loss and the binning operator. However, this approach does not eliminate the dependence on the binning scheme and it is not clear how it can be extended to calibration of the whole probability vector.

\textbf{Kernel density estimation} \citep{parzen1962,rosenblatt1956,silverman86} is a non-parametric method to estimate a probability density function from a finite sample. 
\citet{zhang2020} propose a KDE-based estimator of the calibration error  (Mix-n-Match) for measuring calibration performance. Although they demonstrate consistency of the method, it requires a numerical integration step that is infeasible in high dimensions.  In practice, they only implemented binary calibration, and not canonical calibration.

Although many calibration strategies have been empirically shown to decrease the calibration error, very few of them are based on an estimator of miscalibration. 
Our estimator is the first consistent, differentiable estimator with favourable scaling properties that has been successfully applied to the estimation of $L_p$ canonical calibration error in the multi-class setting.

\section{Methods}
We study a classical supervised classification problem. Let $(\Omega, \mathcal{A}, \mathbb{P})$ be a probability space, where $\Omega$ is the set of possible outcomes, $\mathcal{A}=\mathcal{A}(\Omega)$ is the sigma field of events and $\mathbb{P}:\mathcal{A}\rightarrow [0,1]$ is a probability measure, let $\mathcal{X}=\mathbb{R}^d$ and $\mathcal{Y}=\{1, ..., K\}$. Let $x: \Omega \rightarrow \mathcal{X}$ and $y: \Omega \rightarrow \mathcal{Y}$ be random variables, while realizations are denoted with subscripts. Suppose we have a model $f:\mathcal{X} \rightarrow \triangle^K$, where $\triangle^K$ denotes the $K-1$ dimensional simplex as obtained, e.g.,\ from the output of a final softmax layer in a neural network. We measure the (mis-)calibration in terms of the $L_p$ calibration error, defined below. 
\begin{definition}[Calibration error, \citep{naeini2015,kumar2019,wenger2020}]
    \label{def:calibration-err}
    The $L_p$ calibration error of $f$ is:
    \begin{align}\label{eq:LPcalibrationError}
        \operatorname{CE}_p(f) = \biggl(\mathbb{E}\biggl[\Bigl\| \mathbb{E}[y \mid f(x)]-f(x)\Bigr\|_p^p\biggr] \biggr)^{\frac{1}{p}}.
    \end{align}
\end{definition}
We note that we consider multiclass calibration, and that $f(x)$ and the conditional expectation in \Cref{eq:LPcalibrationError} therefore map to points on a probability simplex.
We say that a classifier $f$ is perfectly calibrated if $\operatorname{CE}_p(f) = 0$.

In order to empirically compute the conditional expectation in \Cref{eq:LPcalibrationError}, we need to perform density estimation over the probability simplex. In a binary setting, this has traditionally been done with binned estimates \citep{naeini2015,guo2017,kumar2019}. However, this is not differentiable w.r.t.\ the function $f$, and cannot be incorporated into a gradient based training procedure.  Furthermore, binned estimates suffer from the curse of dimensionality and do not have a practical extension to multiclass settings. We consider an estimator for the $\operatorname{CE}_p$ based on Beta and Dirichlet kernel density estimates in the binary and multiclass setting, respectively. We require that this estimator is consistent and differentiable such that we can train it in a calibration error regularized risk minimization framework. This estimator is given by:
\begin{align}
    \widehat{\operatorname{CE}_p(f)^p} = \frac{1}{n}\sum_{j=1}^n\biggl[\Bigl\| \widehat{\mathbb{E}[y \mid f(x)]}\Bigr|_{f(x_j)}-f(x_j)\Bigr\|_p^p\biggr] ,
\end{align}
where $\widehat{\mathbb{E}[y \mid f(x)]}\Bigr|_{f(x_j)}$ denotes $\widehat{\mathbb{E}[y \mid f(x)]}$ evaluated at $f(x)=f(x_j)$.
If 
probability density $p_{x, y}$ is measurable
with respect to the product of the Lebesgue and counting measure, we can define: $p_{x, y}(x_i, y_i) =p_{y|x=x_i}(y_i) \, p_x(x_i)$. Then we define the estimator of the conditional expectation as follows:
\begin{align}
    \mathbb{E}[y\mid f(x)]&= \sum_{y_k \in \mathcal{Y}} y_k \, p_{y|f(x)}(y_k) = \frac{\sum_{y_k \in \mathcal{Y}} y_k \, p_{f(x), y}(f(x), y_k)}{p_{f(x)}(f(x))} \nonumber \\ &\approx  \frac{\sum_{i=1}^n k(f(x) ; f(x_i))y_i}{\sum_{i=1}^n k(f(x) ; f(x_i))} =: \widehat{\mathbb{E}[y \mid f(x)]}
    \label{eq:estimator_EYfX}
\end{align}
where $k$ is the kernel of a kernel density estimate evaluated at point $x$ and $p_{f(x)}$ is uniquely determined by $p_x$ and $f$.

\begin{proposition}
\label{prop:Eyfx_consistent}
    Assuming that $p_{f(x)}(f(x))$ is Lipschitz continuous over the interior of the simplex, there exists a kernel $k$ such that $\widehat{\mathbb{E}[y \mid f(x)]}$ is a pointwise consistent estimator of $\mathbb{E}[y\mid f(x)]$, 
    that is:
    \begin{align}
        \underset{n\to\infty}{\operatorname{plim}} 
        \frac{\sum_{i=1}^n k(f(x) ; f(x_i))y_i}{\sum_{i=1}^n k(f(x) ; f(x_i))} = \frac{\sum_{y_k \in \mathcal{Y}} y_k \, p_{f(x), y}(f(x), y_k)}{p_{f(x)}(f(x))}.
        \label{eq:prop_cond_exp}
    \end{align}
\end{proposition}
\begin{proof}
    Let $k$ be a Dirichlet kernel \citep{ouimet2020}. By the consistency of the Dirichlet kernel density estimators \citep[Theorem~4]{ ouimet2020} Lipschitz continuity of the density over the simplex is a sufficient condition for uniform convergence of the kernel density estimate. This in turn implies that for a given $f$, for all $f(x)\in (0,1)$, $\frac{1}{n}\sum_{i=1}^n k(f(x) ; f(x_i))y_i\xrightarrow{p} \sum_{y_k \in \mathcal{Y}} y_k \, p_{f(x), y}(f(x), y_k)$ and $\frac{1}{n}\sum_{i=1}^n k(f(x) ; f(x_i))\xrightarrow{p} p_{f(x)}(f(x))$. 
    Let $g(x)=1/ x$, 
    then the set of discontinuities of $g$ applied to the denominator of the l.h.s.\ of \eqref{eq:prop_cond_exp} has measure zero since $\frac{1}{n} \sum_{i=1}^n k(f(x) ; f(x_i))=0$ with probability zero. From the continuous mapping theorem \citep{mann1943} it follows, that $n/(\sum_{i=1}^n k(f(x) ; f(x_i))) \xrightarrow{p} 1/p_{f(x)}(f(x))$. Since products of convergent (in probability) sequences of random variables converge in probability to the product of their limits \citep{resnick2019}, we have that $\sum_{i=1}^n k(f(x) ; f(x_i))y_i g(\sum_{i=1}^n k(f(x) ; f(x_i)))\xrightarrow{p} \sum_{y_k \in \mathcal{Y}} y_k \, p_{f(x), y}(f(x), y_k) g(p_{f(x)}(f(x)))$, which is equal to the r.h.s.\ of \eqref{eq:prop_cond_exp}.
\end{proof}

The most commonly used loss functions are designed to achieve consistency in the sense of Bayes optimality under risk minimization, however, they do not guarantee calibration - neither for finite samples nor in the asymptotic limit. Since we are interested in models $f$ that are both accurate and calibrated, we consider the following optimization problem bounding the calibration error $\operatorname{CE}(f)$:
    $f = \arg\min_{f\in \mathcal{F}}\, \operatorname{Risk}(f),  \text{s.t. } \operatorname{CE}(f) \leq B$ 
for some $B>0$, and its associated Lagrangian
\begin{align}
    f = \arg\min_{f\in \mathcal{F}}\, \Bigl(\operatorname{Risk}(f) + \lambda \cdot \operatorname{CE}(f)\Bigr).
    \label{eq:risk_calibration_min}
\end{align}

\paragraph{Mean squared error in binary classification}
As a first instantiation of our framework we consider a binary classification setting, with mean squared error $\operatorname{MSE}(f)=\mathbb{E}[(f(x)-y)^2]$ as the risk function, jointly optimized with the $L_2$ calibration error $\operatorname{CE}_2$: 

\begin{align}
    \label{eq:mse_binary}
    f = \arg\min_{f\in \mathcal{F}} \Bigl( \operatorname{MSE}(f)  +\lambda \operatorname{CE}_2(f)^2 \Bigr) 
      &= \arg\min_{f\in \mathcal{F}} \biggl( \operatorname{MSE}(f)    +\gamma \mathbb{E}\Bigl[\mathbb{E}[y\mid f(x)]^2\Bigr]\biggr)
\end{align}

where $\gamma=\frac{\lambda}{\lambda+1} \in [0,1)$. The full derivation using the MSE decomposition \citep{Murphy1973, degroot1983, kuleshov2015, nguyen2015} is given in Appendix \ref{appendix:mse_decomposition}. For optimization we wish to find an estimator for $\mathbb{E}[\mathbb{E}[y\mid f(x)]^2]$. Building upon \Cref{eq:estimator_EYfX}, a partially debiased estimator can be written as:
\begin{align}\label{eq:SharpnessSlightlyDebiased}
    \widehat{\mathbb{E}\Bigl[\mathbb{E}[y\mid f(x)]^2\Bigr]} \approx  \frac{1}{n} \sum_{j=1}^n \frac{\left(\sum_{i\neq j} \ k(f(x_j) ; f(x_i))y_i\right)^2 - \sum_{i\neq j} \left(k(f(x_j);f(x_i))y_i\right)^2 }{\left( \sum_{i\neq j} k(f(x_j) ; f(x_i)) \right)^2 - \sum_{i\neq j} \left(k(f(x_j);f(x_i))\right)^2}.
\end{align}
Thus, the conditional expectation is estimated using a ratio of unbiased estimators of the square of a mean. 
\begin{proposition}
Equation~\eqref{eq:SharpnessSlightlyDebiased} is a ratio of two U-statistics and has a bias converging as $\mathcal{O}\left(\frac{1}{n}\right)$.
\end{proposition}
    The proof is given in Appendix \ref{App:Proof_lemma}.

\begin{proposition}
    There exist de-biasing schemes for the ratios in \Cref{eq:SharpnessSlightlyDebiased} and \Cref{eq:estimator_EYfX} that achieve an improved $\mathcal{O}\left(\frac{1}{n^2}\right)$ convergence of the bias.
\end{proposition}
Proofs are given in Appendix \ref{App:L1_debiasing} and \ref{App:L2_debiasing}.

In a binary setting, the kernels $k(\cdot, \cdot)$ are Beta distributions defined as:
    \begin{align}
        k_{\operatorname{B}}(f(x_j),f(x_i)) := f(x_j)^{\alpha_i-1}(1-f(x_j))^{\beta_i-1} \frac{\operatorname{\Gamma}(\alpha_i+\beta_i)}{\operatorname{\Gamma}(\alpha_i)\operatorname{\Gamma}(\beta_i)},
    \end{align}

with $\alpha_i = \frac{f(x_i)}{h}+1$ and $\beta_i = \frac{1-f(x_i)}{h}+1$ \citep{Chen1999, bouezmarni2003, zhang2010}, where $h$ is a bandwidth parameter in the kernel density estimate that goes to 0 as $n \rightarrow \infty $.  We note that the computational complexity of this estimator is $\mathcal{O}(n^2)$. If we would use this within a gradient descent training procedure, the density can be estimated using a mini-batch and therefore the $\mathcal{O}(n^2)$ complexity is w.r.t.\ the size of a mini-batch, not the entire dataset.

The estimator in \Cref{eq:SharpnessSlightlyDebiased} is a ratio of two second order U-statistics that converge as $n^{-1/2}$ \citep{Ferguson2005a}. Therefore, the overall convergence will be $n^{-1/2}$. Empirical convergence rates are calculated in Appendix~{\ref{subsection:bias_and_convergece_rates}} and shown to be close to the theoretically expected value.

\paragraph{Multiclass calibration with Dirichlet kernel density estimates}
There are multiple definitions regarding multiclass calibration that differ in the strictness regarding the calibration of the probability vector $f(x)$.  
The strongest notion of multiclass calibration, and the one that we consider in this paper, is canonical (also called multiclass or distribution) calibration \citep{brocker_2009, kull2015, vaicenavicius2019}, which requires that the whole probability vector $f(x)$ is calibrated  (Definition~\ref{def:calibration-err}). Its estimator is:
\begin{equation}
    \label{eq:canonical_estimator}
   \widehat{\operatorname{CE}_{p}(f)^p} =  \frac{1}{n} \sum_{j=1}^n \left\| \frac{\sum_{i\neq j}  k_{\operatorname{Dir}}( f(x_j); f(x_i))y_i}{\sum_{i\neq j} k_{\operatorname{Dir}}(f(x_j) ; f(x_i))} - f(x_j) \right\|_p^p 
\end{equation}
where $k_{\operatorname{Dir}}$ is a Dirichlet kernel defined as: 
\begin{equation}
    k_{Dir}(f(x_j),f(x_i)) = \frac{\Gamma(\sum_{k=1}^K \alpha_{ik})}{\prod_{k=1}^K \Gamma(\alpha_{ik})} \prod_{k=1}^K f(x_j)_{k}^{\alpha_{ik}-1}
\end{equation}
with $\alpha_{i} = \frac{f(x_i)}{h} + 1$ \citep{ouimet2020}.  As before, the computational complexity is $\mathcal{O}(n^2)$ irrespective of $p$.

This estimator is differentiable and furthermore, the following proposition holds:

\begin{proposition}
The Dirichlet kernel based $\operatorname{CE}$ estimator is consistent when $p_{f(x)}(f(x))$ is Lipschitz:
    \begin{align}
     &\underset{n\to\infty}{\operatorname{plim}}  
    \frac{1}{n} \sum_{j=1}^n \left\| \frac{\sum_{i\neq j}^n  k_{\operatorname{Dir}}( f(x_j); f(x_i))y_i}{\sum_{i\neq j}^n k_{\operatorname{Dir}}(f(x_j) ; f(x_i))} - f(x_j) \right\|_p^p \nonumber = \mathbb{E}\biggl[ \Bigl\|\mathbb{E}[y \mid f(x)] - f(x)\Bigr\|_p^p \biggr]^p. 
    \end{align}
\end{proposition}
\begin{proof}
    Dirichlet kernel estimators are consistent when the density is Lipschitz continuous over the simplex \citep[Theorem~4]{ouimet2020}, consequently, by Proposition \ref{prop:Eyfx_consistent} the term inside the norm is consistent for any fixed $f(x_j)$ (note, that summing over $i \neq j$ ensures that the ratio of the KDE's does not depend on the outer summation). Moreover, for any convergent sequence also the norm of that sequence converges to the norm of its limit. Ultimately, the outer sum is merely the sample mean of consistent summands, which again is consistent.
\end{proof}

With this development, we have for the first time a consistent, differentiable, tractable estimator of $L_p$ canonical calibration error with $\mathcal{O}(n^2)$ computational cost and $\mathcal{O}(n^{-1/2})$ convergence rate, with a debiasing scheme that achieves $\mathcal{O}(n^{-2})$ bias for $p\in \{1,2\}$.

\section{Empirical validation of $ECE^{KDE}$}

Accurately evaluating the calibration error is a crucial step towards designing trustworthy models that can be used in societally important settings. The most widely used metric for evaluating miscalibration, and the only other estimator that can be straightforwardly extended to measure canonical calibration, is the histogram-based estimator $ECE^{bin}$. However, as discussed in \citet{vaicenavicius2019,widmann2019,ding2020,ashukha2021}, it has numerous flaws, such as: \begin{inparaenum} [(i)]
\item it is sensitive to the binning scheme
\item it is severely affected by the curse of dimensionality, as the number of bins grows exponentially with the number of classes
\item it is asymptotically inconsistent in many cases.
\end{inparaenum}

To investigate its relationship with our estimator $ECE^{KDE}$, we first introduce an extension of the top-label binned estimator to the probability simplex in the three class setting. We start by partitioning the probability simplex into equally-sized, triangle-shaped bins and assign the probability scores to the corresponding bin, as shown in Figure~{\ref{subfig:scatterplot}}. Then, we define the binned estimate of canonical calibration error as follows:
    $\operatorname{CE}_p(f)^p 
    \approx \mathbb{E}\left[\left\|H(f(x))-f(x)\right\|_p^p\right] 
    \approx  \frac{1}{n} \sum_{i=1}^n \left\|H(f(x_i))-f(x_i)\right\|_p^p$,
where $H(f(x_i))$ is the histogram estimate, shown in Figure~{\ref{subfig:histogram}}. The surface of the corresponding Dirichlet KDE is presented in Figure~{\ref{subfig:surface}}.
 See Appendix~\ref{appendix:binning_vs_kde} for 
 \begin{inparaenum}[(i)] 
 \item an experiment investigating their relationship for the three types of calibration (top-label, marginal, canonical) and with varying number of points used for the estimation, and
 \item another example of the binned estimator and Dirichlet KDE on CIFAR-10.
 \end{inparaenum}

\begin{figure}[ht!]
    \centering
    \subfloat[Splitting the simplex in 16 bins]{
    \label{subfig:scatterplot}
    \includegraphics[width=.3\linewidth]{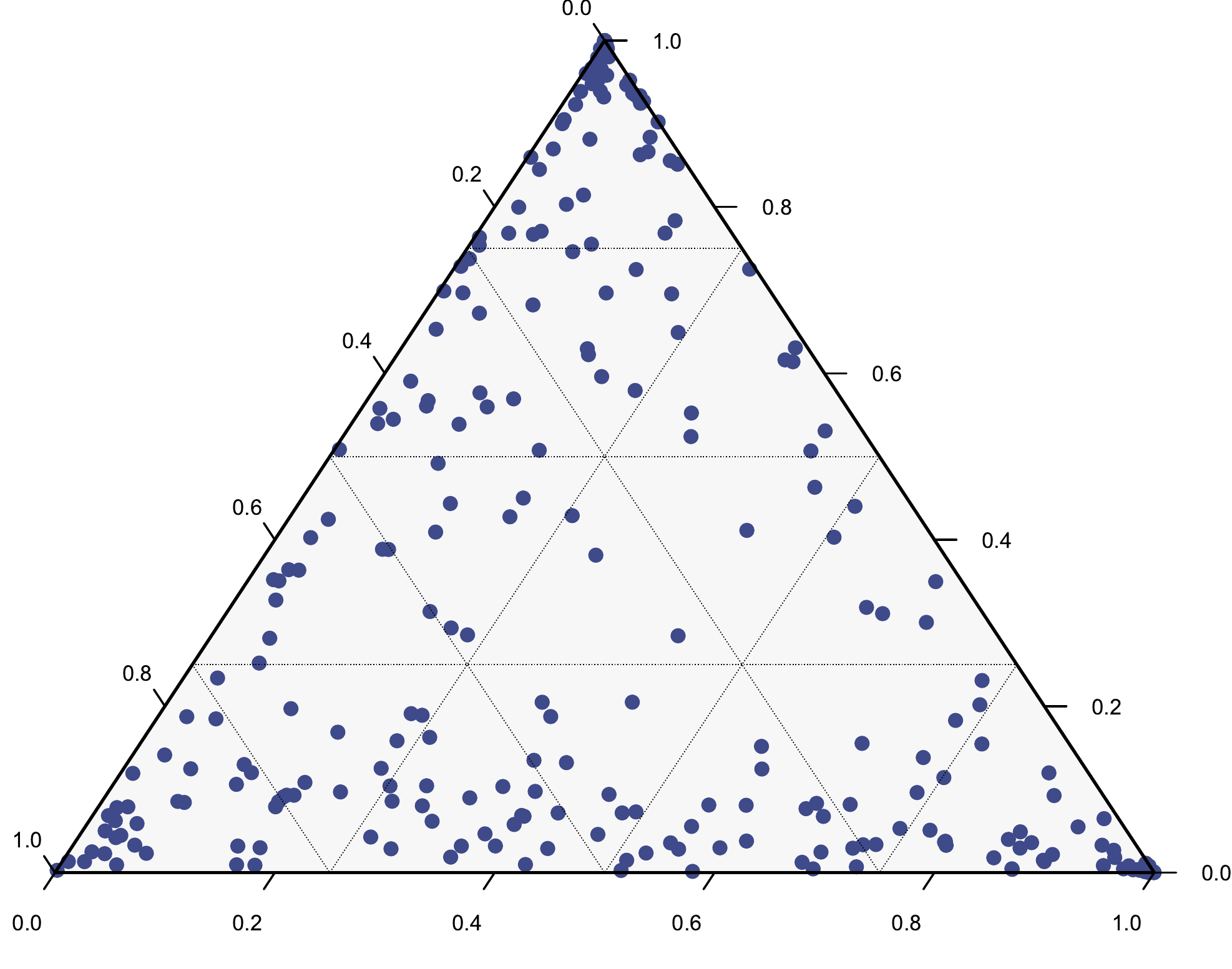}
    } \hfill
    \subfloat[Histogram]{
    \label{subfig:histogram}
    \includegraphics[width=.3\linewidth]{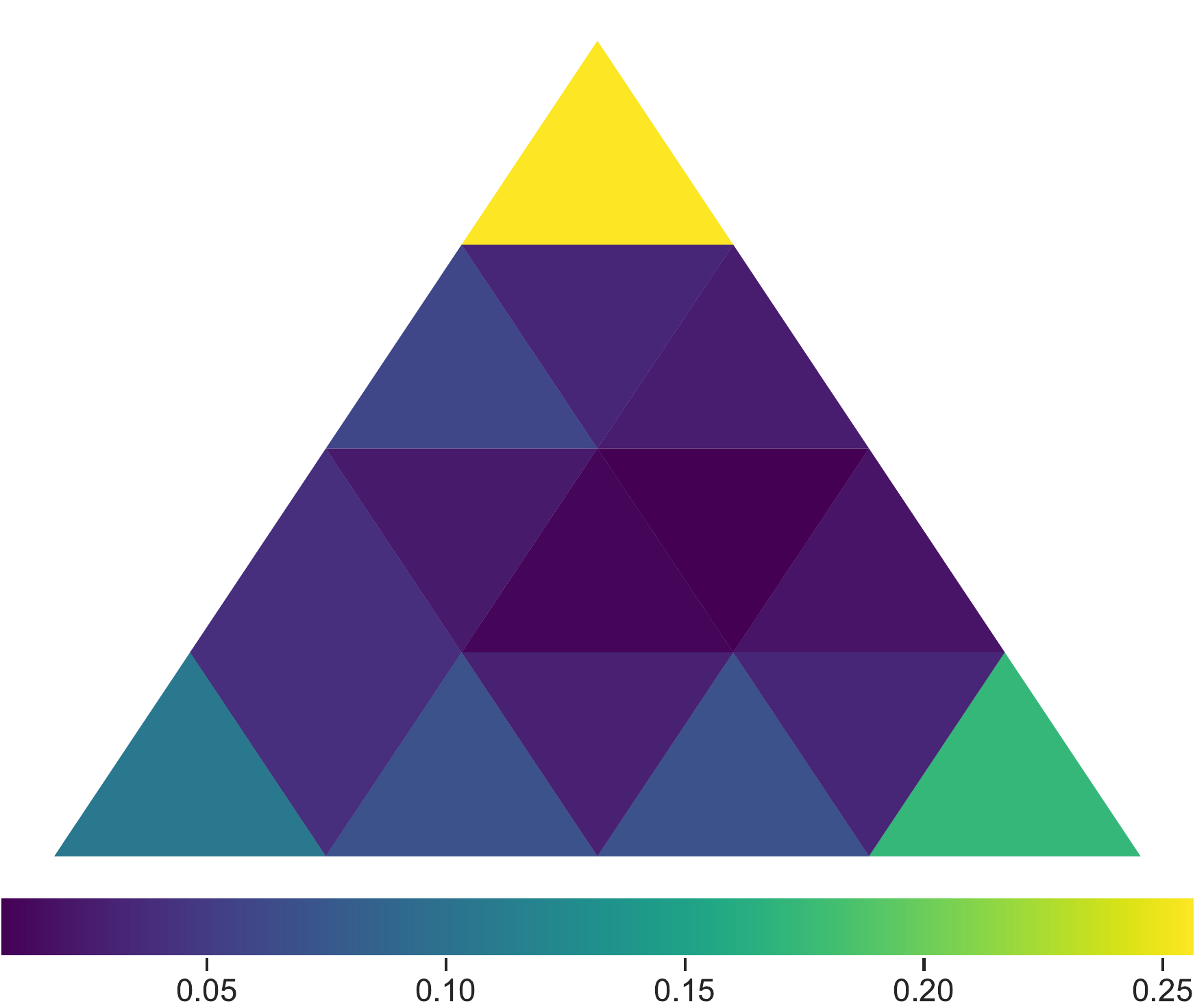} 
    } \hfill
    \subfloat[Dirichlet KDE]{
    \label{subfig:surface}
    \includegraphics[width=.3\linewidth]{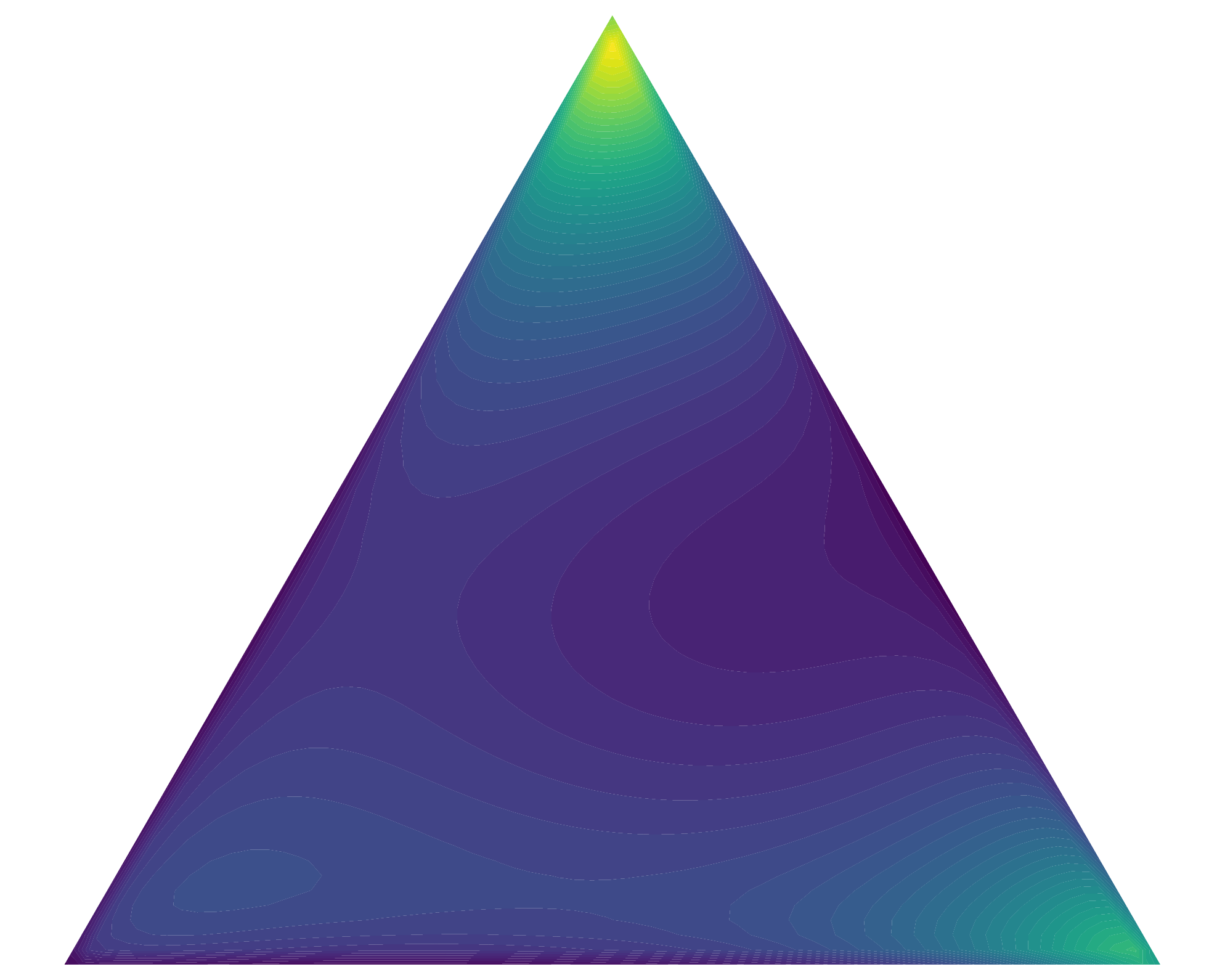}
    }
    \caption{Extension of the binned estimator $ECE^{bin}$ to the probability simplex, compared with 
    the $ECE^{KDE}$. The 
    $ECE^{KDE}$ achieves a better approximation to the finite sample, and accurately models the fact that samples tend to be concentrated near low dimensional faces of the simplex.}
    \label{fig:binned_simplex_estimator}
\end{figure}

\crefformat{subfigure}{#2#1#3}
\crefrangeformat{subfigure}{#3#1#4--#5#2#6}

\paragraph{Synthetic experiments}
We consider an extension of $ECE^{bin}$ to arbitrary number of classes and investigate its performance compared to $ECE^{KDE}$. Since on real data the ground truth calibration error is unknown, we generate synthetic data with known transformations with the following procedure. First, we sample uniformly from the simplex using the Kraemer algorithm \citep{smith2004}. Then, we apply temperature scaling with $t_1=0.6$ to simulate realistic scenarios where the probability scores are concentrated along lower dimensional faces of the simplex. We generate ground truth labels according to the sampled probabilities and therefore obtain a perfectly calibrated classifier. Subsequently, the classifier is miscalibrated by additional temperature scaling with $t_2=0.6$.
Figure \ref{fig:synthetic_experiment} depicts the performance of the two estimators as a function of the sample size on generated data for 4 and 8 classes. $ECE^{KDE}$ converges to the ground truth value obtained by integration in both cases, whereas $ECE^{bin}$ provides poor estimates even with 20000 points. \\
In another experiment with synthetic data we look at the bias of the sharpness\footnote{The sharpness is defined as $\operatorname{Var}(\mathbb{E}[y \mid f(x)])$ \citep{kuleshov2015}. Here we neglect the term that does not depend on $f(x)$, and thereby refer to ${\mathbb{E}\Bigl[\mathbb{E}[y\mid f(x)]^2\Bigr]}$ as the sharpness.} term in a binary setting. In Figure \ref{fig:debiasing_experiment} we plot the estimated value of the sharpness term for varying number of samples, both using the partially debiased ratio from Equation (\ref{eq:SharpnessSlightlyDebiased}) and the ratio debiased with the scheme introduced in Appendix \ref{App:L2_debiasing}. A sigmoidal function is applied to the calibrated data to obtain an uncalibrated sample that is used to compute the partially debiased and the fully debiased ratio of the sharpness term. The ground truth value is obtained by using 100 million samples to compute the ratio with the partially debiased version, as it converges asymptotically to the true value due to its consistency. We use a bandwidth of 0.5 and average over 10000 repetitions for each number of samples that range from 32 to 16384. We fix the location of the KDE at $f(x_j)=0.17$.
\begin{figure}[ht!]
    \centering
    \subfloat[$ECE^{bin}$ vs. $ECE^{KDE}$]{
    \includegraphics[width=0.6\linewidth]{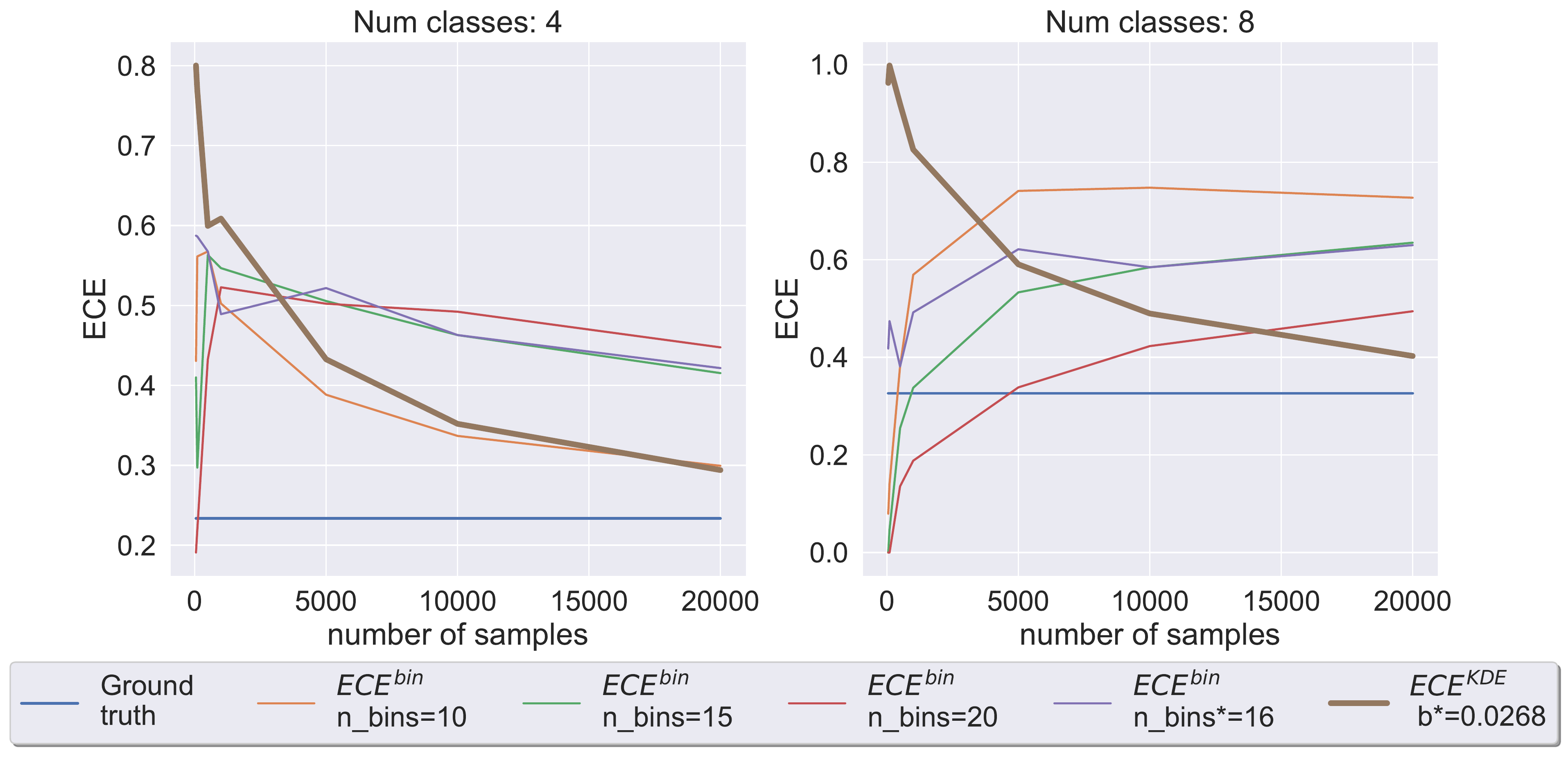}
    \label{fig:synthetic_experiment}
    }
    \subfloat[Debiasing]{
    \includegraphics[width=0.37\linewidth]{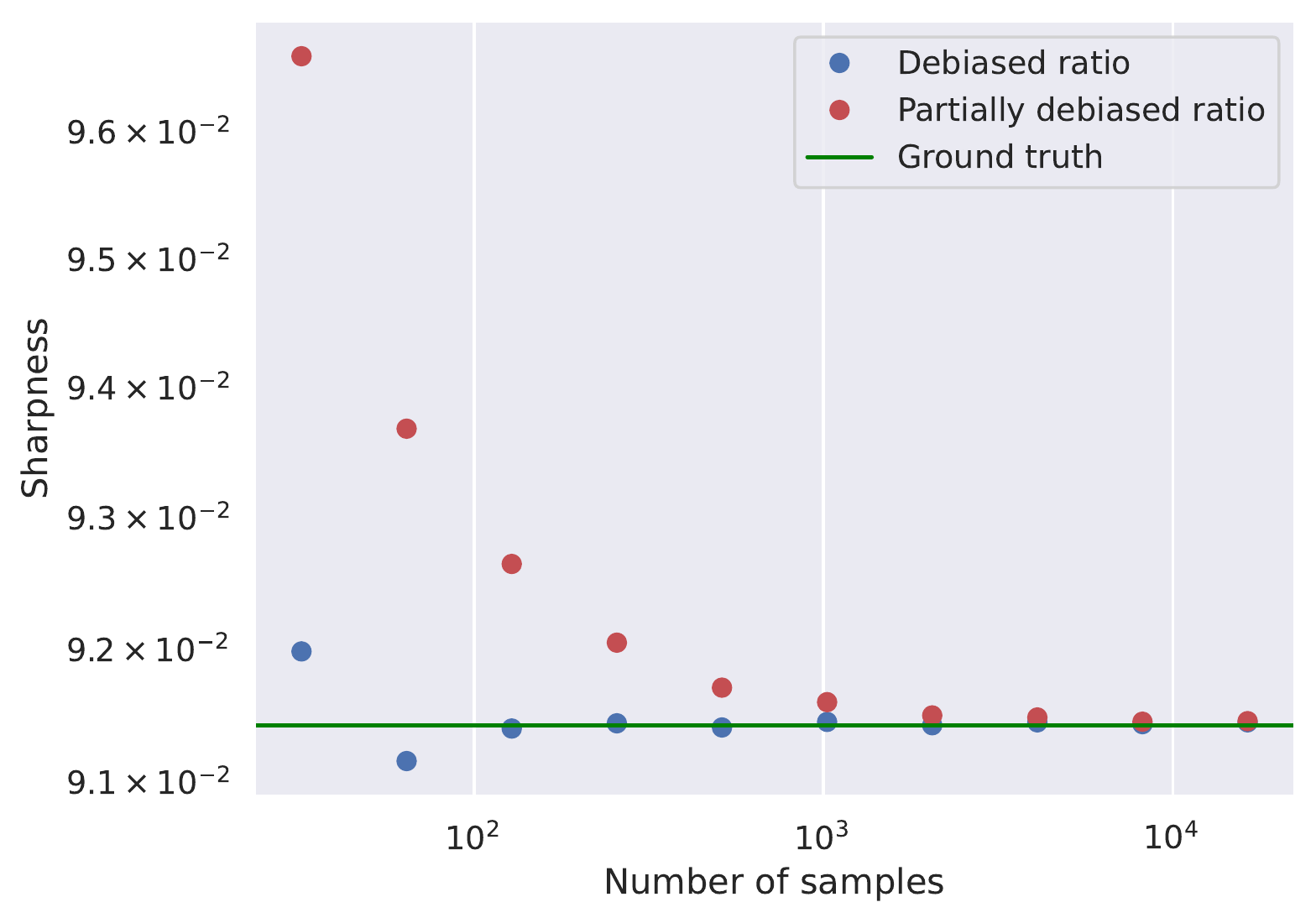}
    \label{fig:debiasing_experiment}
    }

    \caption{ \ref{fig:synthetic_experiment} Performance of $ECE^{bin}$ and $ECE^{KDE}$ on synthetic data for varying number of classes, as a function of the sample size. Ground truth represents the true value of the integral. $ECE^{bin}$ is calculated using several common choices for the number of bins (\textit{n\_bins} represents number of bins per-class.) \textit{n\_bins\*} and \textit{b\*} are found as optimal values according to Doane's formula \citep{doanesformula1976} and LOO MLE, respectively. $ECE^{KDE}$ converges to the true value in all settings, in contrast to $ECE^{bin}$. \ref{fig:debiasing_experiment} Sharpness term evaluated for different numbers of samples with the partially debiased ratio from Equation (\ref{eq:SharpnessSlightlyDebiased}) and with the debiasing scheme derived in Appendix \ref{App:L2_debiasing} on synthetic data. }
\end{figure}

\section{Calibration regularized training}
\paragraph{Empirical setup}
To showcase our estimator in applications where canonical calibration is crucial, we consider two medical datasets, namely Kather and DermaMNIST. The Kather dataset \citep{Kather2016} consists of 5000 histological images of human colorectal cancer and it has eight different classes of tissue. 
DermaMNIST \citep{yang2021} is a pre-processed version of the HAM10000 dataset \citep{tschandl2018}, containing 10015 dermatoscopic images of skin lesions, categorized in seven classes. Both datasets have been collected in accordance with the Declaration of Helsinki. 
According to standard practice in related works, we trained ResNet \citep{he2015}, ResNet with stochastic depth (SD) \citep{huang2016}, DenseNet \citep{huang2018} and WideResNet \citep{zagoruyko2017} networks also on CIFAR-10/100 \citep{krizhevsky2009}. 
We use 45000 images for training on the CIFAR datasets, 4000 for Kather and 7007 for DermaMNIST. The code is available at \url{https://github.com/tpopordanoska/ece-kde}.
\paragraph{Baselines}
\textit{Cross-entropy}: The first baseline model is trained using cross-entropy (\textbf{XE}), with the data preprocessing, training procedure and hyperparameters described in the corresponding paper for the architecture. 

\textit{Trainable calibration strategies} \textbf{KDE-XE} denotes our proposed estimator $ECE^{KDE}$, as defined in \Cref{eq:canonical_estimator}, jointly trained with cross entropy. \textbf{MMCE} \citep{kumar2018} is a differentiable measure of calibration with a property that it is minimized at perfect calibration, i.e., MMCE is 0 if and only if $\operatorname{CE}_p=0$. It is used as a regulariser alongside NLL, with the strength of regularization parameterized by $\lambda$. \textbf{Focal loss (FL)} \citep{mukhoti2020} is an alternative to the cross-entropy loss, defined as $\mathcal{L}_f = -(1 - f(y|x))^\gamma \log (f(y|x))$, where $\gamma$ is a hyperparameter and $f(y|x)$ is the probability score that a neural network $f$ outputs for a class $y$ on an input $x$. Their best-performing approach is the sample-dependent FL-53 where $\gamma = 5$ for $f(y|x) \in [0, 0.2)$ and $\gamma = 3$ otherwise, followed by the method with fixed $\gamma = 3$. 

\textit{Post-hoc calibration strategies} \citet{guo2017} investigated the performance of several post-hoc calibration methods and found \textbf{temperature scaling} to be a strong baseline, which we use as a representative of this group. It works by scaling the logits with a scalar $T > 0$, typically learned on a validation set by minimizing NLL. 
Following \citet{kumar2018,mukhoti2020}, we also use temperature scaling as a post-processing step for our method.

\paragraph{Metrics}
We report $L_1$ canonical calibration using our $ECE^{KDE}$ estimator, calculated according to \Cref{eq:canonical_estimator}. 
Additional experiments with $L_1$ and $L_2$ top-label calibration on CIFAR-10/100 can be found in Appendix~{\ref{appendix:results}}.

\paragraph{Hyperparameters}

\begin{wrapfigure}{R}{5.5cm}
    \centering
    \includegraphics[width=\linewidth]{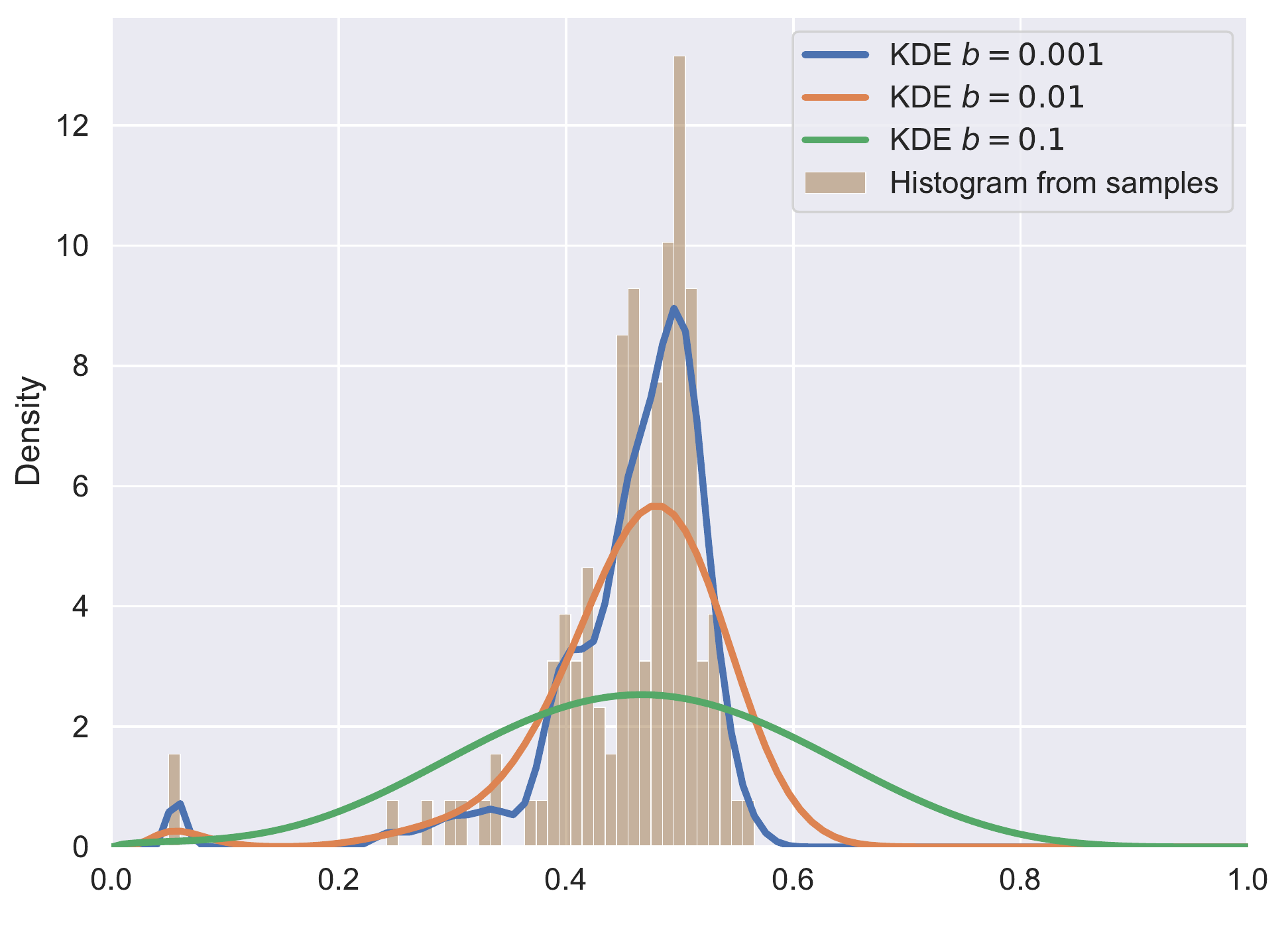}
    \caption{Effect of the bandwidth $b$ on the shape of the estimate.}
    \label{fig:bandwidth_selection}
\end{wrapfigure}

A crucial parameter for KDE is the bandwidth $b$, a positive number that defines the smoothness of the density plot. Poorly chosen bandwidth may lead to undersmoothing (small bandwidth) or oversmoothing (large bandwidth), as shown in Figure \ref{fig:bandwidth_selection}. A commonly used non-parametric bandwidth selector is maximum likelihood cross validation \citep{duin1976}. For our experiments we choose the bandwidth from a list of possible values by maximizing the leave-one-out likelihood (LOO MLE). The $\lambda$ parameter for weighting the calibration error w.r.t the loss is typically chosen via cross-validation or using a holdout validation set. We found that for KDE-XE, values of $\lambda \in [0.001, 0.2]$ provide a good trade-off in terms of accuracy and calibration error. The $p$ parameter is selected depending on the desired $L_p$ calibration error and the corresponding theoretical guarantees. The rest of the hyperparameters for training are set as proposed in the corresponding papers for the architectures we benchmark. In particular, for the CIFAR-10/100 datasets we used a batch size of 64 for DenseNet and 128 for the other architectures. For the medical datasets, we used a batch size of 64, due to their smaller size.

\subsection{Experiments}

An important property of our $ECE^{KDE}$ estimator is differentiability, allowing it be used in a calibration regularized training framework. In this section, we benchmark KDE-XE with several baselines
on medical diagnosis applications, where the calibration of the whole probability vector is of particular interest.
For completeness, we also include an experiment on CIFAR-10. 

Table~\ref{table:ece_tab1} summarizes the canonical $L_1$ $ECE^{KDE}$ and Table~\ref{table:acc_tab1} the accuracy, measured across multiple architectures. The bandwidth is chosen by LOO MLE. For MMCE and KDE-XE the best performing regularization weight is reported. 
In Table~\ref{table:ece_tab1} we notice that KDE-XE consistently achieves very competitive ECE values, while also boosting the accuracy, as shown in Table~{\ref{table:acc_tab1}}. Interestingly, we observe that temperature scaling does not improve canonical calibration error, contrary to its reported improvements on top-label calibration. This observation that temperature scaling is less effective for stronger notions of calibration is consistent with a similar finding in \cite{kull2019}, where the authors show that although the temperature-scaled model has well calibrated top-label confidence scores, the calibration error is much larger for class-wise calibration. 

\begin{table*}[ht]
	\centering
	\footnotesize
	\caption{Canonical $L_1$ $ECE^{KDE}$ ($\downarrow$) for different loss functions and architectures, both trained from scratch (Pre T) and after temperature scaling on a validation set (Post T). Best results across Pre T methods are marked in bold.}
	\resizebox{0.9\linewidth}{!}{%
		\begin{tabular}{cccccccccc}
			\toprule
			Dataset & Model & \multicolumn{2}{c}{XE} & \multicolumn{2}{c}{MMCE}
			& \multicolumn{2}{c}{FL-53} & \multicolumn{2}{c}{\textbf{$L_1$ KDE-XE (Our)} } \\
			&& Pre T & Post T & Pre T & Post T & Pre T & Post T & Pre T & Post T  \\
			\midrule
			
			\multirow{4}{*}{Kather} 
			& ResNet-110 &  0.335 & 0.304 & 0.343 & 0.300 & 0.325 & 0.248 & \textbf{0.311} & 0.289\\
			& ResNet-110 (SD) & 0.329 & 0.334 &  0.235 & 0.159 & 0.209 & 0.122 &  \textbf{0.198} & 0.147\\
			& Wide-ResNet-28-10 & 0.177 & 0.259 & 0.201 & 0.241 & 0.270 & 0.328 & \textbf{0.162} & 0.212\\
			& DenseNet-40 & 0.244 & 0.251 & 0.159 & 0.218 & 0.165 & 0.207 & \textbf{0.114} & 0.154  \\
			\midrule
			\multirow{4}{*}{DermaMNIST} 
			& ResNet-110 & 0.579 & 0.602 & 0.575 & 0.603 & 0.684 & 0.618 & \textbf{0.467} & 0.516\\
			& ResNet-110 (SD) & 0.534 & 0.571 & 0.470 & 0.526 & 0.567 & 0.594 & \textbf{0.461} & 0.538\\
			& Wide-ResNet-28-10 & 0.546 & 0.599 & 0.470 & 0.512 & 0.623 & 0.608 & \textbf{0.455} & 0.599 \\
			& DenseNet-40 & 0.573 & 0.578 & 0.514 & 0.558 & 0.577 & 0.557 & \textbf{0.366} & 0.418 \\
			\midrule
			\multirow{4}{*}{CIFAR-10} 
			& ResNet-110 & 0.133 & 0.170 & 0.171 & 0.196 & 0.138 & 0.171 & \textbf{0.126} & 0.163\\
			& ResNet-110 (SD) & \textbf{0.132} & 0.172 & 0.164 & 0.203 & 0.156 & 0.201 & 0.178 & 0.223\\
			& Wide-ResNet-28-10 & 0.083 & 0.098 & 0.143 & 0.155 & 0.147 & 0.177& \textbf{0.077} & 0.091\\
			& DenseNet-40 & 0.104 & 0.131 &  0.133 & 0.155 & 0.081 &0.081 & \textbf{0.090} & 0.124\\
			\bottomrule
		\end{tabular}%
	}
	\label{table:ece_tab1}
\end{table*}

\begin{table}[ht]
	\centering
	\caption{Accuracy ($\uparrow$) computed for different architectures. Best results are marked in bold.}
	\resizebox{0.7\linewidth}{!}{%
		\begin{tabular}{cccccccccc}
			\toprule
			Dataset & Model & XE & MMCE & FL-53 & \textbf{$L_1$ KDE-XE} \\
			\midrule
			\multirow{4}{*}{Kather} 
			& ResNet-110 &  0.840 & \textbf{0.860} & 0.844 & \textbf{0.860} \\
			& ResNet-110 (SD) & 0.870 & 0.900 & 0.885 & \textbf{0.914} \\
			& Wide-ResNet-28-10 & \textbf{0.933} & 0.899 & 0.873 & 0.921 \\
			& DenseNet-40 & 0.913 & 0.93 & 0.916 & \textbf{0.941} \\
			\midrule
			\multirow{4}{*}{DermaMNIST} 
			& ResNet-110 & 0.720 & 0.721 & 0.674 & \textbf{0.744} \\
			& ResNet-110 (SD) & 0.743 & 0.753 & 0.689 & \textbf{0.764} \\
			& Wide-ResNet-28-10 & 0.736 & 0.741 & 0.715 & \textbf{0.754}\\
			& DenseNet-40 & 0.741 & \textbf{0.758} & 0.705 & 0.748 \\
			\midrule
			\multirow{4}{*}{CIFAR-10} 
			& ResNet-110 & 0.925 & \textbf{0.929} & 0.922 & \textbf{0.929}\\
			& ResNet-110 (SD) &  \textbf{0.926} & 0.925 & 0.92 & 0.907\\
			& Wide-ResNet-28-10 & \textbf{0.954} & 0.947 & 0.936 & \textbf{0.954}\\
			& DenseNet-40 & 0.947 & 0.944 & \textbf{0.948} & 0.947\\
			\bottomrule
		\end{tabular}%
	}
	\label{table:acc_tab1}
\end{table}

Figure \ref{fig:canonical_calibration} shows the performance of several architectures and datasets in terms of accuracy and $L_1$ $ECE^{KDE}$ for various choices of the regularization parameter for MMCE and KDE-XE. The 95$\%$ confidence intervals for $ECE^{KDE}$ are calculated using 100 and 10 bootstrap samples on the medical datasets and CIFAR-10, respectively. In all settings, KDE-XE Pareto dominates the competitors, for several choices of $\lambda$. For example, on DermaMNIST trained with DenseNet, KDE-XE with $\lambda=0.2$ reduces $ECE^{KDE}$ from 66\% to 45\%.
\begin{figure}[ht!]
    \centering
    \subfloat[ResNet-110 (SD) on Kather]{
    \label{subfig:kather}
    \includegraphics[width=.33\linewidth]{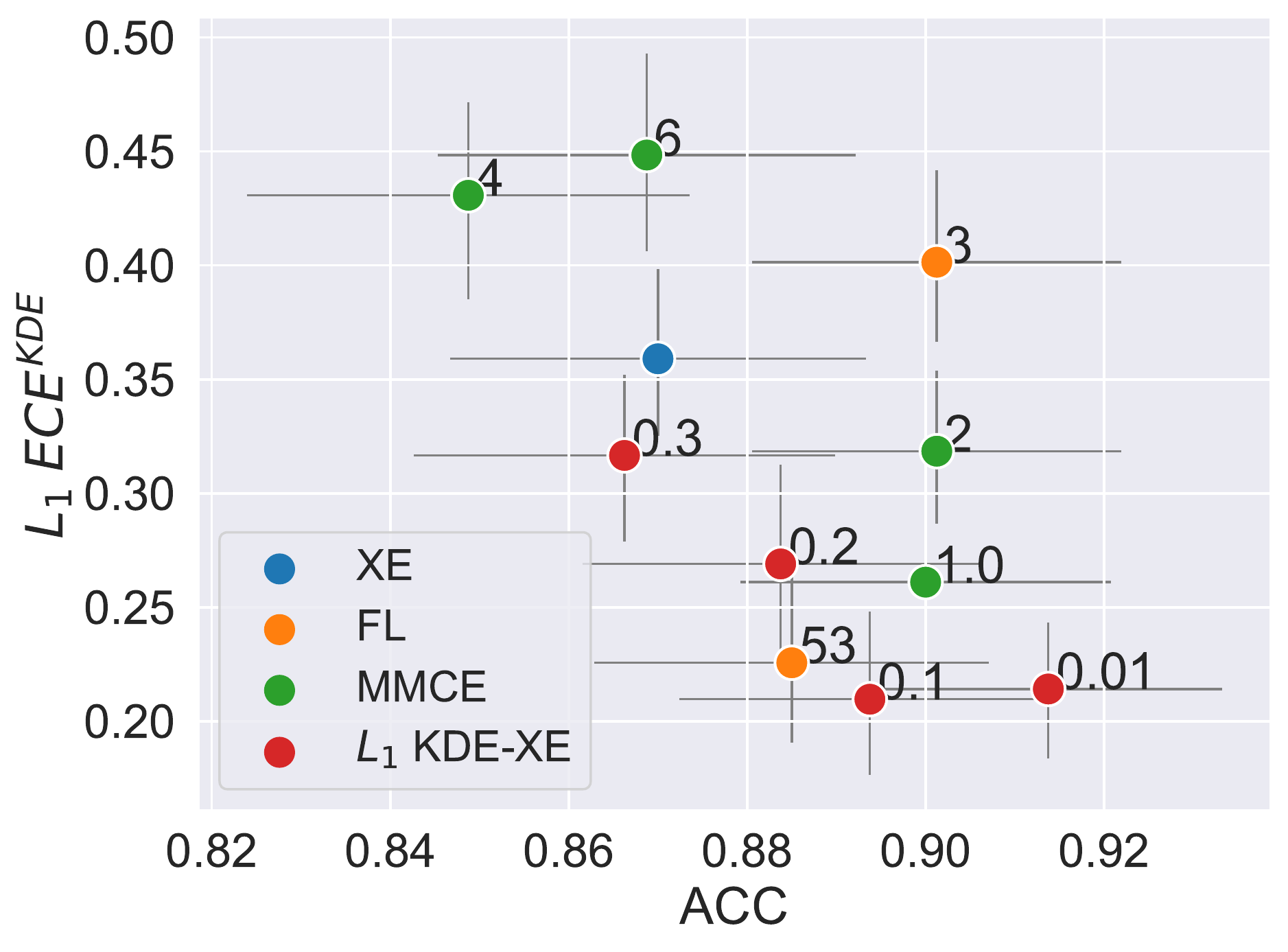}
    } 
    \subfloat[DenseNet on DermaMNIST]{
    \label{subfig:derma}
    \includegraphics[width=.33\linewidth]{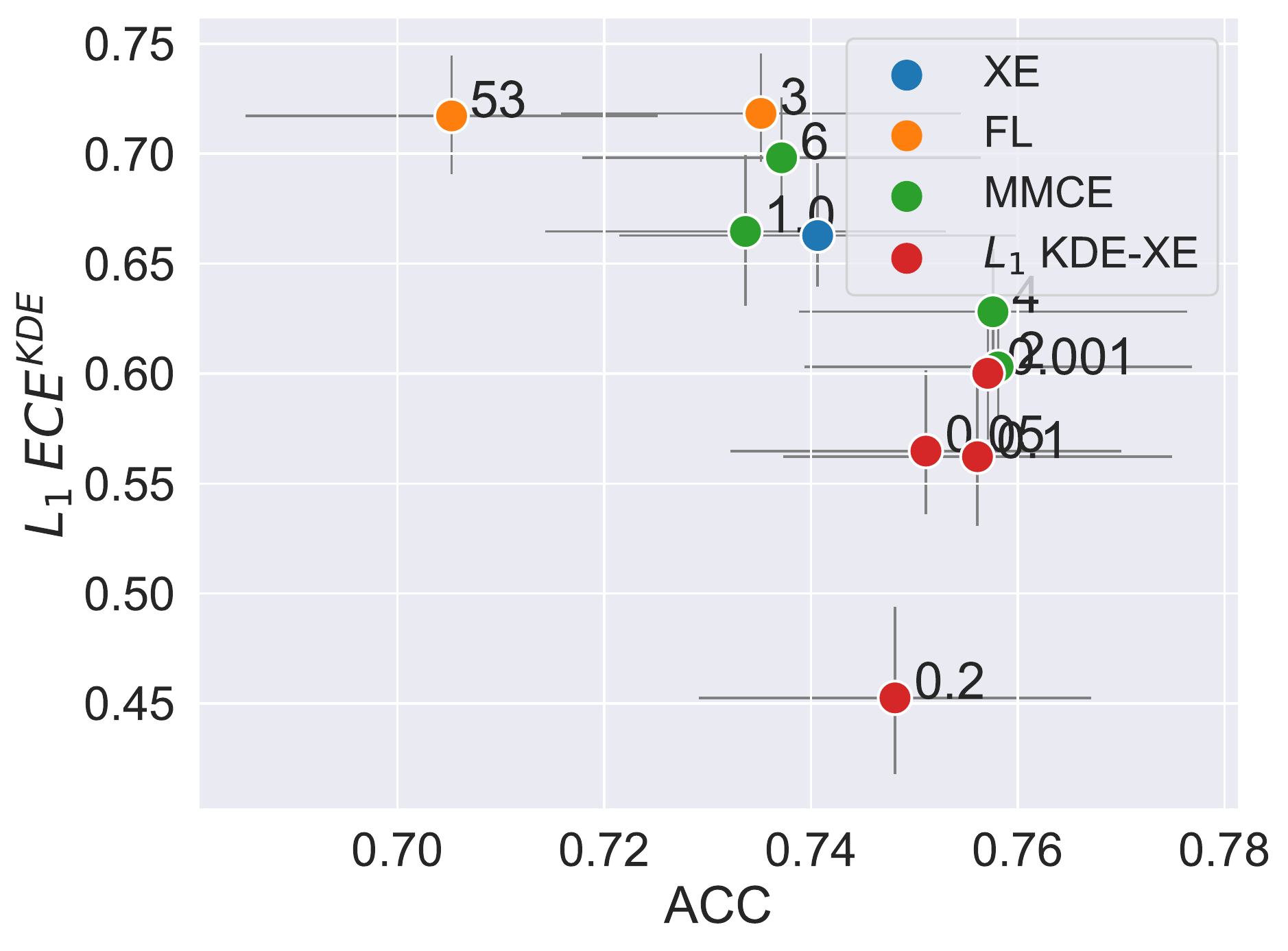} 
    } 
    \subfloat[ResNet-110 on CIFAR-10]{
    \label{subfig:cifar}
    \includegraphics[width=.33\linewidth]{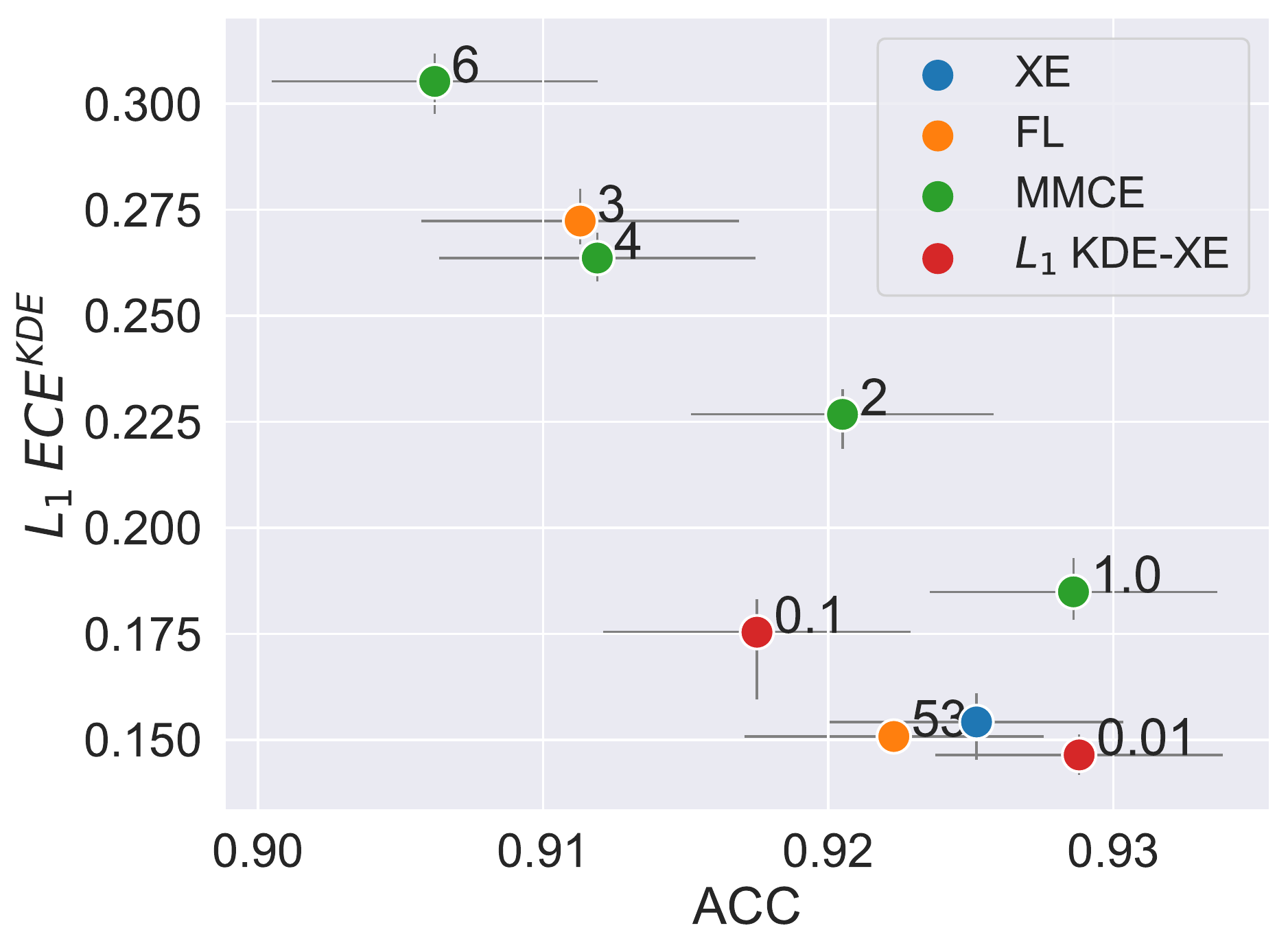}
    }
    \caption{Canonical calibration on various datasets and architectures. The numbers next to the points denote the value of the regularization parameter. KDE-XE outperforms the competitors, both in terms of accuracy and calibration error, for several choices of $\lambda$. }
    \label{fig:canonical_calibration}
\end{figure}

\paragraph{Training time measurements}
\label{subsection:training_times}

\begin{wraptable}{r}{7cm}
\caption{Training time [sec] per epoch for XE and KDE-XE for different models on CIFAR-10.}
	\centering
    \footnotesize
	\resizebox{\linewidth}{!}{%
    	\begin{tabular}{cccc}
    		\toprule
    		\textbf{Dataset} & \textbf{Model} & \textbf{XE} &  \textbf{$L_1$ KDE-XE}  \\
    		\midrule
        	\multirow{4}{*}{CIFAR-10} 
            & ResNet-110 & 51.8 & 53 \\
            & ResNet-110 (SD) & 45 & 46\\
            & Wide-ResNet-28-10 & 152.9 & 154.9\\
            & DenseNet-40 & 103.2 & 106.8\\
    		\bottomrule
    	\end{tabular}%
	}
	\label{table:training_times}
\end{wraptable}

In Table~\ref{table:training_times} we summarize the running time per epoch of the four architectures, with regularization (KDE-XE) and without regularization (XE). We observe only an insignificant impact on the training speed when using KDE-XE, dispelling any concerns w.r.t. the computational overhead. 

To summarize, the experiments show that our estimator is consistently producing competitive calibration errors with other state-of-the-art approaches, while maintaining accuracy and keeping the computational complexity at $\mathcal{O}(n^2)$. We note that within the proposed calibration-regularized training framework, this complexity is w.r.t. to a mini-batch, and the added cost is less than a couple percent. Furthermore, the $\mathcal{O}(n^2)$ complexity shows up in other related works \citep{kumar2018,zhang2020}, and is intrinsic to the problem of density estimators of calibration error. As a future work, a larger scale benchmarking will be beneficial for exploring the limits of canonical calibration using Dirichlet kernels. 

\section{Conclusion}
In this paper, we proposed a consistent and differentiable estimator of canonical $L_p$ calibration error using Dirichlet kernels. It has favorable computational and statistical properties, with a complexity of $\mathcal{O}(n^2)$, convergence of $\mathcal{O}(n^{-1/2})$ and a bias that converges as $\mathcal{O}(n^{-1})$, which can be further reduced to $\mathcal{O}(n^{-2})$ using our debiasing strategy. The $ECE^{KDE}$ can be directly optimized alongside any loss function in the existing batch stochastic gradient descent framework. Furthermore, we propose using it as a measure of the highest form of calibration, which requires the entire probability vector to be calibrated. To the best of our knowledge, this is the only metric that can tractably capture this type of calibration, which is crucial in safety-critical applications where downstream decisions are made based on the predicted probabilities. We showed empirically on a range of neural architectures and datasets that the performance of our estimator in terms of accuracy and calibration error is competitive against the current state-of-the-art, while having superior properties as a consistent estimator of canonical calibration error.

\section*{Acknowledgments}
 This research received funding from  the Research Foundation - Flanders (FWO) through project number S001421N, and the Flemish Government under  the  ``Onderzoeksprogramma  Artifici\"{e}le  Intelligentie (AI) Vlaanderen'' programme. R.S. was supported in part by the Tübingen AI centre.

\section*{Ethical statement}
\label{App:Ethical_statement}
The paper is concerned with estimation of calibration error, a topic for which existing methods are deployed, albeit not typically for canonical calibration error in a multi-class setting.  We therefore consider the ethical risks to be effectively the same as for any probabilistic classifier.  Experiments apply the method to medical image classification, for which misinterpretation of benchmark results with respect to their clinical applicability has been highlighted as a risk, see e.g.\ \cite{varoquaux2022}.


\bibliographystyle{plainnat}
\bibliography{refs}

\section*{Checklist}

\begin{enumerate}

\item For all authors...
\begin{enumerate}
  \item Do the main claims made in the abstract and introduction accurately reflect the paper's contributions and scope?
    \answerYes{}
  \item Did you describe the limitations of your work?
    \answerYes{}
  \item Did you discuss any potential negative societal impacts of your work?
    \answerYes{Please refer to our ethical statement}.
  \item Have you read the ethics review guidelines and ensured that your paper conforms to them?
    \answerYes{}
\end{enumerate}

\item If you are including theoretical results...
\begin{enumerate}
  \item Did you state the full set of assumptions of all theoretical results?
    \answerYes{}
        \item Did you include complete proofs of all theoretical results?
    \answerYes{}
\end{enumerate}

\item If you ran experiments...
\begin{enumerate}
  \item Did you include the code, data, and instructions needed to reproduce the main experimental results (either in the supplemental material or as a URL)?
    \answerYes{It is in the supplementary material}
  \item Did you specify all the training details (e.g., data splits, hyperparameters, how they were chosen)?
    \answerYes{}
        \item Did you report error bars (e.g., with respect to the random seed after running experiments multiple times)?
    \answerYes{See Figure \ref{fig:canonical_calibration}}
        \item Did you include the total amount of compute and the type of resources used (e.g., type of GPUs, internal cluster, or cloud provider)?
    \answerYes{}\answerNo{} Table~\ref{table:training_times} includes compute times. Most of our results are provided in big O complexity.
\end{enumerate}

\item If you are using existing assets (e.g., code, data, models) or curating/releasing new assets...
\begin{enumerate}
  \item If your work uses existing assets, did you cite the creators?
    \answerYes{}
  \item Did you mention the license of the assets?
    \answerNo{} We do not release the data.  Data license is available via the citation.
  \item Did you include any new assets either in the supplemental material or as a URL?
    \answerNo{}
  \item Did you discuss whether and how consent was obtained from people whose data you're using/curating?
    \answerYes{} Medical datasets used in this paper conform to the Declaration of Helsinki.
  \item Did you discuss whether the data you are using/curating contains personally identifiable information or offensive content?
    \answerYes{} Medical datasets used in this paper conform to the Declaration of Helsinki.
\end{enumerate}

\item If you used crowdsourcing or conducted research with human subjects...
\begin{enumerate}
  \item Did you include the full text of instructions given to participants and screenshots, if applicable?
    \answerNA{We did not use crowdsourcing or conducted research with human subjects}
  \item Did you describe any potential participant risks, with links to Institutional Review Board (IRB) approvals, if applicable?
    \answerNA{We did not use crowdsourcing or conducted research with human subjects}
  \item Did you include the estimated hourly wage paid to participants and the total amount spent on participant compensation?
    \answerNA{We did not use crowdsourcing or conducted research with human subjects}
\end{enumerate}

\end{enumerate}



\appendix

\section{Additional derivations}
\label{appendix:mse_decomposition}

\subsection{Derivation of the MSE decomposition}

\begin{definition}[Mean Squared Error (MSE)]
The mean squared error of an estimator is
\begin{align}
    \operatorname{MSE}(f) := \mathbb{E}[(f(x) - y)^2] .
\end{align}
\end{definition}

\begin{proposition}
    \label{prop:MSEgeqCE2}
    $\operatorname{MSE}(f)\geq \operatorname{CE}_2(f)^2$
\end{proposition}

\begin{proof}
\begin{align}
    \operatorname{MSE}(f):=&\mathbb{E}[(f(x) - y))^2] 
    = \mathbb{E}[((f(x) - \mathbb{E}[y\mid f(x)]) + (\mathbb{E}[y\mid f(x)] - y))^2] \\
    =& \underbrace{\mathbb{E}[(f(x) - \mathbb{E}[y\mid f(x)])^2]}_{=CE_{2}^2} + \mathbb{E}[(\mathbb{E}[y\mid f(x)] - y)^2]  \\
    &+ 2 \mathbb{E}[(f(x) - \mathbb{E}[y\mid f(x)])(\mathbb{E}[y\mid f(x)] - y)] \nonumber
    \end{align}
    which implies
    \begin{align}
    \operatorname{MSE}(f)-\operatorname{CE}_2(f)^2 =& \mathbb{E}[(\mathbb{E}[y\mid f(x)] - y)^2] \\
    &+ 2 \mathbb{E}[(f(x) - \mathbb{E}[y\mid f(x)])(\mathbb{E}[y\mid f(x)] - y)] \nonumber \\
    =& \mathbb{E}[(\mathbb{E}[y\mid f(x)] - y)^2] + 2 \mathbb{E}[(f(x)\mathbb{E}[y\mid f(x)]] \\
    &- 2 \mathbb{E}[f(x)y] - 2 \mathbb{E}[\mathbb{E}[y\mid f(x)]^2] +2 \mathbb{E}[\mathbb{E}[y\mid f(x)] y]] \nonumber \\
    =&\mathbb{E}[\mathbb{E}[y\mid f(x)]^2] + \mathbb{E}[y^2] - 2 \mathbb{E}[\mathbb{E}[y\mid f(x)] y] \\
    &+ 2 \mathbb{E}[(f(x)\mathbb{E}[y\mid f(x)]] 
    - 2 \mathbb{E}[f(x)y] \nonumber \\
    &- 2 \mathbb{E}[\mathbb{E}[y\mid f(x)]^2] +2 \mathbb{E}[\mathbb{E}[y\mid f(x)] y]] \nonumber \\
    =&  \mathbb{E}[y^2]  + 2 \mathbb{E}[(f(x)\mathbb{E}[y\mid f(x)]] - 2 \mathbb{E}[f(x)y] \\
    &-  \mathbb{E}[\mathbb{E}[y\mid f(x)]^2] \nonumber \\
    =& \mathbb{E}[(2 f(x) - y - \mathbb{E}[y\mid f(x)]) (\mathbb{E}[y\mid f(x)])-y] \label{eq:midstep_mse_decomposition} \\
    =& \mathbb{E}[( f(x) - y)  (\mathbb{E}[y\mid f(x)] - y)]  \\
    &+ \mathbb{E}[ (f(x) - \mathbb{E}[y\mid f(x)]) (\mathbb{E}[y\mid f(x)] - y)] . \nonumber
\end{align}
%

By the law of total expectation, we will write the above as
\begin{align}
    \operatorname{MSE}(f)-\operatorname{CE}_2(f)^2 
    = \mathbb{E}[ \mathbb{E}[&( f(x) - y)  (\mathbb{E}[y\mid f(x)] - y)  \\
    &+  (f(x) - \mathbb{E}[y\mid f(x)]) (\mathbb{E}[y\mid f(x)] - y)\mid f(x)]] . \nonumber
\end{align}

Focusing on the inner conditional expectation, we have that
\begin{align}
    \mathbb{E}[(& f(x)  - y)  (\mathbb{E}[y\mid f(x)] - y)  +  (f(x) - \mathbb{E}[y\mid f(x)]) (\mathbb{E}[y\mid f(x)] - y)\mid f(x)] \nonumber \\
    =&  \mathbb{E}[y\mid f(x)]  (f(x) - 1)(\mathbb{E}[y\mid f(x)] - 1)
    +(1 - \mathbb{E}[y\mid f(x)])f(x) \mathbb{E}[y\mid f(x)]  \nonumber \\
    &+ \mathbb{E}[y\mid f(x)] (f(x) - \mathbb{E}[y\mid f(x)])(\mathbb{E}[y\mid f(x)] - 1)
    \nonumber \\ &+ (1-\mathbb{E}[y\mid f(x)]) (f(x) - \mathbb{E}[y\mid f(x)])\mathbb{E}[y\mid f(x)] \label{eq:mse_ce_diff_expansion}\\
    =& (1-\mathbb{E}[y\mid f(x)])\mathbb{E}[y\mid f(x)] \geq 0 \quad \forall f(x) \label{eq:mse_ce_diff}
\end{align}
and therefore

\begin{align} 
    \label{eq:MSEminusCE2}
    \operatorname{MSE}(f)-\operatorname{CE}_2(f)^2 = \mathbb{E}[(1-\mathbb{E}[y\mid f(x)])\mathbb{E}[y\mid f(x)]]\geq 0 .
\end{align}

\end{proof}
The expectation in ~\Cref{eq:MSEminusCE2} is over variances of Bernoulli random variables with probabilities $\mathbb{E}[y\mid f(x)]$.

\subsection{Derivation of \Cref{eq:estimator_EYfX}}

By considering $y \in \{0, 1\}$, we have the following:
\begin{align}
    \mathbb{E}[y\mid f(x)]&= \sum_{y_k \in \mathcal{Y}} y_k \, p_{y|f(x)}(y_k) = \frac{\sum_{y_k \in \mathcal{Y}} y_k \, p_{f(x), y}(f(x), y_k)}{p_{f(x)}(f(x))} \\ 
    &= \frac{p_{f(x), y}(f(x), y_k=1)}{p_{f(x)}(f(x))} = \frac{p_{f(x)|y}(f(x)|y_k=1)p_y(y_k=1)}{p_{f(x)}(f(x))} \\
    &\approx \frac{\frac{1}{\sum_{i=1}^n y_i}\sum_{i=1}^n k(f(x) ; f(x_i))y_i \frac{\sum_{i=1}^n y_i}{n}}{\frac{1}{n}\sum_{i=1}^n k(f(x) ; f(x_i))} \\
    &\approx  \frac{\sum_{i=1}^n k(f(x) ; f(x_i))y_i}{\sum_{i=1}^n k(f(x) ; f(x_i))} =: \widehat{\mathbb{E}[y \mid f(x)]}
\end{align}

\subsection{Derivation of \Cref{eq:mse_binary}} 

We consider the optimization problem for some $\lambda>0$:
\begin{align}
    f = \arg\min_{f\in \mathcal{F}} \Bigl( \operatorname{MSE}(f)  +\lambda \operatorname{CE}_2(f)^2 \Bigr).
\end{align}
Using \Cref{eq:MSEminusCE2} we rewrite: 
    \begin{align}
        \operatorname{MSE}(f)  +\lambda \operatorname{CE}_2(f)^2 \nonumber 
        &= (1+\lambda) \operatorname{MSE}(f)  -\lambda \Bigl(\operatorname{MSE}(f)-\operatorname{CE}_2(f)^2\Bigr) \nonumber \\
        &= (1+\lambda) \operatorname{MSE}(f)  -\lambda \mathbb{E}\biggl[\Bigl(1-\mathbb{E}[y\mid f(x)]\Bigr)\mathbb{E}[y\mid f(x)]\biggr]. \label{eq:CalibrationRetularized}
    \end{align}
Rescaling \Cref{eq:CalibrationRetularized} by a factor of $(1+\lambda)^{-1}$ and a variable substitution $\gamma=\frac{\lambda}{1+\lambda} \in [0,1)$, we have that:
\begin{align}
    f=&\arg\min_{f\in \mathcal{F}}\Bigl( \operatorname{MSE}(f)  +\lambda \operatorname{CE}_2(f)^2\Bigr) \nonumber \\ 
    =& \arg\min_{f\in \mathcal{F}}\biggl(  \operatorname{MSE}(f)    -\gamma \mathbb{E}\biggl[\Bigl(1-\mathbb{E}[y\mid f(x)]\Bigr)\mathbb{E}[y\mid f(x)]\biggr]\biggr)  \nonumber \\
    =& \arg\min_{f\in \mathcal{F}} \biggl( \operatorname{MSE}(f)    +\gamma \mathbb{E}\Bigl[\mathbb{E}[y\mid f(x)]^2\Bigr]\biggr) .
\end{align}


\section{Bias of ratio of U-statistics}
The unbiased estimator for the square of a mean $\mu_X^2$ is given by:
\begin{align}
    \widehat{\mu_X^2} = \frac{1}{n(n-1)}\sum_{i=1}^n \sum_{\substack{j=1\\j \neq i}}^n X_i X_j = \frac{1}{n(n-1)}\left(\left(\sum_{i=1}^n X_i\right)^2 - \sum_{i=1}^n X_i^2\right).
\end{align}
This is a second order U-statistics with kernel $h(x_1, x_2)=x_1 x_2$. The bias of the ratio of two of these estimators converges as $\mathcal{O}\left(\frac{1}{n}\right)$, as the following lemma proves.
\label{App:Proof_lemma}
\begin{lemma}
\label{Lemma:Convergence_Bias}
Let $\theta_1$ and $\theta_2$ be two estimable parameters and let $U_1$ and $U_2$ be the two corresponding U-statistics of order $m_1$ and $m_2$, respectively, based on a sample of $n$ i.i.d. RVs. The bias of the ratio $U_1 / U_2$ of these two U-statistics will converge as  $\mathcal{O}\left(\frac{1}{n}\right)$.
\end{lemma}

\begin{proof}
Let $R=\theta_1 / \theta_2$ be the ratio of two estimable parameters and $r=U_1 / U_2$ the ratio of the corresponding U-statistics. Note, that $U_i$ is an unbiased estimator of $\theta_i$, $\mathbb{E}[U_i]=\theta_i$, $i=1,2$, however, the ratio is usually biased. To investigate the bias of that ratio we rewrite 
\begin{align}
    r = R \biggl(1+\frac{U_1 -\theta_1 }{\theta_1} \biggr)\biggl(1+\frac{U_2 -\theta_2 }{\theta_2} \biggr)^{-1}.
\end{align}
If $\left|\frac{U_2 -\theta_2 }{\theta_2} \right|<1$, we can expand $\biggl(1+\frac{U_2 -\theta_2 }{\theta_2} \biggr)^{-1}$ in a geometric series:
\begin{align}
    r &=R \biggl(1+\frac{(U_1 -\theta_1) }{\theta_1} \biggr) \biggl(1- \frac{(U_2 -\theta_2) }{\theta_2} +\frac{(U_2 -\theta_2)^2 }{\theta_2^2}-\frac{(U_2 -\theta_2)^3 }{\theta_2^3}+\frac{(U_2 -\theta_2)^4 }{\theta_2^4}-...\biggr)\\
        &=R\Biggl(1 +\frac{(U_1 -\theta_1) }{\theta_1}-\frac{(U_2 -\theta_2) }{\theta_2}-\frac{(U_2 -\theta_2)(U_1 -\theta_1) }{\theta_2\theta_1} \nonumber \\
        &\phantom{asdfasd}+\frac{(U_2 -\theta_2)^2 }{\theta_2^2}+\frac{(U_2 -\theta_2)^2(U_1 -\theta_1)}{\theta_2^2\theta_1}
        -\frac{(U_2 -\theta_2)^3 }{\theta_2^3}-\frac{(U_2 -\theta_2)^3(U_1 -\theta_1) }{\theta_2^3\theta_1} \nonumber \\
        &\phantom{asdfasd}+\frac{(U_2 -\theta_2)^4 }{\theta_2^4}+
        \frac{(U_2 -\theta_2)^4(U_1 -\theta_1) }{\theta_2^4\theta_1} -
        ... \Biggr).  
\end{align}

If $\zeta_1 > 0$, a U-statistic $U$ of order $m$ obtained from a sample of $n$ observations converges in distribution \citep{shao2003}: 
\begin{align}
    \sqrt{n}\left(U - \mathbb{E}[U]\right) \xrightarrow{\text{d}} N(0, m^2 \zeta_1).
    \label{eq:convergence_U_statistics}
\end{align}
Keeping the terms up to $\Theta\left(\frac{1}{n} \right)$:
\begin{align}
    \begin{split}
        r&=R\Biggl(1 +\frac{(U_1 -\theta_1 )}{\theta_1}-\frac{(U_2 -\theta_2) }{\theta_2}-\frac{(U_2 -\theta_2)(U_1 -\theta_1) }{\theta_2\theta_1}+\frac{(U_2 -\theta_2)^2 }{\theta_2^2} + o\left(\frac{1}{n}\right)\Biggl)
    \end{split}
\end{align}
To examine the bias, we take the expectation value of this expression:
\begin{align}
    \begin{split}
        \mathbb{E}[r] &= R\Biggl(1 +\frac{\mathbb{E}\bigl[(U_1 -\theta_1) \bigr]}{\theta_1}-\frac{\mathbb{E}\bigl[(U_2 -\theta_2) \bigr]}{\theta_2}-\frac{\mathbb{E}\bigl[(U_2 -\theta_2)(U_1 -\theta_1) \bigr]}{\theta_2\theta_1}+\frac{\mathbb{E}\bigl[(U_2 -\theta_2)^2 \bigr]}{\theta_2^2}
        + o\left(\frac{1}{n}\right) \Biggr)
    \end{split}
\end{align}
We now make use of the following expressions:
\begin{align}
    \mathbb{E}\bigl[(U_1 -\theta_1) \bigr]&=\mathbb{E}\bigl[(U_2 -\theta_2) \bigr]=0\\
    \mathbb{E}\bigl[(U_2 -\theta_2)(U_1 -\theta_1) \bigr]&=\operatorname{Cov}(U_2, U_1)\\
    \mathbb{E}\bigl[(U_2 -\theta_2)^2 \bigr]&= \operatorname{Var}(U_2)\\
\end{align}
Using these expressions the expectation of $r$ becomes:
\begin{align}
    \begin{split}
        \mathbb{E}[r] &= R\Biggl(1 - \frac{\operatorname{Cov}(U_2, U_1)}{\theta_2 \theta_1} + \frac{\operatorname{Var}(U_2)}{\theta_2^2} + o\left(\frac{1}{n}\right) \Biggr)
    \end{split}
\end{align}
Using Equation (\ref{eq:convergence_U_statistics}), 
the linearity of covariance and with $\operatorname{Var}(aX)=a^2 \operatorname{Var}(X)$ we obtain:
\begin{align}
\operatorname{Cov}(U_2, U_1), \operatorname{Var}(U_2) \in \mathcal{O}\left( \frac{1}{n} \right) \implies \mathbb{E}[r] = R\Biggl(1 + \mathcal{O}\left(\frac{1}{n}\right) \Biggr) .
\end{align}

\end{proof}

\section{De-biasing of ratios of straight averages}
\label{App:L1_debiasing}
Let $X$ and $Y$ be random variables and let $\mu_X$ and $\mu_Y$ be the means of their distributions, respectively.\\
Consider the problem of finding an unbiased estimator for the ratio of means:
\begin{align}
    R = \frac{\mu_Y}{\mu_X}.
\end{align}
A first approach to estimate this ratio $R$ is to compute the ratio of the sample means: Let $(X_1, Y_1), ..., (X_n, Y_n)$ be pairs of i.i.d. random variables that are jointly distributed:
\begin{align}
    r = \hat{R} = \frac{\hat{\mu_Y}}{\hat{\mu_X}} =\frac{\frac{1}{n}\sum_{i=1}^n Y_i}{\frac{1}{n}\sum_{i=1}^n 
    X_i}=\frac{\bar{Y}}{\bar{X}}.
\end{align}
This, however, is a biased estimator, which can be seen as follows (we follow \citep{Tin1965, Ogliore2011} here):
\begin{align}
    r =\frac{\bar{Y}}{\bar{X}} = \frac{\mu_Y}{\mu_X} \left(\frac{\bar{Y}}{\mu_Y}\right)\left(\frac{\bar{X}}{\mu_X}\right)^{-1}=R\biggl(1 + \frac{\bar{Y}-\mu_Y}{\mu_Y} \biggr)\biggl(1 + \frac{\bar{X}-\mu_X}{\mu_X} \biggr)^{-1}.
\end{align}
This has now the form of a converging geometric series. Thus, if
\begin{align}
    \biggl| \frac{\bar{X}-\mu_X}{\mu_X}\biggr| < 1, 
\end{align}
we can expand $\biggl(1 + \frac{\bar{X}-\mu_X}{\mu_X} \biggr)^{-1}$ in a geometric series, which is defined as:
\begin{align}
    \sum_{k=0}^{\infty} a \, b^k = a + ab + a b^2 + ... = \frac{a}{1-b}.
\end{align}
In our case we can identify $a=R\biggl(1 + \frac{\bar{Y}-\mu_Y}{\mu_Y} \biggr)$ and $b=-\frac{\bar{X}-\mu_X}{\mu_X}$.\\
Thus, using the geometric series expansion, we can write:
\begin{align}
    r&= R\biggl(1 + \frac{\bar{Y}-\mu_Y}{\mu_Y} \biggr)\biggl(1 - \frac{(\bar{X}-\mu_X)}{\mu_X} + \frac{(\bar{X}-\mu_X)^2}{\mu_X^2} - \frac{(\bar{X}-\mu_X)^3}{\mu_X^3} + \frac{(\bar{X}-\mu_X)^4}{\mu_X^4} - ...
    \biggr)\\
    \begin{split}
        &=R\biggl(1 + \frac{(\bar{Y}-\mu_Y)}{\mu_Y} - \frac{(\bar{X}-\mu_X)}{\mu_X} - \frac{(\bar{X}-\mu_X)(\bar{Y}-\mu_Y)}{\mu_Y \mu_X} + \frac{(\bar{X}-\mu_X)^2}{\mu_X^2} \\
        & \phantom{asdfa}+ \frac{(\bar{X}-\mu_X)^2 (\bar{Y}-\mu_Y)}{\mu_X^2 \mu_Y} - \frac{(\bar{X}-\mu_X)^3}{\mu_X^3} - \frac{(\bar{X}-\mu_X)^3 (\bar{Y}-\mu_Y)}{\mu_X^3 \mu_Y}  + \frac{(\bar{X}-\mu_X)^4}{\mu_X^4} + ... \biggr)
    \end{split}
\end{align}
\paragraph{Neglecting higher order terms}
Since $\bar{X}$ and $\bar{Y}$ are U-statistics, we make use of the asymptotic behaviour of U-statistics. If $\zeta_1 > 0$, a U-statistics $U_n$ of order $m$ obtained from a sample of $n$ observations behaves as $n \rightarrow \infty$ like (\citep{shao2003}):
\begin{align}
    \sqrt{n}\left(U_n - \mathbb{E}[U_n]\right) \xrightarrow{\text{d}} N(0, m^2 \zeta_1).
\end{align}
As we seek an estimator that is unbiased up until order $n^{-2}$ and since $\mathbb{E}[\bar{X}] = \mu_X$, we can neglect all terms of order 5 or higher since for $n \rightarrow \infty$:
\begin{align}
    (\bar{X} - \mu_X)^5 &\in \mathcal{O}(n^{-2.5})\\
    (\bar{X} - \mu_X)^4 (\bar{Y} - \mu_Y) &\in \mathcal{O}(n^{-2.5})
\end{align}
Therefore, we obtain:
\begin{align}
    \begin{split}
        r&\approx R\biggl(1 + \frac{(\bar{Y}-\mu_Y)}{\mu_Y} - \frac{(\bar{X}-\mu_X)}{\mu_X} - \frac{(\bar{X}-\mu_X)(\bar{Y}-\mu_Y)}{\mu_Y \mu_X} + \frac{(\bar{X}-\mu_X)^2}{\mu_X^2} \\
        & \phantom{asdfa}+ \frac{(\bar{X}-\mu_X)^2 (\bar{Y}-\mu_Y)}{\mu_X^2 \mu_Y} - \frac{(\bar{X}-\mu_X)^3}{\mu_X^3} - \frac{(\bar{X}-\mu_X)^3 (\bar{Y}-\mu_Y)}{\mu_X^3 \mu_Y}  + \frac{(\bar{X}-\mu_X)^4}{\mu_X^4}  \biggr)
    \end{split}
\end{align}

\paragraph{Identities to compute the terms of the series expansion of $r$}
\begin{align}
    \mathbb{E}[\bar{X}-\mu_X]&=\mathbb{E}[\bar{Y}-\mu_Y]=0\\
    \mathbb{E}[(\bar{X}-\mu_X)^2]&=\operatorname{Var}(\bar{X})=\frac{1}{n}\operatorname{Var(X)}\\
    \mathbb{E}\biggl[\Bigl(\bar{X}- \mu_X\Bigr)\Bigl(\bar{Y}- \mu_Y\Bigr)\biggr]&=\operatorname{Cov}(\bar{X}, \bar{Y}) = \frac{1}{n}\operatorname{Cov}(X, Y)\\
    \mathbb{E}\biggl[\Bigl(\bar{X}- \mu_X\Bigr)^2\Bigl(\bar{Y}- \mu_Y\Bigr)\biggr]&=\operatorname{Cov}(\bar{X}^2, \bar{Y}) -2\mu_X \operatorname{Cov}(\bar{X}, \bar{Y})\\
    &= \frac{1}{n^2}\Bigl(\operatorname{Cov}(X^2, Y) -2\mu_X \operatorname{Cov}(X, Y)\Bigr)\\
    \mathbb{E}\biggl[\Bigl(\bar{X}-\mu_X\Bigr)^3\biggr]&=\operatorname{Cov}(\bar{X}^2, \bar{X})-2\mu_X \operatorname{Var}(\bar{X})\\
    &=\frac{1}{n^2}\operatorname{Cov}(X^2, X) - \frac{2}{n^2}\mu_X \operatorname{Var}(X)\\
    \mathbb{E}\Bigl[\Bigl( \bar{X} - \mu_X \Bigr)^3 \Bigl(\bar{Y}-\mu_Y \Bigr) \Bigr]&= \operatorname{Cov}(\bar{X}^3, \bar{Y}) - 3\mu_X \operatorname{Cov}(\bar{X}^2, \bar{Y)} + 3 \mu_X^2 \operatorname{Cov}(\bar{X}, \bar{Y})\\
    &=\frac{3}{n^2}\operatorname{Var(X)} \operatorname{Cov}(X, Y)+ \mathcal{O}(n^{-3})\\
    \mathbb{E}\biggl[\Bigl(\bar{X}-\mu_X\Bigr)^4\biggr]&=\operatorname{Cov}(\bar{X}^3, \bar{X}) - 3\mu_X \operatorname{Cov}(\bar{X}^2, \bar{X)} + 3 \mu_X^2 \operatorname{Var}(\bar{X})\\
    &=\frac{3}{n^2}\operatorname{Var}(X)^2+ \mathcal{O}(n^{-3})
\end{align}
\paragraph{Bias}
Using these expressions we can compute the expectation value of $r=\hat{R}$:
\begin{align}
    \begin{split}
        \mathbb{E}[r]&\approx R\Biggl(1+\frac{1}{n}\biggl(\frac{\operatorname{Var(X)}}{\mu_X^2} - \frac{\operatorname{Cov}(X, Y)}{\mu_X \mu_Y} \biggr) + \frac{1}{n^2}\biggl(\frac{(\operatorname{Cov}(X^2, Y) -2\mu_X \operatorname{Cov}(X, Y))}{\mu_X^2 \mu_Y}\\
        &\phantom{asdfa}-\frac{(\operatorname{Cov}(X^2, X) - 2\mu_X \operatorname{Var}(X))}{\mu_X^3} - \frac{3\operatorname{Var}(X) \operatorname{Cov}(X, Y)}{\mu_X^3 \mu_Y}+\frac{3 \operatorname{Var}(X)^2}{\mu_X^4}\biggr)\Biggr)
    \end{split}
\end{align}
The bias or $r=\widehat{R}$ is defined as:
\begin{align}
    \operatorname{Bias}(r) &= \mathbb{E}[r] - R\\
    &=R\Biggl(\frac{1}{n}\biggl(\frac{\operatorname{Var}(X)}{\mu_X^2} - \frac{\operatorname{Cov}(X, Y)}{\mu_X \mu_Y} \biggr) + \frac{1}{n^2}\biggl(\frac{(\operatorname{Cov}(X^2, Y) -2\mu_X \operatorname{Cov}(X, Y))}{\mu_X^2 \mu_Y}\\
        &\phantom{asdfa}-\frac{(\operatorname{Cov}(X^2, X) - 2\mu_X \operatorname{Var}(X))}{\mu_X^3} - \frac{3\operatorname{Var}(X) \operatorname{Cov}(X, Y)}{\mu_X^3 \mu_Y}+\frac{3 \operatorname{Var}(X)^2}{\mu_X^4}\biggr)\Biggr)
\end{align}
Therefore an unbiased version of $r$ is:
\begin{align}
    r_{\text{unbiased}} &= r - R\Biggl(\frac{1}{n}\biggl(\frac{\operatorname{Var}(X)}{\mu_X^2} - \frac{\operatorname{Cov}(X, Y)}{\mu_X \mu_Y} \biggr) + \frac{1}{n^2}\biggl(\frac{(\operatorname{Cov}(X^2, Y) -2\mu_X \operatorname{Cov}(X, Y))}{\mu_X^2 \mu_Y}\\
        &\phantom{asdfas}-\frac{(\operatorname{Cov}(X^2, X) - 2\mu_X \operatorname{Var}(X))}{\mu_X^3} - \frac{3\operatorname{Var}(X) \operatorname{Cov}(X, Y)}{\mu_X^3 \mu_Y}+\frac{3 \operatorname{Var}(X)^2}{\mu_X^4}\biggr)\Biggr)
\end{align}

A corrected version of the estimator $r=\hat{R}$ is consequently given by:
\begin{align}
    \begin{split}
        r_{corr} &:=  r\Biggl(1-\frac{1}{n}\biggl(\frac{\widehat{\operatorname{Var}(X)}}{\widehat{\mu_X^2}} - \frac{\widehat{\operatorname{Cov}(X, Y)}}{\widehat{\mu_X \mu_Y}} \biggr) - \frac{1}{n^2}\biggl(\frac{\widehat{(\operatorname{Cov}(X^2, Y)} -2\widehat{\mu_X} \widehat{\operatorname{Cov}(X, Y)})}{\widehat{\mu_X^2} \widehat{\mu_Y}}\\
        &\phantom{asdfas}-\frac{(\widehat{\operatorname{Cov}(X^2, X)} - 2\widehat{\mu_X} \widehat{\operatorname{Var}(X)})}{\widehat{\mu_X^3}} - \frac{3\widehat{\operatorname{Var}(X)} \widehat{\operatorname{Cov}(X, Y)}}{\widehat{\mu_X^3} \widehat{\mu_Y}}+\frac{3 \widehat{\operatorname{Var}(X)}^2}{\widehat{\mu_X^4}}\biggr)\Biggr)
    \end{split}
    \label{eq:L1_corrected_not_yet_recursed}
\end{align}
In the above equation we again encounter rations of estimators which again might be biased. Since we want to achieve a second order de-biasing we have to again recurse on the terms that have a $\mathcal{O}\left(\frac{1}{n} \right)$ dependency. However, we do not have to recurse on the terms that have a $\mathcal{O}\left(\frac{1}{n^2} \right)$ dependency, since any recursion would increase the power of the $n$-dependency.
Therefore a debiased estimator up to order $\mathcal{O}(n^2)$ is:
\begin{align}
    \begin{split}
        r_{corr} &:=  r\Biggl(1-\frac{1}{n}\biggl(r_{b}^{*} - r_{a}^{*} \biggr) - \frac{1}{n^2}\biggl(\frac{\widehat{(\operatorname{Cov}(X^2, Y)} -2\widehat{\mu_X} \widehat{\operatorname{Cov}(X, Y)})}{\widehat{\mu_X^2} \widehat{\mu_Y}}\\
        &\phantom{asdfa}-\frac{(\widehat{\operatorname{Cov}(X^2, X)} - 2\widehat{\mu_X} \widehat{\operatorname{Var}(X)})}{\widehat{\mu_X^3}} - \frac{3\widehat{\operatorname{Var}(X)} \widehat{\operatorname{Cov}(X, Y)}}{\widehat{\mu_X^3} \widehat{\mu_Y}}+\frac{3 \widehat{\operatorname{Var}(X)}^2}{\widehat{\mu_X^4}}\biggr)\Biggr)
    \end{split}
    \label{eq:L1_corrected}
\end{align}
where 
\begin{align}
\begin{split}
    r_{a}^{*} &= \underbrace{\frac{\widehat{\operatorname{Cov}(X, Y)}}{\widehat{\mu_X \mu_Y}}}_{=r_a}
    \Biggl(1+\frac{1}{(n-1)}\biggl(\frac{\widehat{\mu_Y}\widehat{\operatorname{Cov}(X^2, Y)}+\widehat{\mu_X}\widehat{\operatorname{Cov}(Y^2, X)}}{\widehat{\operatorname{Cov}(X, Y)}\widehat{\mu_X}\widehat{ \mu_Y}}-4\biggr)\\
    &\phantom{asasasasddasdf}- \frac{1}{(n-1)}\biggl(\frac{\ \widehat{\operatorname{Var}(X)}}{\widehat{\mu_X^2} }+ \frac{ \widehat{\operatorname{Var}(Y)}}{ \widehat{\mu_Y^2}} + 2\frac{ \widehat{\operatorname{Cov}(X, Y)}}{\widehat{\mu_X} \widehat{\mu_Y}}\biggr)\Biggr)
\end{split}
\label{eq:r_a_corrected_}
\end{align}
\begin{align}\label{eq:r_b_corrected_}
    r_{b}^{*} = \underbrace{\frac{\widehat{\operatorname{Var}(X)}}{\widehat{\mu_X^2}}}_{=r_b}\Biggl(1+ \frac{4}{(n-1)}\biggl(\frac{\frac{1}{2}\widehat{\operatorname{Cov}(X^2, X)}}{\widehat{\mu_X}\widehat{ \operatorname{Var}(X)}}-1\biggr)-\frac{4}{(n-1)}\frac{\widehat{\operatorname{Var}(X)}}{\widehat{\mu_X^2}}\Biggr).
\end{align}

\section{De-biasing of ratios of squared means}
\label{App:L2_debiasing}
Now consider the problem of finding an unbiased estimator for the ratio of the squared means of $x$ and $Y$:
\begin{align}
    R = \frac{\mu_Y^2}{\mu_X^2}.
\end{align}
Both the numerator and denominator of $R$ can separately be estimated by a second order U-statistics, respectively:
\begin{align}
    r = \hat{R} = \frac{\widehat{\mu_Y^2}}{\widehat{\mu_X^2}} =\frac{\frac{1}{n(n-1)}\sum_{i=1}^n \sum_{j=1 \wedge j\neq i}^n Y_i Y_j}{\frac{1}{n(n-1)}\sum_{i=1}^n \sum_{j=1 \wedge j\neq i}^n
    X_i X_j}=:\frac{\bar{Y_2}}{\bar{X_2}}.
\end{align}
The subscript 2 in $\bar{X}_2$ should emphasize that we are dealing with a second order U-statistics here.
Again, the ratio $\frac{\bar{Y}_2}{\bar{X}_2}$,is a biased estimator. Using the approach with the converging geometric series and neglecting the higher order terms, we obtain:
\begin{align}
    \begin{split}
        r&\approx R\biggl(1 + \frac{(\bar{Y}_2-\mu_Y^2)}{\mu_Y^2} - \frac{(\bar{X}_2-\mu_X^2)}{\mu_X^2} - \frac{(\bar{X}_2-\mu_X^2)(\bar{Y}_2-\mu_Y^2)}{\mu_Y^2 \mu_X^2} + \frac{(\bar{X}_2-\mu_X^2)^2}{\mu_X^4} \\
        & \phantom{asdfa}+ \frac{(\bar{X}_2-\mu_X^2)^2 (\bar{Y}_2-\mu_Y^2)}{\mu_X^4 \mu_Y^2} - \frac{(\bar{X}_2-\mu_X^2)^3}{\mu_X^6} - \frac{(\bar{X}_2-\mu_X^2)^3 (\bar{Y}_2-\mu_Y^2)}{\mu_X^6 \mu_Y^2}  + \frac{(\bar{X}_2-\mu_X^2)^4}{\mu_X^8}\biggr)
    \end{split}
\end{align}
\paragraph{Identities to compute the terms of the series expansion of $r$}
\begin{align}
    \mathbb{E}[\bar{X}_2-\mu_X^2]&=\mathbb{E}[\bar{Y}_2-\mu_Y^2]=0\\
    \mathbb{E}[(\bar{X}_2-\mu_X^2)^2]&=\operatorname{Var}(\bar{X}_2)\\
    \mathbb{E}\biggl[\Bigl(\bar{X}_2- \mu_X^2\Bigr)\Bigl(\bar{Y}_2- \mu_Y^2\Bigr)\biggr]&=\operatorname{Cov}(\bar{X}_2, \bar{Y}_2) \\
    \mathbb{E}\biggl[\Bigl(\bar{X}_2- \mu_X^2\Bigr)^2\Bigl(\bar{Y}_2- \mu_Y^2\Bigr)\biggr]&=\operatorname{Cov}(\bar{X}_2^2, \bar{Y}_2) -2\mu_X^2 \operatorname{Cov}(\bar{X}_2, \bar{Y}_2)\\
    \mathbb{E}\biggl[\Bigl(\bar{X}_2-\mu_X^2\Bigr)^3\biggr]&=\operatorname{Cov}(\bar{X}_2^2, \bar{X}_2)-2\mu_X^2 \operatorname{Var}(\bar{X}_2)\\
    \mathbb{E}\Bigl[\Bigl( \bar{X}_2 - \mu_X^2 \Bigr)^3 \Bigl(\bar{Y}_2-\mu_Y^2 \Bigr) \Bigr]&= \operatorname{Cov}(\bar{X}_2^3, \bar{Y}_2) - 3\mu_X^2 \operatorname{Cov}(\bar{X}_2^2, \bar{Y}_2) + 3 \mu_X^4 \operatorname{Cov}(\bar{X}_2, \bar{Y}_2)\\
    \mathbb{E}\biggl[\Bigl(\bar{X}_2-\mu_X^2\Bigr)^4\biggr]&=\operatorname{Cov}(\bar{X}_2^3, \bar{X}_2) - 3\mu_X^2 \operatorname{Cov}(\bar{X}_2^2, \bar{X}_2) + 3 \mu_X^4 \operatorname{Var}(\bar{X}_2)
\end{align}

\paragraph{Bias}
Computing $\mathbb{E}[r]$ using the above identities:
\begin{align}
    \begin{split}
        \mathbb{E}[r]&\approx R\Biggl(1-\frac{\operatorname{Cov}(\bar{X}_2, \bar{Y}_2)}{\mu_X^2 \mu_Y^2}+\frac{\operatorname{Var}(\bar{X}_2)}{\mu_X^4}+\biggl(\frac{\operatorname{Cov}(\bar{X}_2^2, \bar{Y}_2) -2\mu_X^2 \operatorname{Cov}(\bar{X}_2, \bar{Y}_2)}{\mu_X^4 \mu_Y^2}\biggr)\\
        &\phantom{asdfa}-\biggl(\frac{\operatorname{Cov}(\bar{X}_2^2, \bar{X}_2)-2\mu_X^2 \operatorname{Var}(\bar{X}_2)}{\mu_X^6}\biggr) - \biggl(\frac{\operatorname{Cov}(\bar{X}_2^3, \bar{Y}_2) - 3\mu_X^2 \operatorname{Cov}(\bar{X}_2^2, \bar{Y}_2) + 3 \mu_X^4 \operatorname{Cov}(\bar{X}_2, \bar{Y}_2)}{\mu_X^6 \mu_Y^2}\biggr)\\
        &\phantom{asdfa}+\biggl(\frac{\operatorname{Cov}(\bar{X}_2^3, \bar{X}_2) - 3\mu_X^2 \operatorname{Cov}(\bar{X}_2^2, \bar{X}_2) + 3 \mu_X^4 \operatorname{Var}(\bar{X}_2)}{\mu_X^8}\biggr)\Biggr)
    \end{split}
    \\
    \begin{split}
        &= R\Biggl(1-\overbrace{\frac{6 \operatorname{Cov}(\bar{X}_2, \bar{Y}_2)}{\mu_X^2 \mu_Y^2}}^{\text{Term (a)}} + \overbrace{\frac{6 \operatorname{Var}(\bar{X}_2)}{\mu_X^4}}^{\text{Term (b)}} + \overbrace{\frac{4 \operatorname{Cov}(\bar{X}_2^2, \bar{Y}_2)}{\mu_X^4\mu_Y^2}}^{\text{Term (c)}} -\overbrace{\frac{4 \operatorname{Cov}(\bar{X}_2^2, \bar{X}_2)}{\mu_X^6}}^{\text{Term (d)}} \\
        & \phantom{asdfa} 
        - \underbrace{\frac{\operatorname{Cov}(\bar{X}_2^3, \bar{Y}_2)}{\mu_X^6 \mu_Y^2}}_{\text{Term (e)}} + \underbrace{\frac{\operatorname{Cov}(\bar{X}_2^3, \bar{X}_2)}{\mu_X^8}}_{\text{Term (f)}}\Biggl)
    \end{split}
\end{align}
\begin{align}
    \begin{split}
        &=R\Biggl\{1-
        \Biggl(\frac{12}{n(n-1)}\frac{\operatorname{Cov}(X, Y)^2}{\mu_X^2 \mu_Y^2}+ \frac{24}{n} R_a\Biggr)
        \\
        &\phantom{asdfasd}+\Biggl(\frac{12}{n(n-1)}\frac{\operatorname{Var}(X)^2}{\mu_X^4}+\frac{24}{n}  R_b\Biggr)
        \\
        &\phantom{asdfasd}+\Biggl(\frac{32(n-2)}{n(n-1)^2}\biggl(\frac{ \operatorname{Cov}(X^2, Y)}{\mu_X^2 \mu_Y} +\frac{ 2\operatorname{Cov}(X, Y)(\operatorname{Var}(X)+\mu_X^2)}{\mu_X^3 \mu_Y}\biggr)\\
        &\phantom{asdfasdasd}+\frac{4(n-2)(n-3)}{n(n-1)}\biggl(\frac{8}{n} R_a + \frac{12}{n(n-1)} \frac{\operatorname{Cov}(X, Y)^2}{\mu_X^2 \mu_Y^2}\biggr)\Biggr)\\
        &\phantom{asdfasd}- \Biggl(\frac{32(n-2)}{n(n-1)^2}\biggl(\frac{ \operatorname{Cov}(X^2, X)}{\mu_X^3} +\frac{ 2\operatorname{Var}(X)(\operatorname{Var}(X)+\mu_X^2)}{\mu_X^4}\biggr)\\
        &\phantom{asdfasdasd}+\frac{4(n-2)(n-3)}{n(n-1)}\biggl( \frac{8}{n} R_b + \frac{12}{n(n-1)} \frac{\operatorname{Var}(X)^2}{\mu_X^4}\biggr)\Biggr)
        \\
        &\phantom{asdfasd}-\Biggl(
        \frac{24 (n-2)(n-3)(n-4)}{n^2(n-1)^3}\biggl(\frac{ \operatorname{Cov}(X^2, Y)}{\mu_Y \mu_X^2}+\frac{4 \operatorname{Cov}(X, Y)(\operatorname{Var}(X)+\mu_X^2)}{\mu_Y \mu_X^3} \biggr)\\
        &\phantom{asdfasdasd}+\frac{(n-2)(n-3)(n-4)(n-5)}{n^2(n-1)^2}\biggl(\frac{12}{n} R_a + \frac{30}{n(n-1)} \frac{\operatorname{Cov}(X, Y)^2}{\mu_X^2 \mu_Y^2} \biggr)\Biggr)
        \\
        &\phantom{asdfasd}+\Biggl(
        \frac{24 (n-2)(n-3)(n-4)}{n^2(n-1)^3} \biggl(\frac{ \operatorname{Cov}(X^2, X)}{\mu_X^3}+\frac{4  \operatorname{Var}(X)(\operatorname{Var}(X)+\mu_X^2)}{\mu_X^4}\biggr)\\
        &\phantom{asdfasdasd}+\frac{(n-2)(n-3)(n-4)(n-5)}{n^2(n-1)^2}\biggl(\frac{12}{n} R_b + \frac{30}{n(n-1)} \frac{\operatorname{Var}(X)^2}{\mu_X^4} \biggr)\Biggr)
    \end{split}
\end{align}
where $R_a = \frac{\operatorname{Cov(X, Y)}}{\mu_X \mu_Y}$ and $R_b = \frac{\operatorname{Var}(X)}{\mu_X^2}$
and where we have used \ref{eq:term_a_corrected}, \ref{eq:term_b_corrected} \ref{eq:term_c_corrected}, \ref{eq:term_d_corrected}, \ref{eq:term_e_corrected} and \ref{eq:term_f_corrected} for terms (a)-(f). 

Therefore, an estimator unbiased up to order two is given by:
\begin{align}
    \begin{split}
        r_{\text{corr}}&=\frac{\widehat{\mu_Y^2}}{\widehat{\mu_X^2}}\Biggl\{1+
        \Biggl(\frac{12}{n(n-1)}\frac{\widehat{\operatorname{Cov}(X, Y)^2}}{\widehat{\mu_X^2} \widehat{\mu_Y^2}}+ \frac{24}{n} r_{a}^{*}\Biggr)
        \\
        &\phantom{asdfasd}-\Biggl(\frac{12}{n(n-1)}\frac{\widehat{\operatorname{Var}(X)^2}}{\widehat{\mu_X^4}}+\frac{24}{n}  r_{b}^{*}\Biggr)
        \\
        &\phantom{asdfasd}-\Biggl(\frac{32(n-2)}{n(n-1)^2}\biggl(\frac{ \widehat{\operatorname{Cov}(X^2, Y)}}{\widehat{\mu_X^2} \widehat{\mu_Y}} +\frac{ 2\widehat{\operatorname{Cov}(X, Y)}(\widehat{\operatorname{Var}(X)}+\widehat{\mu_X^2)}}{\widehat{\mu_X^3} \widehat{\mu_Y}}\biggr)\\
        &\phantom{asdfasdasd}+\frac{4(n-2)(n-3)}{n(n-1)}\biggl(\frac{8}{n} r_{a}^{*} + \frac{12}{n(n-1)} \frac{\widehat{\operatorname{Cov}(X, Y)^2}}{\widehat{\mu_X^2} \widehat{\mu_Y^2}}\biggr)\Biggr)
        \\
        &\phantom{asdfasd}+\Biggl(\frac{32(n-2)}{n(n-1)^2}\biggl(\frac{ \widehat{\operatorname{Cov}(X^2, X)}}{\widehat{\mu_X^3}} +\frac{ 2\widehat{\operatorname{Var}(X)}(\widehat{\operatorname{Var}(X)}+\widehat{\mu_X^2})}{\widehat{\mu_X^4}}\biggr)\\
        &\phantom{asdfasdasd}+\frac{4(n-2)(n-3)}{n(n-1)}\biggl( \frac{8}{n} r_{b}^{*} + \frac{12}{n(n-1)} \frac{\widehat{\operatorname{Var}(X)^2}}{\widehat{\mu_X^4}}\biggr)\Biggr)
        \\
        &\phantom{asdfasd}+\Biggl(\frac{24 (n-2)(n-3)(n-4)}{n^2(n-1)^3}\biggl(\frac{ \widehat{\operatorname{Cov}(X^2, Y)}}{\widehat{\mu_Y} \widehat{\mu_X^2}}+\frac{4 \widehat{\operatorname{Cov}(X, Y})(\widehat{\operatorname{Var}(X)}+\widehat{\mu_X^2})}{\widehat{\mu_Y} \widehat{\mu_X^3}} \biggr)\\
        &\phantom{asdfasdasd}+\frac{(n-2)(n-3)(n-4)(n-5)}{n^2(n-1)^2}\biggl(\frac{12}{n} r_{a}^{*} + \frac{30}{n(n-1)} \frac{\widehat{\operatorname{Cov}(X, Y)^2}}{\widehat{\mu_X^2} \widehat{\mu_Y^2}}\Biggr) \biggr)
        \\
        &\phantom{asdfasd}-\Biggl(\frac{24 (n-2)(n-3)(n-4)}{n^2(n-1)^3} \biggl(\frac{ \widehat{\operatorname{Cov}(X^2, X)}}{\widehat{\mu_X^3}}+\frac{4  \widehat{\operatorname{Var}(X)}(\widehat{\operatorname{Var}(X)}+\widehat{\mu_X^2)}}{\widehat{\mu_X^4}}\biggr)\\
        &\phantom{asdfasdasd}+\frac{(n-2)(n-3)(n-4)(n-5)}{n^2(n-1)^2}\biggl(\frac{12}{n} r_{b}^{*} + \frac{30}{n(n-1)} \frac{\widehat{\operatorname{Var}(X)^2}}{\widehat{\mu_X^4}} \biggr)\Biggr),
    \end{split}
\end{align}
where we used equations (\ref{eq:r_a_corrected}), (\ref{eq:r_b_corrected}):
\begin{align}
\begin{split}
    r_{a}^{*} &= \underbrace{\frac{\widehat{\operatorname{Cov}(X, Y)}}{\widehat{\mu_X \mu_Y}}}_{=r_a}
    \Biggl(1+\frac{1}{(n-1)}\biggl(\frac{\widehat{\mu_Y}\widehat{\operatorname{Cov}(X^2, Y)}+\widehat{\mu_X}\widehat{\operatorname{Cov}(Y^2, X)}}{\widehat{\operatorname{Cov}(X, Y)}\widehat{\mu_X}\widehat{ \mu_Y}}-4\biggr)\\
    &\phantom{asasasasddasdf}- \frac{1}{(n-1)}\biggl(\frac{\ \widehat{\operatorname{Var}(X)}}{\widehat{\mu_X^2} }+ \frac{ \widehat{\operatorname{Var}(Y)}}{ \widehat{\mu_Y^2}} + 2\frac{ \widehat{\operatorname{Cov}(X, Y)}}{\widehat{\mu_X} \widehat{\mu_Y}}\biggr)\Biggr)
\end{split}
\end{align}
\begin{align}
    r_{b}^{*} = \underbrace{\frac{\widehat{\operatorname{Var}(X)}}{\widehat{\mu_X^2}}}_{=r_b}\Biggl(1+ \frac{4}{(n-1)}\biggl(\frac{\frac{1}{2}\widehat{\operatorname{Cov}(X^2, X)}}{\widehat{\mu_X}\widehat{ \operatorname{Var}(X)}}-1\biggr)-\frac{4}{(n-1)}\frac{\widehat{\operatorname{Var}(X)}}{\widehat{\mu_X^2}}\Biggr).
\end{align}

\subsection{Term (a)}
\label{subsec:term_a}
Let us first look at the first term: $\frac{6 \operatorname{Cov}(\bar{X}_2, \bar{Y}_2)}{\mu_X^2 \mu_Y^2}$. Using the expression for the covariance between two second order U-statistics we get:
\begin{align}
    \frac{6 \operatorname{Cov}(\bar{X}_2, \bar{Y}_2)}{\mu_X^2 \mu_Y^2}&=\frac{6}{\mu_X^2 \mu_Y^2}\biggl(\frac{4}{n}\mu_X \mu_Y \operatorname{Cov(X, Y)} + \frac{2}{n(n-1)} \operatorname{Cov}(X, Y)^2\biggr)\\
    &=\underbrace{\frac{24}{n}\frac{\operatorname{Cov(X, Y)}}{\mu_X \mu_Y}}_{\in \mathcal{O}(n^{-1})}+\underbrace{\frac{12}{n(n-1)}\frac{\operatorname{Cov}(X, Y)^2}{\mu_X^2 \mu_Y^2}}_{\in \mathcal{O}(n^{-2})}
\end{align}
Since we know that for every recursion (i.e., geometric series expansion) we will get at least another factor of $\frac{1}{n}$, we don't have to further recurse on term that is of order $\mathcal{O}(n^{-2})$. Consequently, we only expand the following term via a geometric series,
\begin{align}
    R_a = \frac{\operatorname{Cov}(X, Y)}{\mu_X \mu_Y},
\end{align}
since the ratio of the respective unbiased estimators,
\begin{align}
    r_a = \frac{\widehat{\operatorname{Cov}(X, Y)}}{\widehat{\mu_X \mu_Y}},
\end{align}
is biased.\\
Using the same machinery as before, we obtain a corrected version of $r_a$:
\begin{align}
\begin{split}
    r_{a}^{*} &= \underbrace{\frac{\widehat{\operatorname{Cov}(X, Y)}}{\widehat{\mu_X \mu_Y}}}_{=r_a}
    \Biggl(1+\frac{1}{(n-1)}\biggl(\frac{\widehat{\mu_Y}\widehat{\operatorname{Cov}(X^2, Y)}+\widehat{\mu_X}\widehat{\operatorname{Cov}(Y^2, X)}}{\widehat{\operatorname{Cov}(X, Y)}\widehat{\mu_X}\widehat{ \mu_Y}}-4\biggr)\\
    &\phantom{asasasasddasdf}- \frac{1}{(n-1)}\biggl(\frac{\ \widehat{\operatorname{Var}(X)}}{\widehat{\mu_X^2} }+ \frac{ \widehat{\operatorname{Var}(Y)}}{ \widehat{\mu_Y^2}} + 2\frac{ \widehat{\operatorname{Cov}(X, Y)}}{\widehat{\mu_X} \widehat{\mu_Y}}\biggr)\Biggr)
\end{split}
\label{eq:r_a_corrected}
\end{align}
The complete correction of term (a), $\frac{6 \operatorname{Cov}(\bar{X}_2, \bar{Y}_2)}{\mu_X^2 \mu_Y^2}$, looks therefore as follows:
\begin{align}
    \begin{split}
        \frac{6 \widehat{\operatorname{Cov}(\bar{X}_2, \bar{Y}_2)}}{\widehat{\mu_X^2} \widehat{\mu_Y^2}}&=
        \frac{12}{n(n-1)}\frac{\widehat{\operatorname{Cov}(X, Y)^2}}{\widehat{\mu_X^2} \widehat{\mu_Y^2}}\\
        &\phantom{as}+ \frac{24}{n}
        \frac{\widehat{\operatorname{Cov}(X, Y)}}{\widehat{\mu_X \mu_Y}}
        \Biggl(1+\frac{1}{(n-1)}\biggl(\frac{\widehat{\mu_Y}\widehat{\operatorname{Cov}(X^2, Y)}+\widehat{\mu_X}\widehat{\operatorname{Cov}(Y^2, X)}}{\widehat{\operatorname{Cov}(X, Y)}\widehat{\mu_X} \widehat{\mu_Y}}-4\biggr)\\
        &\phantom{asasasasdfasddasdf}- \frac{1}{(n-1)}\biggl(\frac{\ \widehat{\operatorname{Var}(X)}}{\widehat{\mu_X^2} }+ \frac{\widehat{ \operatorname{Var}(Y)}}{\widehat{ \mu_Y^2}} + 2\frac{\widehat{ \operatorname{Cov}(X, Y)}}{\widehat{\mu_X} \widehat{\mu_Y}}\biggr)\Biggr)
    \end{split}
    \label{eq:term_a_corrected}
\end{align}

\subsection{Term (b)}
The correction of term (b), $\frac{6 \operatorname{Var}(\bar{X}_2)}{\mu_X^4}$ is analogous to that of term (a). Define
\begin{align}
    R_b &=\frac{\operatorname{Var(X)}}{\mu_X^2}\\
    r_b &=\frac{\widehat{\operatorname{Var(X)}}}{\widehat{\mu_X^2}}.
\end{align}
Then using the geometric series expansion, a corrected version of $r_b$ is given by 
\begin{align}
    r_{b}^{*} = \underbrace{\frac{\widehat{\operatorname{Var}(X)}}{\widehat{\mu_X^2}}}_{=r_b}\Biggl(1+ \frac{4}{(n-1)}\biggl(\frac{\frac{1}{2}\widehat{\operatorname{Cov}(X^2, X)}}{\widehat{\mu_X}\widehat{ \operatorname{Var}(X)}}-1\biggr)-\frac{4}{(n-1)}\frac{\widehat{\operatorname{Var}(X)}}{\widehat{\mu_X^2}}\Biggr).
    \label{eq:r_b_corrected}
\end{align}
The full correction of term (b) is
\begin{align}
    \frac{6\widehat{\operatorname{Var}(\bar{X}_2)}}{\widehat{\mu_X^4}} &= \frac{12}{n(n-1)}\frac{\widehat{\operatorname{Var}(X)^2}}{\widehat{\mu_X^4}} +\frac{24}{n}  \frac{\widehat{\operatorname{Var}(X)}}{\widehat{\mu_X^2}}\Biggl(1+ \frac{4}{(n-1)}\biggl(\frac{\frac{1}{2}\widehat{\operatorname{Cov}(X^2, X)}}{\widehat{\mu_X}\widehat{ \operatorname{Var}(X)}}-1-\frac{\widehat{\operatorname{Var}(X)}}{\widehat{\mu_X^2}}\biggr)\Biggr)
    \label{eq:term_b_corrected}
\end{align}

\subsection{Term (c)}
In this section we want to find an expression for term (c):
\begin{align}
    \frac{4\operatorname{Cov}(\bar{X}_2^2, \bar{Y}_2)}{\mu_X^4 \mu_Y^2}.
\end{align}
To this end, we first need a convenient representation of $\bar{X}_2^2$ in terms of other U-statistics:
\begin{align}
    \bar{X}_2^2=\Biggl(\frac{1}{n(n-1)}\sum_{i=1}^n \sum_{\substack{j=1\\j \neq i}}^n X_i X_j \Biggr)^2= \frac{2}{n (n-1)} U_{\alpha}
    + \frac{4(n-2)}{n(n-1)}U_{\beta} +\frac{(n-2)(n-3)}{n(n-1)}\bar{X}_4,
\end{align}
with the U-statistics:
\begin{align}
    U_{\beta}&=\frac{1}{n(n-1)(n-2)}\sum_{i=1}^n \sum_{\substack{j=1\\j \neq i}}^n \sum_{\substack{k=1\\k \neq j\\k \neq i}}^n X_i^2 X_j X_k\\
    U_{\alpha}&=\frac{1}{n(n-1)}\sum_{i=1}^n         \sum_{\substack{j=1\\j \neq i}}^n X_i^2 X_j^2\\
    \bar{X}_4&=\frac{1}{n(n-1)(n-2)(n-3)}\sum_{i=1}^n \sum_{\substack{j=1\\j \neq i}}^n \sum_{\substack{k=1\\k \neq j\\k \neq i}}^n
        \sum_{\substack{l=1\\l \neq k\\l \neq j \\ l \neq i}}^n X_i X_j X_k X_l.
\end{align}
Hence, term (c) becomes:
\begin{align}
    \frac{4\operatorname{Cov}(\bar{X}_2^2, \bar{Y}_2)}{\mu_X^4 \mu_Y^2} = \underbrace{\frac{8}{n (n-1)} \frac{\operatorname{Cov}(U_{\alpha}, \bar{Y}_2)}{\mu_X^4 \mu_Y^2}}_{
    \text{First term}}
    + \underbrace{\frac{16(n-2)}{n(n-1)}\frac{\operatorname{Cov}(U_{\beta}, \bar{Y}_2)}{\mu_X^4 \mu_Y^2}}_{\text{Second term}} +\underbrace{\frac{4(n-2)(n-3)}{n(n-1)}\frac{\operatorname{Cov}(\bar{X}_4, \bar{Y}_2)}{\mu_X^4 \mu_Y^2}}_{\text{Third term}} 
\end{align}
All all the covariances in the above equation are covariances between U-statistics which are $\mathcal{O}\left(\frac{1}{n}\right)$. Therefore, the first term, which already has an explicit $\mathcal{O}\left(\frac{1}{n^2}\right)$ dependence, can be neglected entirely. The second term has an explicit  $\mathcal{O}\left(\frac{1}{n}\right)$, combined with the $\mathcal{O}\left(\frac{1}{n}\right)$ from the covariance this is in total a $\mathcal{O}\left(\frac{1}{n^2}\right)$ dependency. Hence, we have to find an estimator for that term but do not have to recurse on it. On the last term, we do have to recurse, however, we have derived the recursion already in equation (\ref{eq:r_a_corrected}). 
We can rewrite the above equation using the symmetrized U-statistics
\begin{align}
    U_{\beta}&=\frac{1}{n(n-1)(n-2)}\sum_{i=1}^n \sum_{\substack{j=1\\j \neq i}}^n \sum_{\substack{k=1\\k \neq j\\k \neq i}}^n \frac{1}{3} \Bigl(X_i^2 X_j X_k + X_i X_j^2 X_k + X_i X_j X_k^2).
\end{align}
\begin{align}
\begin{split}
    \frac{4\operatorname{Cov}(\bar{X}_2^2, \bar{Y}_2)}{\mu_X^4 \mu_Y^2}&\approx
    \underbrace{\frac{32(n-2)}{n(n-1)^2}\Biggl(\frac{ \operatorname{Cov}(X^2, Y)}{\mu_X^2 \mu_Y} +\frac{ 2\operatorname{Cov}(X, Y)(\operatorname{Var}(X)+\mu_X^2)}{\mu_X^3 \mu_Y}\Biggr)}_{\text{Second term}}\\
    &\phantom{asdfasdf}+\underbrace{\frac{4(n-2)(n-3)}{n(n-1)}\biggl(\frac{8}{n} \frac{\operatorname{Cov}(X, Y)}{\mu_X \mu_Y} + \frac{12}{n(n-1)} \frac{\operatorname{Cov}(X, Y)^2}{\mu_X^2 \mu_Y^2}\biggr)}_{\text{Third term}}
\end{split}
\end{align}
Taking the recursion of the third term into account, the total correction of term (c) is:
\begin{align}
\begin{split}
    \frac{4\widehat{\operatorname{Cov}(\bar{X}_2^2, \bar{Y}_2)}}{\widehat{\mu_X^4} \widehat{\mu_Y^2}}&\approx
    \frac{32(n-2)}{n(n-1)^2}\Biggl(\frac{ \widehat{\operatorname{Cov}(X^2, Y)}}{\widehat{\mu_X^2} \widehat{\mu_Y}} +\frac{ 2\widehat{\operatorname{Cov}(X, Y)}(\widehat{\operatorname{Var}(X)}+\widehat{\mu_X^2)}}{\widehat{\mu_X^3} \widehat{\mu_Y}}\Biggr)\\
    &\phantom{asdf}+\frac{4(n-2)(n-3)}{n(n-1)}\biggl(\frac{8}{n} r_{a}^{*} + \frac{12}{n(n-1)} \frac{\widehat{\operatorname{Cov}(X, Y)^2}}{\widehat{\mu_X^2} \widehat{\mu_Y^2}}\biggr) 
\end{split}
\label{eq:term_c_corrected}
\end{align}

\subsection{Term (d)}
The computation of the correction of term (d), $\frac{4\operatorname{Cov}(\bar{X}_2^2, \bar{X}_2)}{\mu_X^6}$, is similar to that of term (c). Hence, we only present the resulting correction:
\begin{align}
\begin{split}
    \frac{4\widehat{\operatorname{Cov}(\bar{X}_2^2, \bar{X}_2)}}{\widehat{\mu_X^4} \widehat{\mu_Y^2}}&\approx
    \frac{32(n-2)}{n(n-1)^2}\Biggl(\frac{ \widehat{\operatorname{Cov}(X^2, X)}}{\widehat{\mu_X^3}} +\frac{ 2\widehat{\operatorname{Var}(X)}(\widehat{\operatorname{Var}(X)}+\widehat{\mu_X^2})}{\widehat{\mu_X^4}}\Biggr)\\
    &\phantom{asdfasdf}+\frac{4(n-2)(n-3)}{n(n-1)}\Biggl( \frac{8}{n} r_{b}^{*} + \frac{12}{n(n-1)} \frac{\widehat{\operatorname{Var}(X)^2}}{\widehat{\mu_X^4}}\Biggr)
\end{split}
\label{eq:term_d_corrected}
\end{align}

\subsection{Term (e)}
Term (e) is:
\begin{align}
    \frac{\operatorname{Cov}(\bar{X}_2^3, \bar{Y}_2)}{\mu_X^6 \mu_Y^2}
\end{align}
To be able to compute that term, we reexpress the numerator in terms of several U-statistics:
\begin{align}
\begin{split}
    \bar{X}_2^3 &= \frac{4}{n^2(n-1)^2}U_{I}+\frac{24(n-2)}{n^2(n-1)^2}U_{II}+\frac{8(n-2)}{n^2(n-1)^2}U_{III}+\frac{8(n-2)(n-3)}{n^2(n-1)^2}U_{IV}\\
    &+\frac{30 (n-2)(n-3)}{n^2(n-1)^2}U_{V}+\frac{12 (n-2)(n-3)(n-4)}{n^2(n-1)2}U_{VI}+\frac{(n-2)(n-2)(n-4)(n-5)}{n^2(n-1)^2}\bar{X}_6,
\end{split}
\end{align}
where 
\begin{align}
    U_{I}:=\frac{1}{n(n-1)}\sum_{i=1}^n \sum_{\substack{j=1\\j \neq i}}^n X_i^3 X_j^3,
\end{align}

\begin{align}
    U_{II}:=\frac{1}{n(n-1)(n-2)}\sum_{i=1}^n \sum_{\substack{j=1\\j \neq i}}^n \sum_{\substack{k=1\\k \neq j\\k \neq i}}^n X_i^3 X_j^2 X_k,
\end{align}

\begin{align}
    U_{III}=\frac{1}{n(n-1)(n-2)}\sum_{i=1}^n \sum_{\substack{j=1\\j \neq i}}^n \sum_{\substack{k=1\\k \neq j\\k \neq i}}^n X_i^2, X_j^2 X_k^2
\end{align}

\begin{align}
    U_{IV}:=\frac{1}{n(n-1)(n-2)(n-3)}\sum_{i=1}^n \sum_{\substack{j=1\\j \neq i}}^n \sum_{\substack{k=1\\k \neq j\\k \neq i}}^n
        \sum_{\substack{l=1\\l \neq k\\l \neq j \\ l \neq i}}^n X_i^3 X_j X_k X_l,
\end{align}

\begin{align}
    U_{V}:=\frac{1}{n(n-1)(n-2)(n-3)}\sum_{i=1}^n \sum_{\substack{j=1\\j \neq i}}^n \sum_{\substack{k=1\\k \neq j\\k \neq i}}^n \sum_{\substack{l=1\\l \neq k\\l \neq j \\ l \neq i}}^n X_i^2 X_j^2 X_k X_l,
\end{align}

\begin{align}
    U_{VI}:=\frac{1}{n(n-1)(n-2)(n-3)(n-4)}\sum_{i=1}^n \sum_{\substack{j=1\\j \neq i}}^n \sum_{\substack{k=1\\k \neq j\\k \neq i}}^n
        \sum_{\substack{l=1\\l \neq k\\l \neq j \\ l \neq i }}^n \sum_{\substack{p=1\\p\neq l \\p \neq k\\p \neq j \\ p \neq i }}^n  X_i^2 X_j X_k X_l X_p,
\end{align}

\begin{align}
    \bar{X}_6=\frac{1}{n(n-1)(n-2)(n-3)(n-4)(n-5)}\sum_{i=1}^n \sum_{\substack{j=1\\j \neq i}}^n \sum_{\substack{k=1\\k \neq j\\k \neq i}}^n
        \sum_{\substack{l=1\\l \neq k\\l \neq j \\ l \neq i }}^n \sum_{\substack{p=1\\p\neq l \\p \neq k\\p \neq j \\ p \neq i }}^n \sum_{\substack{q=1\\q\neq p \\ q\neq l \\q \neq k\\q \neq j \\ q \neq i }}^n X_i X_j X_k X_l X_p X_q, 
\end{align}
Hence term (e) can be written as:
\begin{align}
\begin{split}
    \frac{\operatorname{Cov}(\bar{X}_2^3, \bar{Y}_2)}{\mu_X^6 \mu_Y^2} &= \underbrace{\frac{4}{n^2(n-1)^2}}_{\in \mathcal{O}\left(\frac{1}{n^4}\right)}\frac{\operatorname{Cov}(U_{I}, \bar{Y}_2)}{\mu_X^6 \mu_Y^2}+\underbrace{\frac{24(n-2)}{n^2(n-1)^2}}_{\in \mathcal{O}\left(\frac{1}{n^3}\right)}\frac{\operatorname{Cov}(U_{II}, \bar{Y}_2)}{\mu_X^6 \mu_Y^2}+\underbrace{\frac{8(n-2)}{n^2(n-1)^2}}_{\in \mathcal{O}\left(\frac{1}{n^3}\right)}\frac{\operatorname{Cov}(U_{III}, \bar{Y}_2)}{\mu_X^6 \mu_Y^2}\\
    &+\underbrace{\frac{8(n-2)(n-3)}{n^2(n-1)^2}}_{\in \mathcal{O}\left(\frac{1}{n^2}\right)}\frac{\operatorname{Cov}(U_{IV}, \bar{Y}_2)}{\mu_X^6 \mu_Y^2}
    +\underbrace{\frac{30 (n-2)(n-3)}{n^2(n-1)^2}}_{\in \mathcal{O}\left(\frac{1}{n^2}\right)}\frac{\operatorname{Cov}(U_{V}, \bar{Y}_2)}{\mu_X^6 \mu_Y^2}\\
    &+\underbrace{\frac{12 (n-2)(n-3)(n-4)}{n^2(n-1)^2}}_{\in \mathcal{O}\left(\frac{1}{n}\right)}\frac{\operatorname{Cov}(U_{VI}, \bar{Y}_2)}{\mu_X^6 \mu_Y^2}+\underbrace{\frac{(n-2)(n-3)(n-4)(n-5)}{n^2(n-1)^2}}_{\in \mathcal{O}\left(1\right)}\frac{\operatorname{Cov}(\bar{X}_6, \bar{Y}_2)}{\mu_X^6 \mu_Y^2}.
\end{split}
\end{align}
At this point, we can immediately discard the first three terms as they are at least $ \mathcal{O}\left(\frac{1}{n^3}\right)$ and so can directly be neglected for a second order correction. In addition, as we are dealing with covariances between U-statistics they add another $ \mathcal{O}\left(\frac{1}{n}\right)$. Therefore, the fourth and fifth term are actually  $\mathcal{O}\left(\frac{1}{n}\right) \mathcal{O}\left(\frac{1}{n^2}\right)=\mathcal{O}\left(\frac{1}{n^3}\right)$, so they can be neglected as well. Only the last and the second to last term remain:
\begin{align}
\begin{split}
    \frac{\operatorname{Cov}(\bar{X}_2^3, \bar{Y}_2)}{\mu_X^6 \mu_Y^2} \approx
    \underbrace{\frac{12 (n-2)(n-3)(n-4)}{n^2(n-1)^2}\frac{\operatorname{Cov}(U_{VI}, \bar{Y}_2)}{\mu_X^6 \mu_Y^2}}_{\text{Sixth term}}+\underbrace{\frac{(n-2)(n-3)(n-4)(n-5)}{n^2(n-1)^2}\frac{\operatorname{Cov}(\bar{X}_6, \bar{Y}_2)}{\mu_X^6 \mu_Y^2}}_{\text{Seventh term}}.
\end{split}
\end{align}
Re-expressing the covariances between U-statistics as covariances between random variables $X$ and $Y$ (and using the symmetrized version of $U_{VI}$), we obtain:
\begin{align}
\begin{split}
    \frac{\operatorname{Cov}(\bar{X}_2^3, \bar{Y}_2)}{\mu_X^6 \mu_Y^2}&\approx
    \underbrace{\frac{24 (n-2)(n-3)(n-4)}{n^2(n-1)^3}\biggl(\frac{ \operatorname{Cov}(X^2, Y)}{\mu_Y \mu_X^2}+\frac{4 \operatorname{Cov}(X, Y)(\operatorname{Var}(X)+\mu_X^2)}{\mu_Y \mu_X^3} \biggr)}_{\text{Sixth term}}\\
    &+\underbrace{\frac{(n-2)(n-3)(n-4)(n-5)}{n^2(n-1)^2}\biggl(\frac{12}{n} \frac{\operatorname{Cov}(X, Y)}{\mu_X \mu_Y} + \frac{30}{n(n-1)} \frac{\operatorname{Cov}(X, Y)^2}{\mu_X^2 \mu_Y^2} \biggr)}_{\text{Seventh term}}
\end{split}
\end{align}
Since the term $\frac{12}{n} \frac{\operatorname{Cov}(X, Y)}{\mu_X \mu_Y}$ is in $\mathcal{O}\left(\frac{1}{n}\right)$ we have to recurse on it. However, we already have derived its correction in equation (\ref{eq:r_a_corrected}). Therefore, the total correction of term (e) comes down to:
\begin{align}
\begin{split}
    \frac{\widehat{\operatorname{Cov}(\bar{X}_2^3, \bar{Y}_2)}}{\widehat{\mu_X^6} \widehat{\mu_Y^2}}&=
    \frac{24 (n-2)(n-3)(n-4)}{n^2(n-1)^3}\biggl(\frac{ \widehat{\operatorname{Cov}(X^2, Y)}}{\widehat{\mu_Y} \widehat{\mu_X^2}}+\frac{4 \widehat{\operatorname{Cov}(X, Y})(\widehat{\operatorname{Var}(X)}+\widehat{\mu_X^2})}{\widehat{\mu_Y} \widehat{\mu_X^3}} \biggr)\\
    &+\frac{(n-2)(n-3)(n-4)(n-5)}{n^2(n-1)^2}\biggl(\frac{12}{n} r_{a}^{*} + \frac{30}{n(n-1)} \frac{\widehat{\operatorname{Cov}(X, Y)^2}}{\widehat{\mu_X^2} \widehat{\mu_Y^2}} \biggr)
\end{split}
\label{eq:term_e_corrected}
\end{align}

\subsection{Term (f)}
Term (f) is:
\begin{align}
    \frac{\operatorname{Cov}(\bar{X}_2^3, \bar{X}_2)}{\mu_X^8}
\end{align}
The procedure to obtain its correction is analogous to that of term (e), hence we only present the result:
\begin{align}
\begin{split}
    \frac{\widehat{\operatorname{Cov}(\bar{X}_2^3, \bar{X}_2)}}{\widehat{\mu_X^8} }&=\frac{24 (n-2)(n-3)(n-4)}{n^2(n-1)^3} \biggl(\frac{ \widehat{\operatorname{Cov}(X^2, X)}}{\widehat{\mu_X^3}}+\frac{4  \widehat{\operatorname{Var}(X)}(\widehat{\operatorname{Var}(X)}+\widehat{\mu_X^2)}}{\widehat{\mu_X^4}}\biggr)\\
    &+\frac{(n-2)(n-3)(n-4)(n-5)}{n^2(n-1)^2}\biggl(\frac{12}{n} r_{b}^{*} + \frac{30}{n(n-1)} \frac{\widehat{\operatorname{Var}(X)^2}}{\widehat{\mu_X^4}} \biggr)
\end{split}
\label{eq:term_f_corrected}
\end{align}

\section{Top-label calibration}
\label{appendix:results}
Following standard practice in related work on calibration, we report the $L_1$ $ECE^{bin}$ for top-label (also called confidence) calibration on CIFAR-10/100. $ECE^{bin}$ was calculated using 15 bins and an adaptive width binning scheme, which determines the bin sizes so that an equal number of samples fall into each bin \citep{nguyen2015, mukhoti2020}. The 95\% confidence intervals for  $ECE^{bin}$ are obtained using 100 bootstrap samples as in \citet{kumar2019}. In all experiments with calibration regularized training, the biased version of $ECE^{KDE}$ was used. 

Table~\ref{table:top_label} summarizes our evaluation of the efficacy of KDE-XE in lowering the calibration error over the baseline XE on CIFAR-10 and CIFAR-100. The best performing $\lambda$ coefficient for KDE-XE is shown in the brackets. The results show that KDE-XE consistently reduces the calibration error, without dropping the accuracy. Figure \ref{fig:toplabel_calibration} depicts the $L_2$ $ECE^{bin}$ for several choices of the $\lambda$ parameter for KDE-XE, using ResNet-110 (SD) on CIFAR-10/100. 
Figure \ref{fig:reliability_plot_toplabel_calibration} shows reliability diagrams with 10 bins for top-label calibration on CIFAR-100 using ResNet and Wide-ResNet. Comapared to XE, we notice that KDE-XE lowers the overconfident predictions, and obtains better calibration than MMCE ($\lambda=2$) and FL-53 on average, as summarized by the ECE value in the gray box.

\begin{table*}[ht]
	\centering
	\footnotesize
	\caption{Top-label $L_1$ adaptive-width $ECE^{bin}$ and accuracy for XE and KDE-XE for various architectures on CIFAR-10/100. Best ECE values are marked in bold. The value in the brackets represent the value of the $\lambda$ parameter.}
	\resizebox{0.95\linewidth}{!}{%
		\begin{tabular}{cccccc}
			\toprule
			Dataset & Model & \multicolumn{2}{c}{$ECE^{bin}$} & \multicolumn{2}{c}{Accuracy} \\
			&& XE & KDE-XE & XE & KDE-XE  \\
			\midrule
			\multirow{2}{*}{CIFAR-10} 
			& ResNet-110 & 3.890 $\pm$ 0.602 & \textbf{3.093} $\pm$ 0.604 (0.001) & 0.925 $\pm$ 0.005 & 0.930 $\pm$ 0.005 \\ 
			& ResNet-110 (SD) & 3.555 $\pm$ 0.623 & \textbf{2.778} $\pm$ 0.468 (0.01) & 0.926 $\pm$ 0.005 & 0.932 $\pm$ 0.005 \\
			\midrule
			\multirow{4}{*}{CIFAR-100} 
			& ResNet-110 & 12.769 $\pm$ 0.784 & \textbf{8.969} $\pm$ 1.047 (0.2) & 0.700 $\pm$ 0.009 & 0.696 $\pm$ 0.009 \\
			& ResNet-110 (SD) & 11.175 $\pm$ 0.642 & \textbf{7.828} $\pm$ 0.814 (0.001) & 0.728 $\pm$ 0.009 & 0.721 $\pm$ 0.009 \\ 
			& Wide-ResNet-28-10 & 7.279 $\pm$ 0.876 & \textbf{3.703} $\pm$ 1.086 (0.5) & 0.762 $\pm$ 0.008 & 0.770 $\pm$ 0.008 \\
			& DenseNet-40 & 9.196 $\pm$ 0.881 & \textbf{8.016} $\pm$ 1.079 (0.01) & 0.756 $\pm$ 0.008 & 0.756 $\pm$ 0.008 \\
			\bottomrule
		\end{tabular}%
	}
	\label{table:top_label}
\end{table*}

\begin{figure}[ht!]
    \centering
    \subfloat[CIFAR-10]{
    \label{subfig:cifar10_resnetsd}
    \includegraphics[width=.4\linewidth]{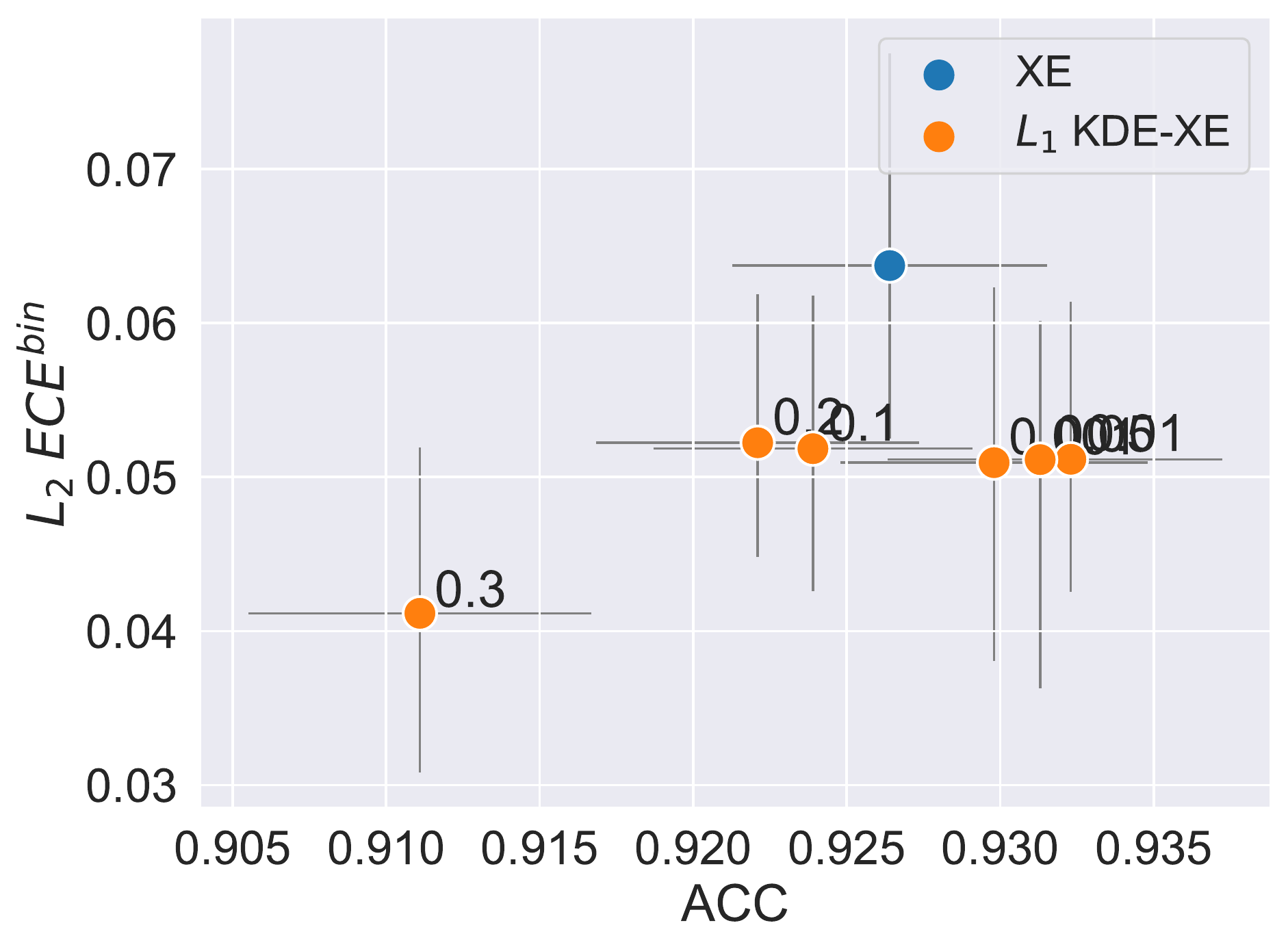}
    } 
    \subfloat[CIFAR-100]{
    \label{subfig:cifar100_resnetsd}
    \includegraphics[width=.4\linewidth]{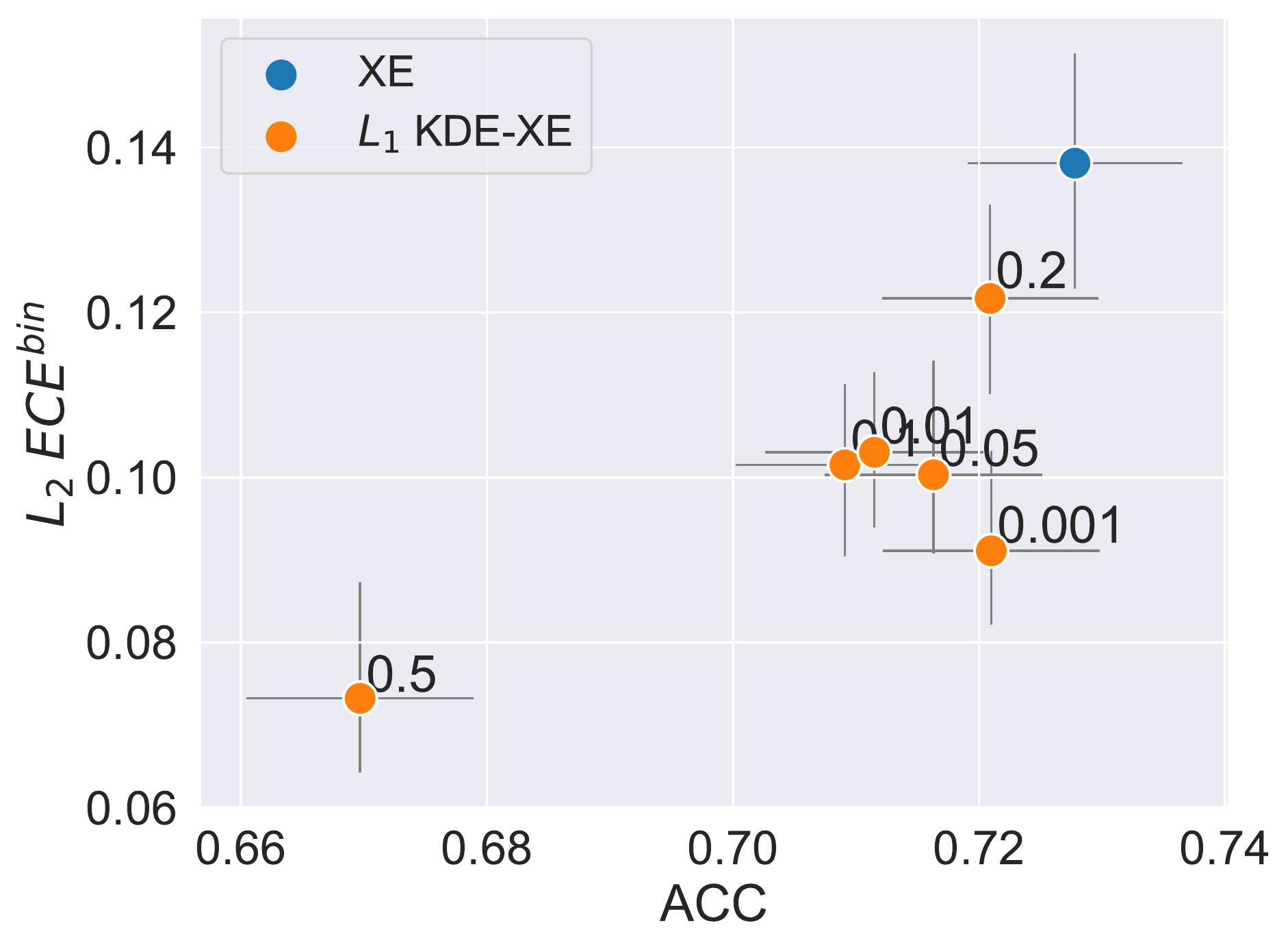} 
    } 
    \caption{$L_2$ $ECE^{bin}$ for top-label calibration using ResNet (SD).}
    \label{fig:toplabel_calibration}
\end{figure}

\begin{figure}[ht!]
    \captionsetup[subfigure]{labelformat=empty}
    \centering
    \subfloat{
    \label{subfig:cifar100_rn_xe}
    \includegraphics[width=.23\linewidth]{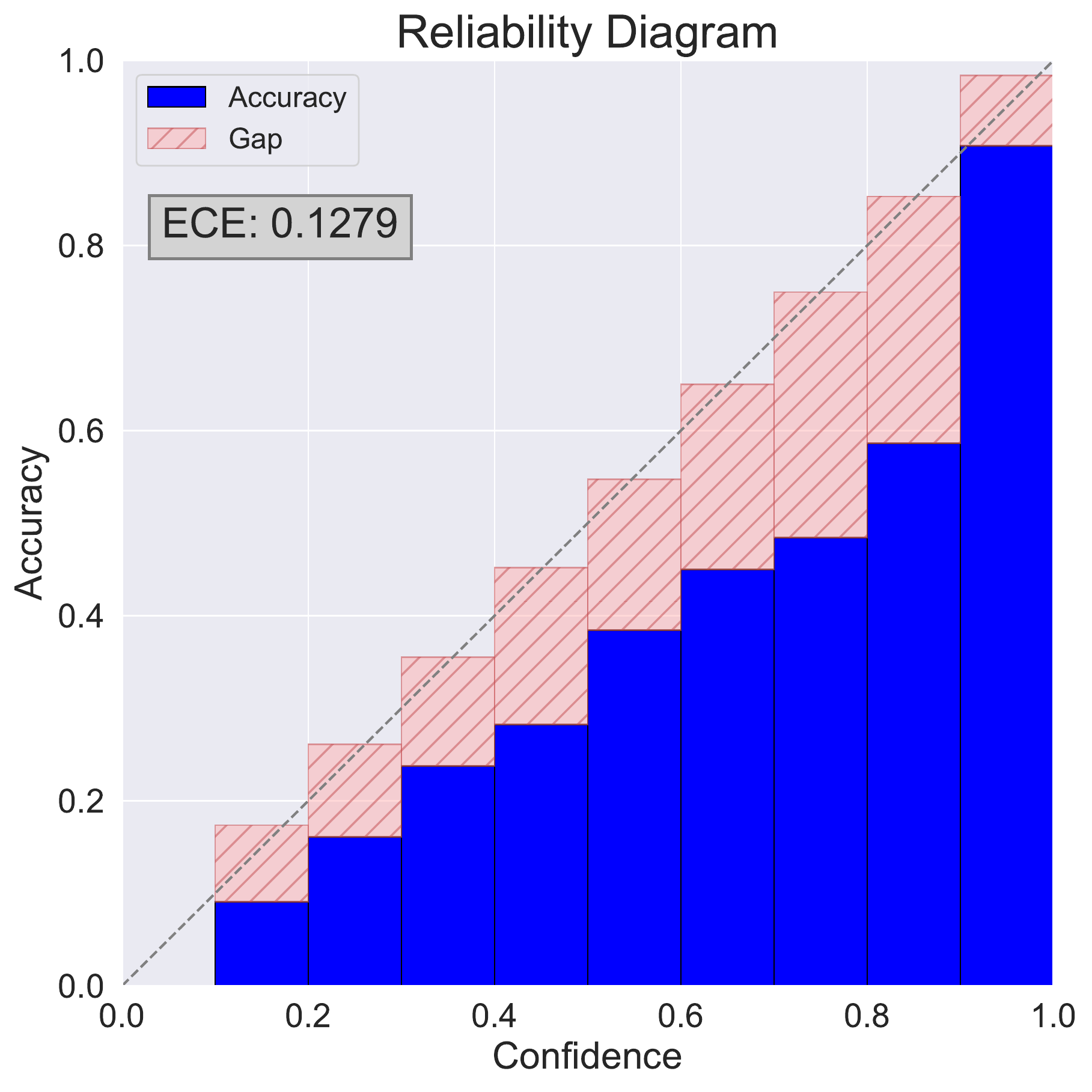}
    } 
    \subfloat{
    \label{subfig:cifar100_rn_mmce}
    \includegraphics[width=.23\linewidth]{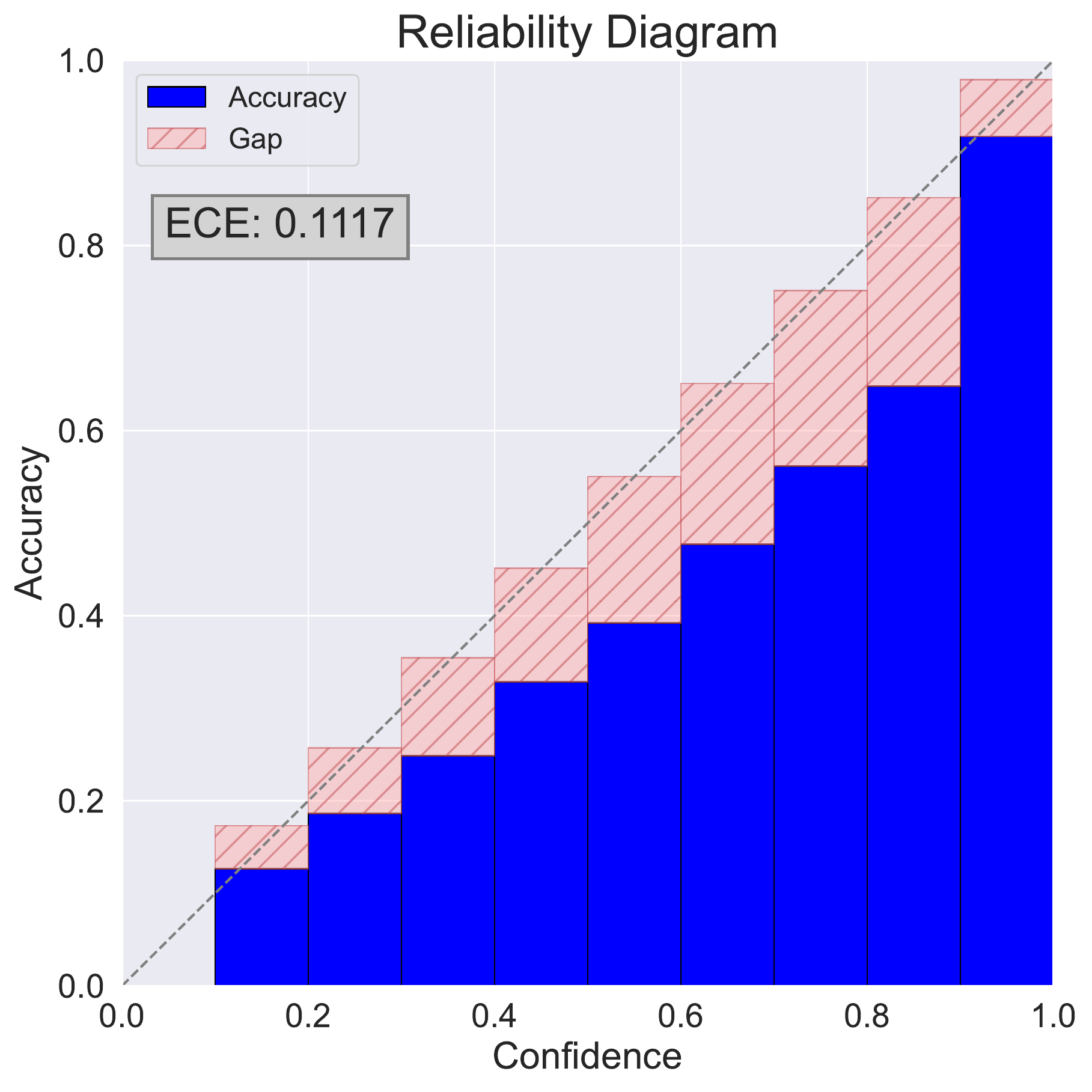} 
    } 
    \subfloat{
    \label{subfig:cifar100_rn_fl}
    \includegraphics[width=.23\linewidth]{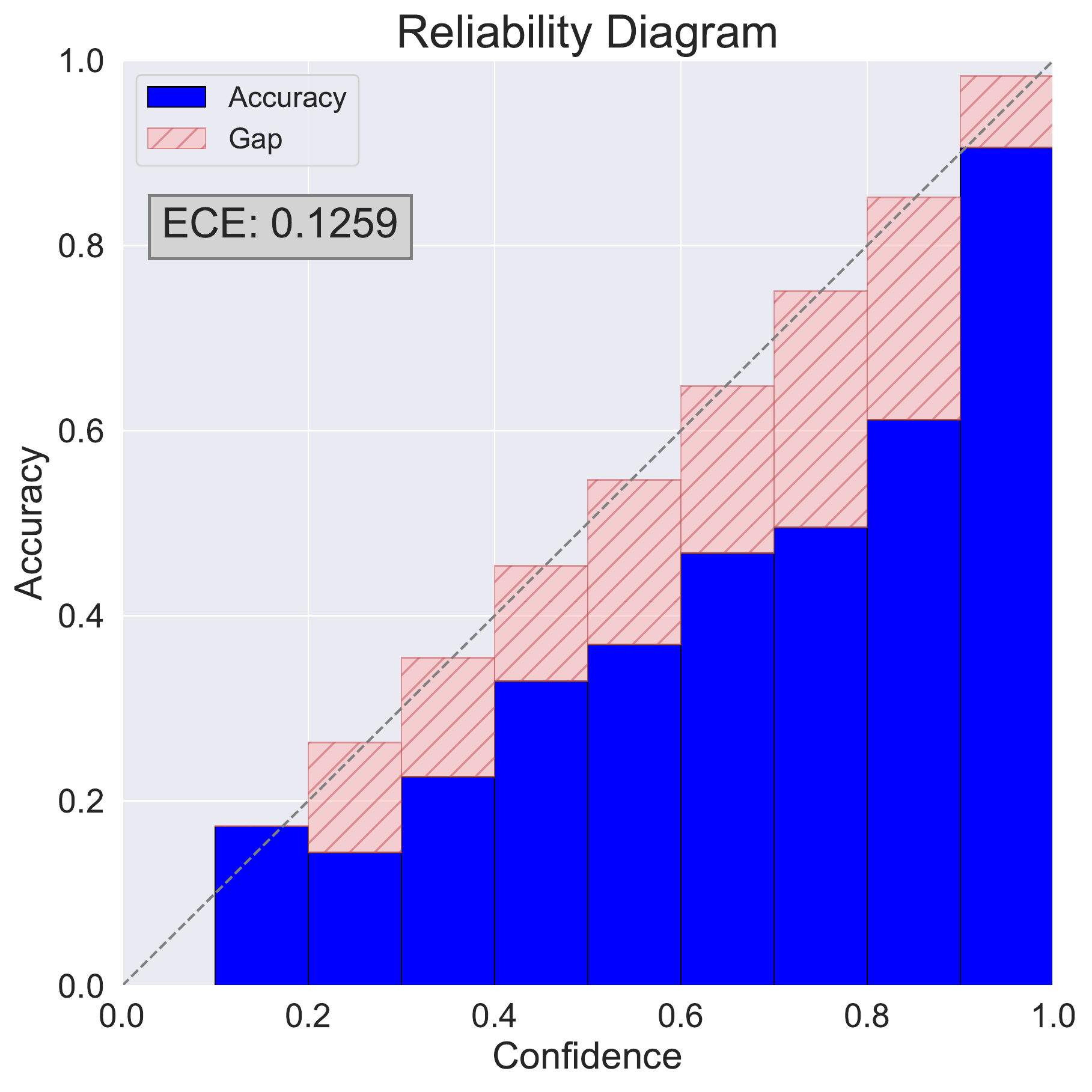} 
    } 
    \subfloat{
    \label{subfig:cifar100_rn_kdexe}
    \includegraphics[width=.23\linewidth]{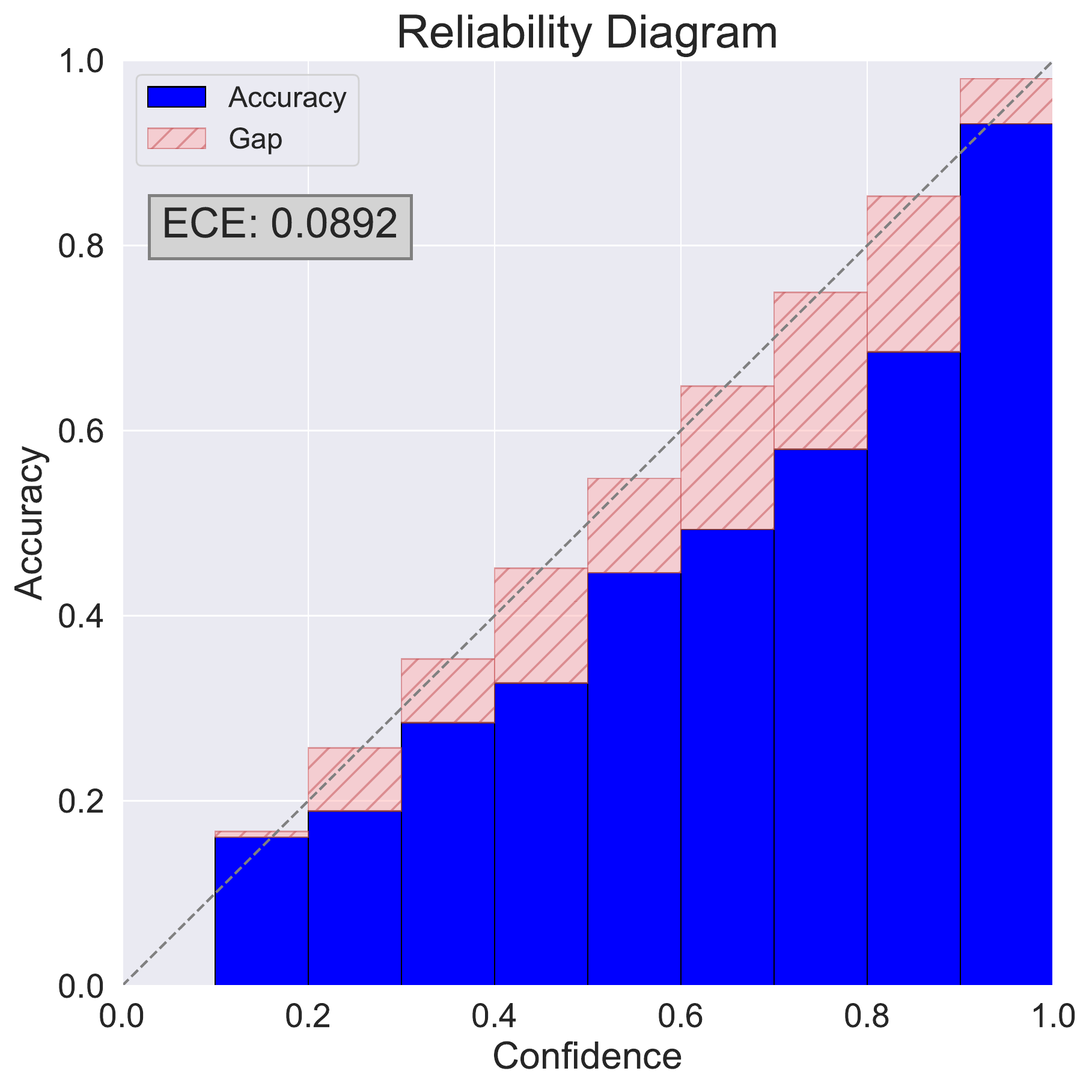} 
    } 
    \hfill
    \subfloat[XE]{
    \label{subfig:cifar100_wrn_xe}
    \includegraphics[width=.23\linewidth]{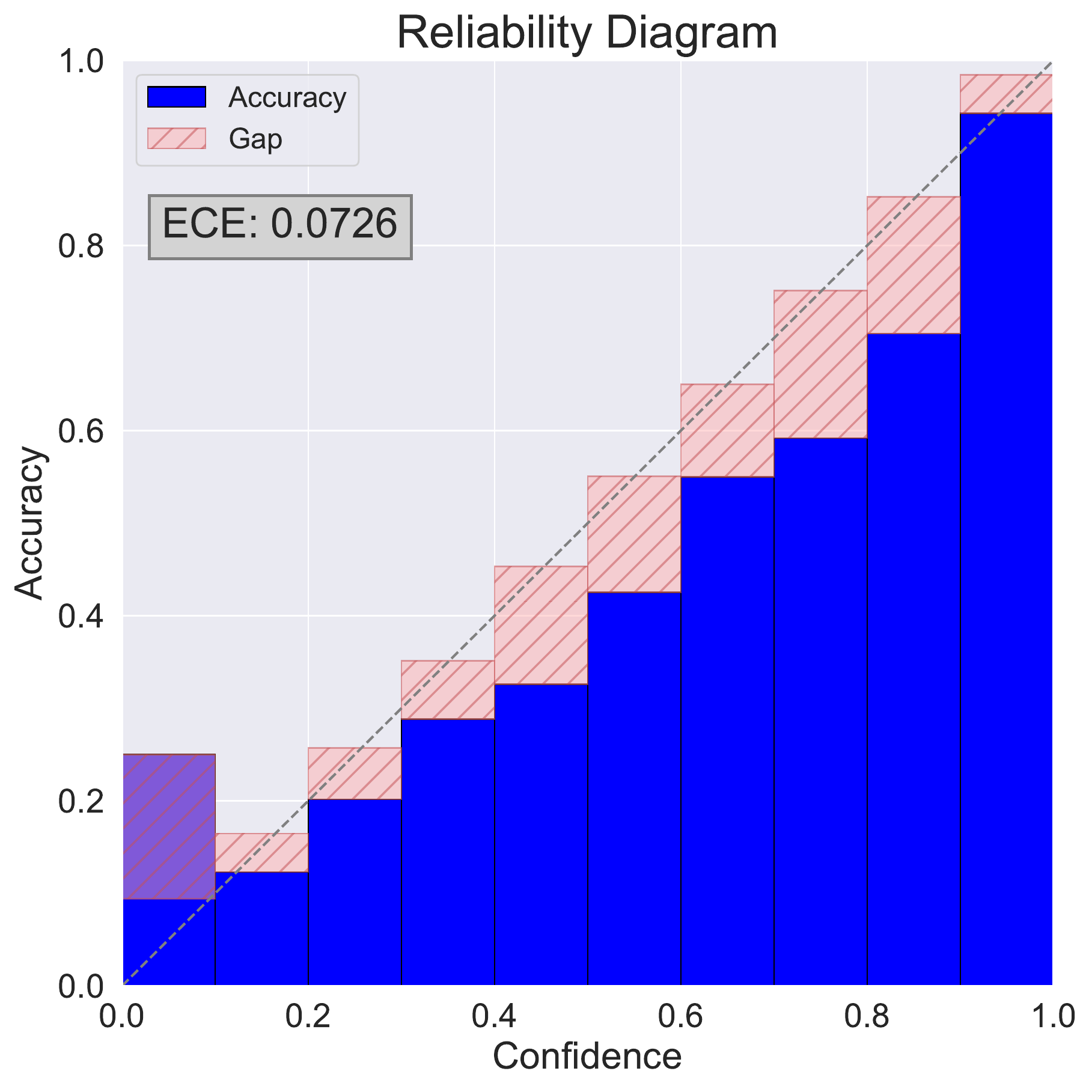}
    } 
    \subfloat[MMCE]{
    \label{subfig:cifar100_wrn_mmce}
    \includegraphics[width=.23\linewidth]{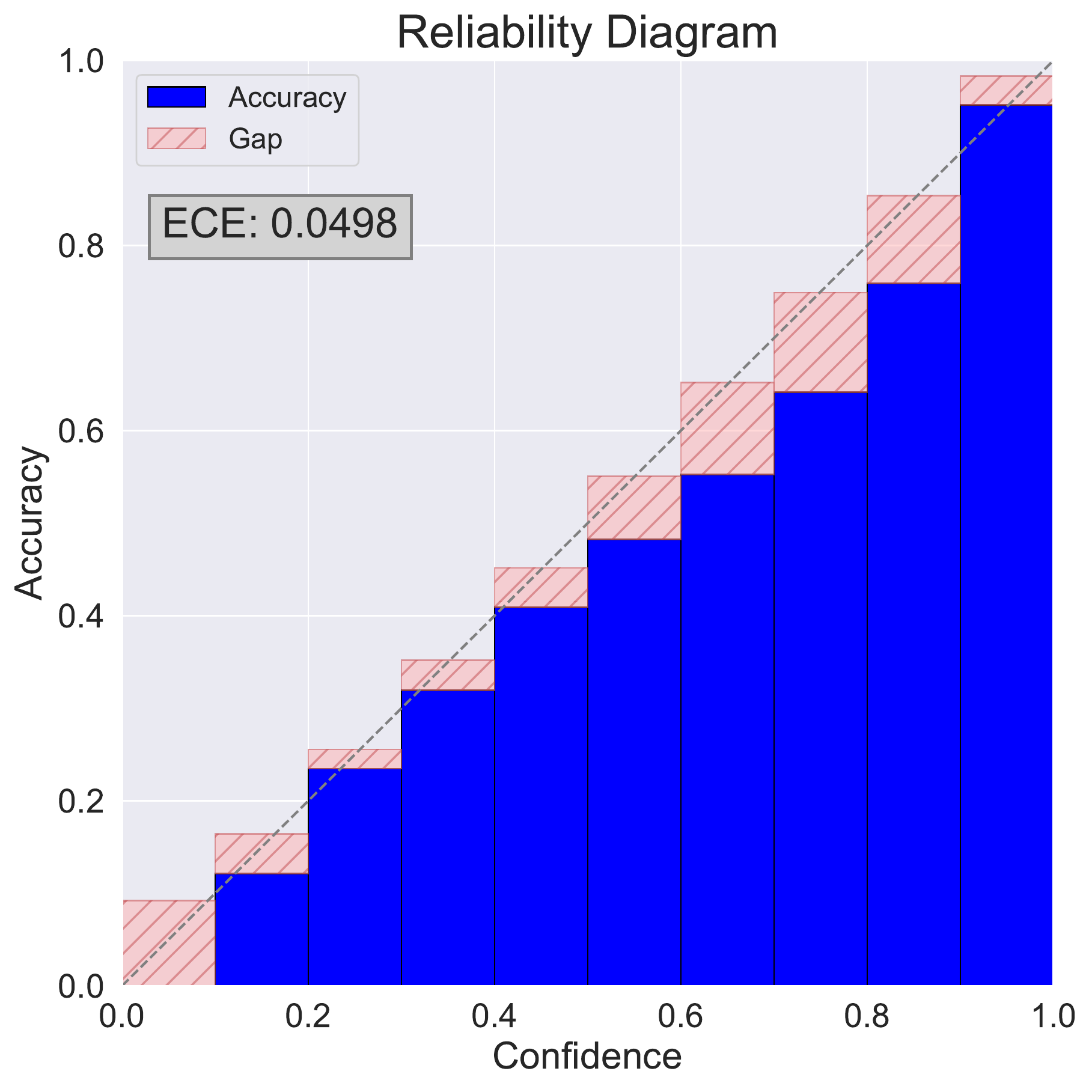} 
    } 
    \subfloat[FL-53]{
    \label{subfig:cifar100_wrn_fl}
    \includegraphics[width=.23\linewidth]{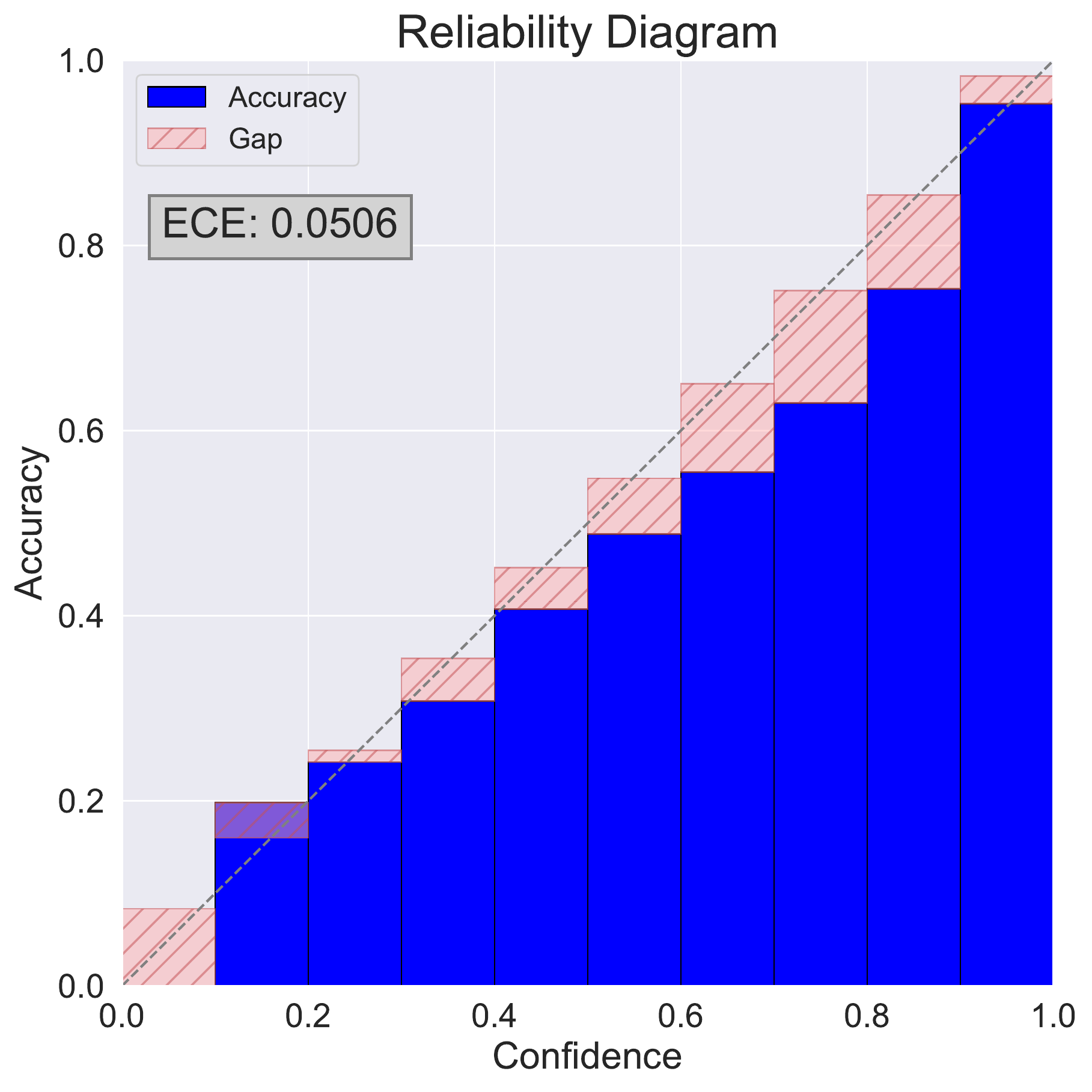} 
    } 
    \subfloat[$L_1$ KDE-XE]{
    \label{subfig:cifar100_wrn_kdexe}
    \includegraphics[width=.23\linewidth]{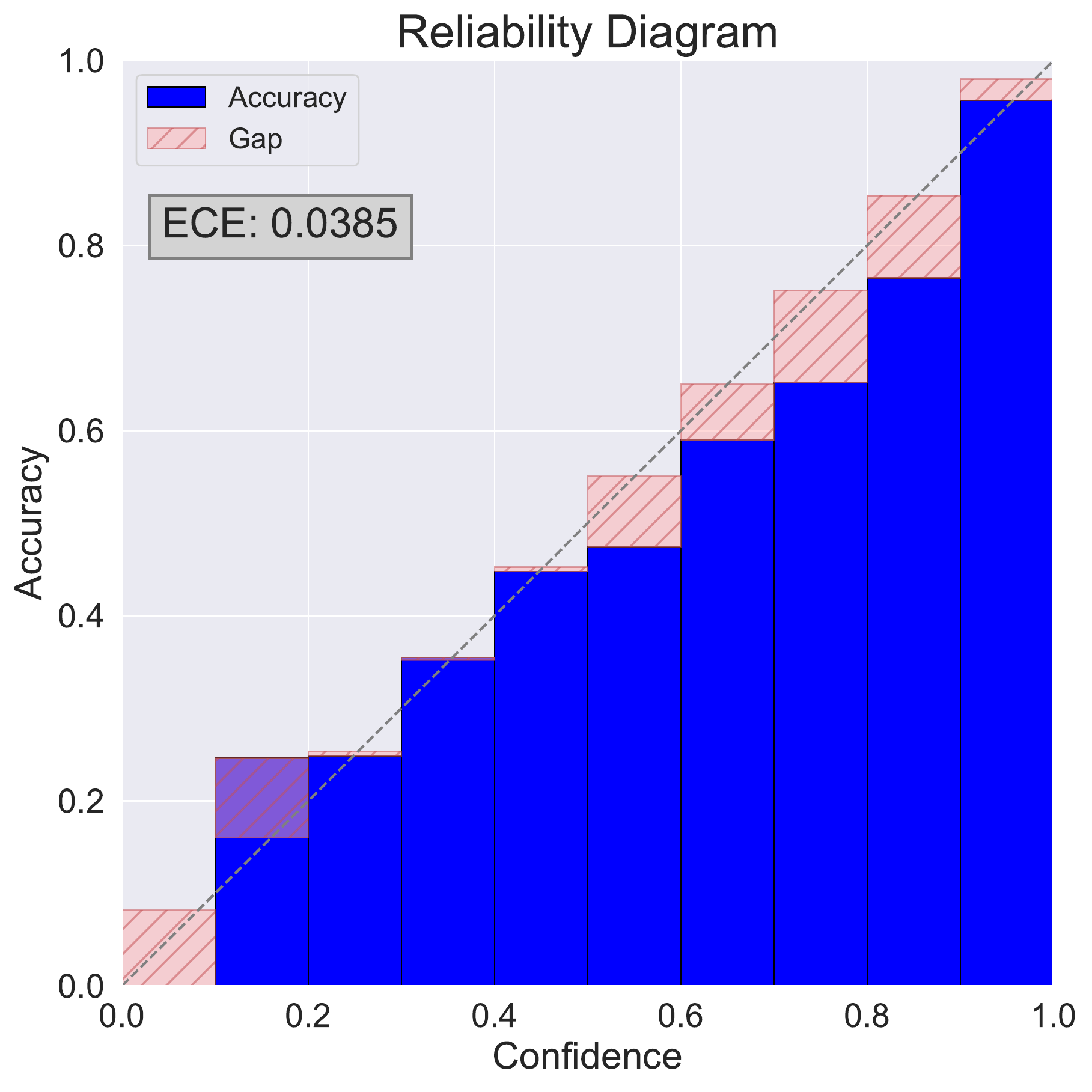} 
    } 
    \caption{Reliability diagrams for top-label calibration on CIFAR-100 using ResNet (top row) and Wide-ResNet (bottom row) for each of the considered baselines.}
    \label{fig:reliability_plot_toplabel_calibration}
\end{figure}

\section{Relationship between $ECE^{bin}$ and $ECE^{KDE}$}
\label{appendix:binning_vs_kde}
In the following two sections, we investigate further the relationship between $ECE^{bin}$, as the most widely used metric, and our $ECE^{KDE}$ estimator. For the three types of calibration, $ECE^{bin}$ is calculated with equal-width binning scheme. The values for the bandwidth in $ECE^{KDE}$ and the number of bins per class for $ECE^{bin}$ are chosen with leave-one-out maximum likelihood procedure and Doane's formula \citep{doanesformula1976}, respectively.

Figure~\ref{fig:binned_simplex_estimator_c10} shows an example of  $ECE^{bin}$ in a three-class setting on CIFAR-10. The points are mostly concentrated at the edges of the histogram, as can be seen from Figure \ref{subfig:histogram_c10}. The surface of the corresponding Dirichlet KDE is given in \ref{subfig:surface_c10}.

Figure~\ref{fig:binned_vs_kde_points} shows the relationship between  $ECE^{bin}$ and $ECE^{KDE}$. The points represent a trained Resnet-56 model on a subset of three classes from CIFAR-10. In every row, a differnt number of points was used to estimate the $ECE^{KDE}$. We notice the $ECE^{KDE}$ estimates of the three types of calibration closely correspond to their histogram-based approximations.


\begin{figure}[ht!]
    \centering
    \subfloat[Splitting the simplex in 16 bins]{
    \label{subfig:scatterplot_c10}
    \includegraphics[width=.31\linewidth]{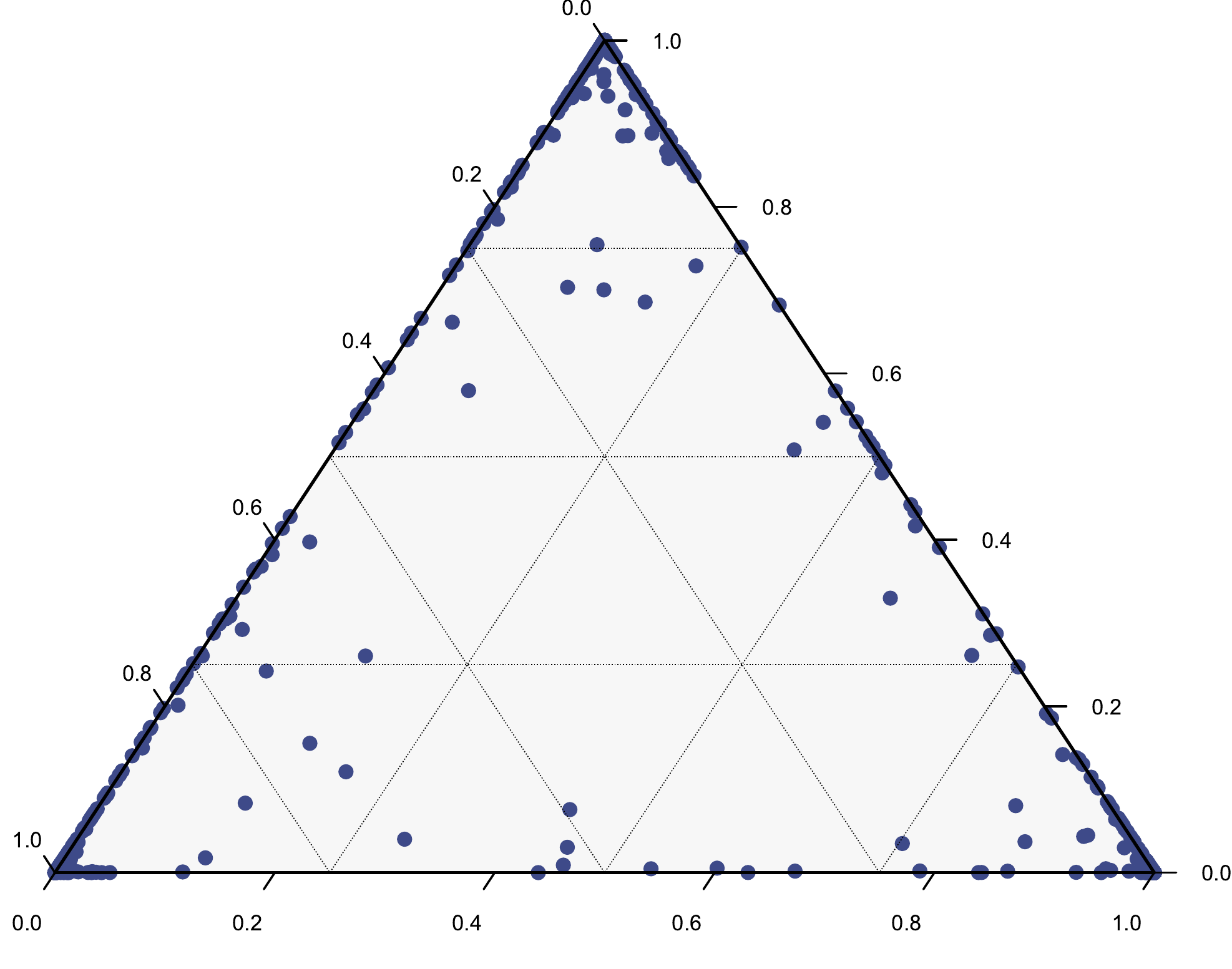}
    } \hfill
    \subfloat[Corresponding histogram]{
    \label{subfig:histogram_c10}
    \includegraphics[width=.31\linewidth]{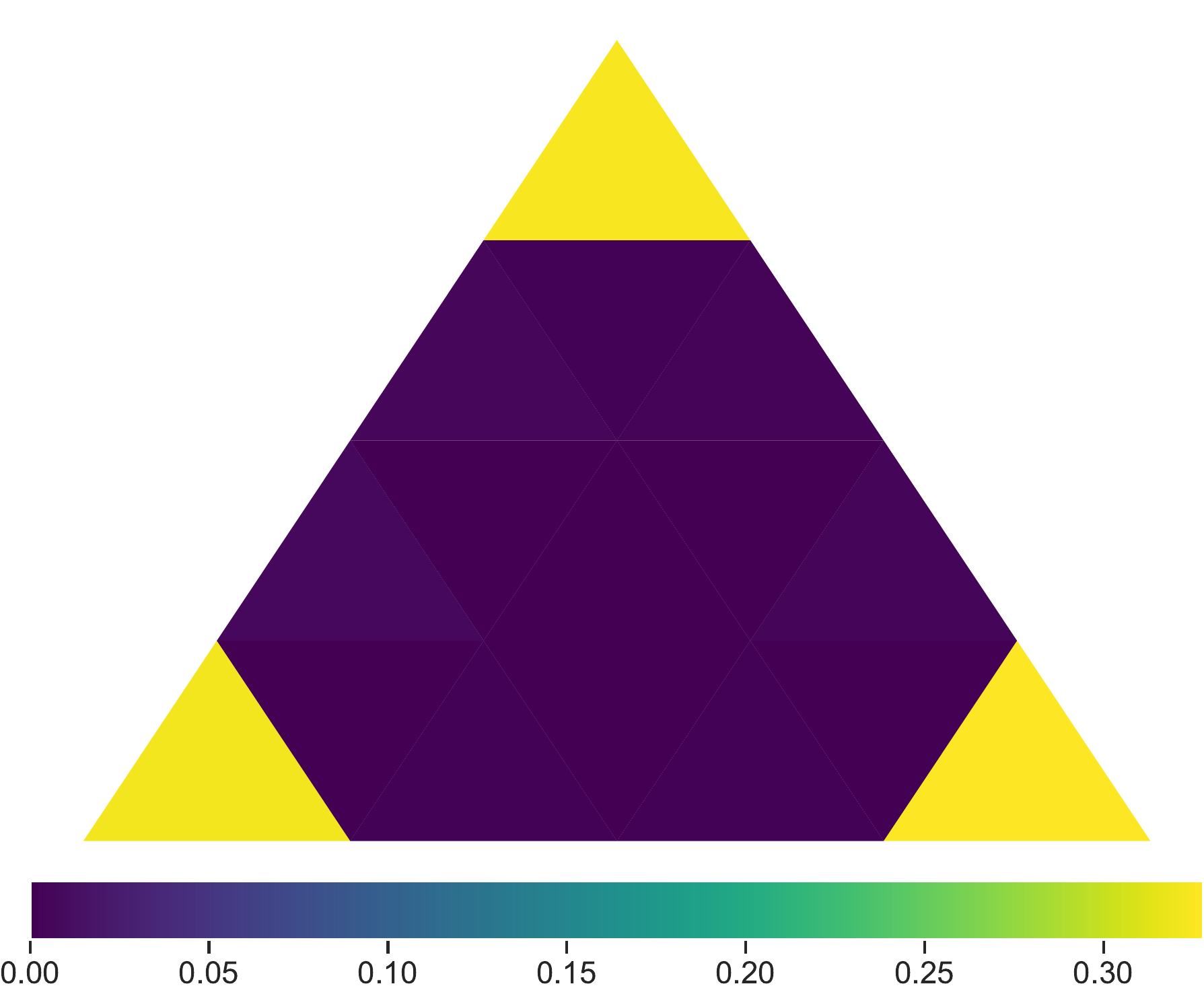} 
    } \hfill
    \subfloat[Corresponding Dirichlet KDE]{
    \label{subfig:surface_c10}
    \includegraphics[width=.31\linewidth]{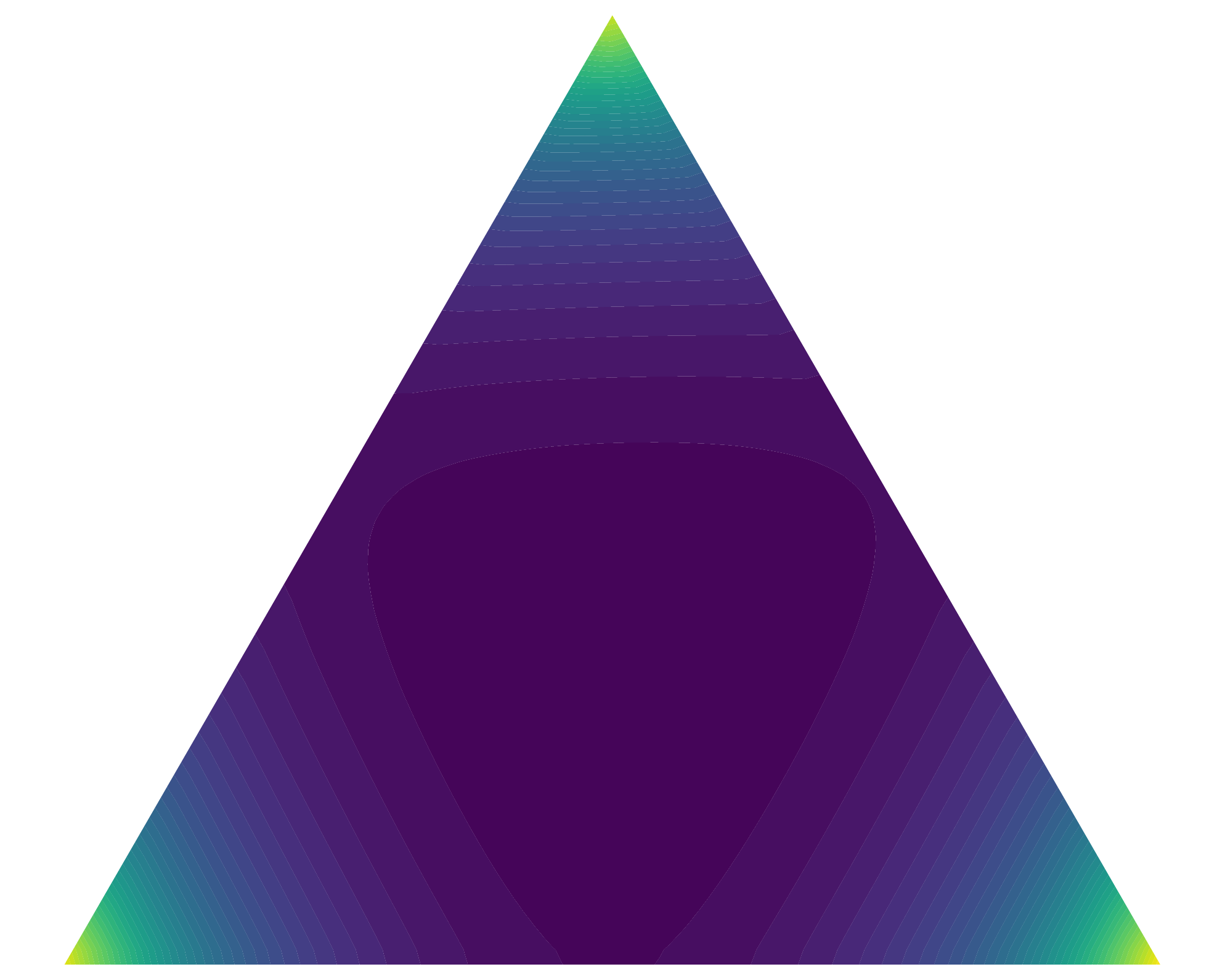}
    }
    \caption{An example of a simplex binned estimator and kernel-density estimator for CIFAR-10 }
    \label{fig:binned_simplex_estimator_c10}
\end{figure}

\begin{figure}[ht!]
    \captionsetup[subfigure]{labelformat=empty}
    \centering
    \subfloat{
    \includegraphics[width=.31\linewidth]{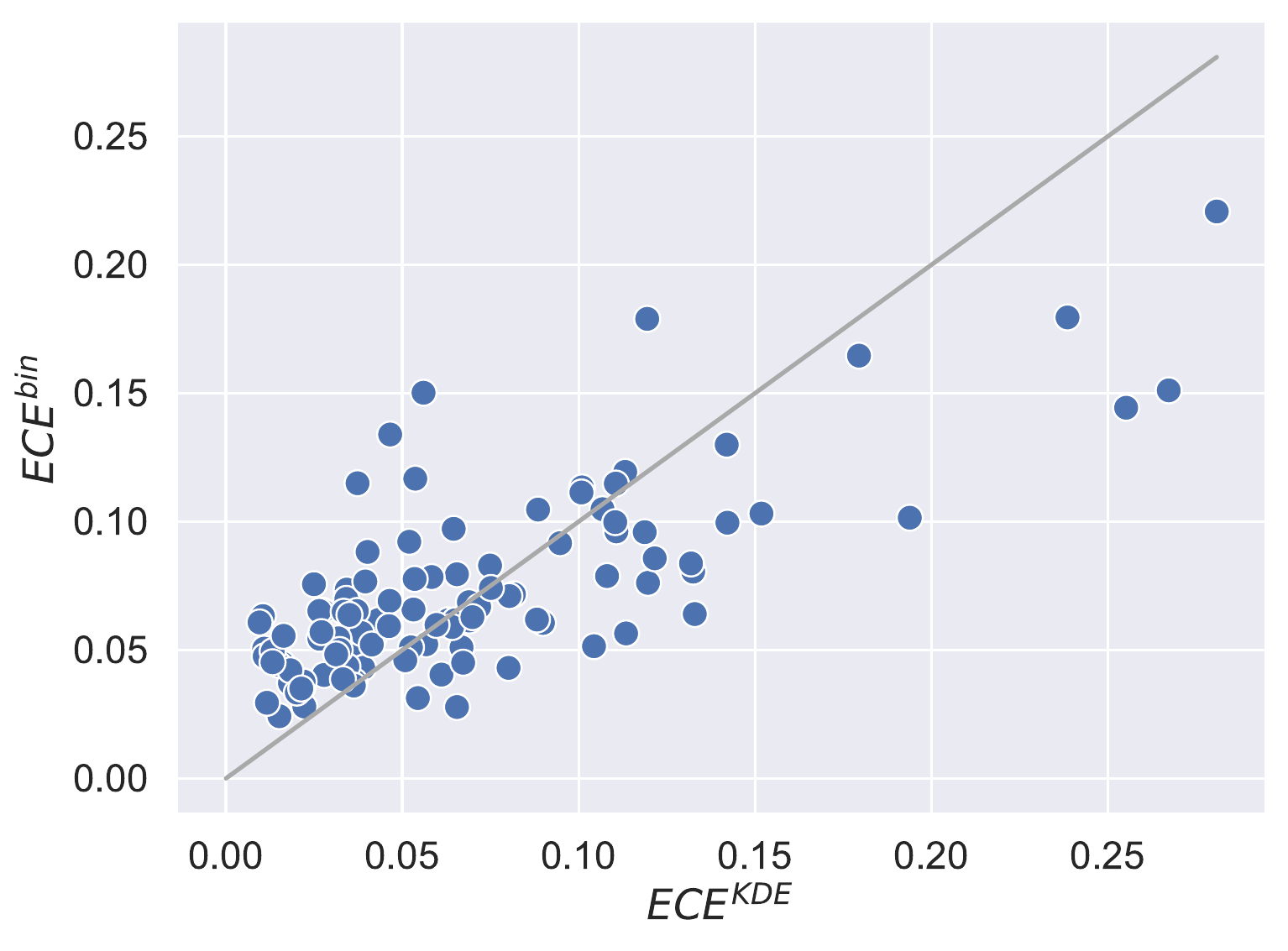}
    }
    \subfloat{
    \includegraphics[width=.31\linewidth]{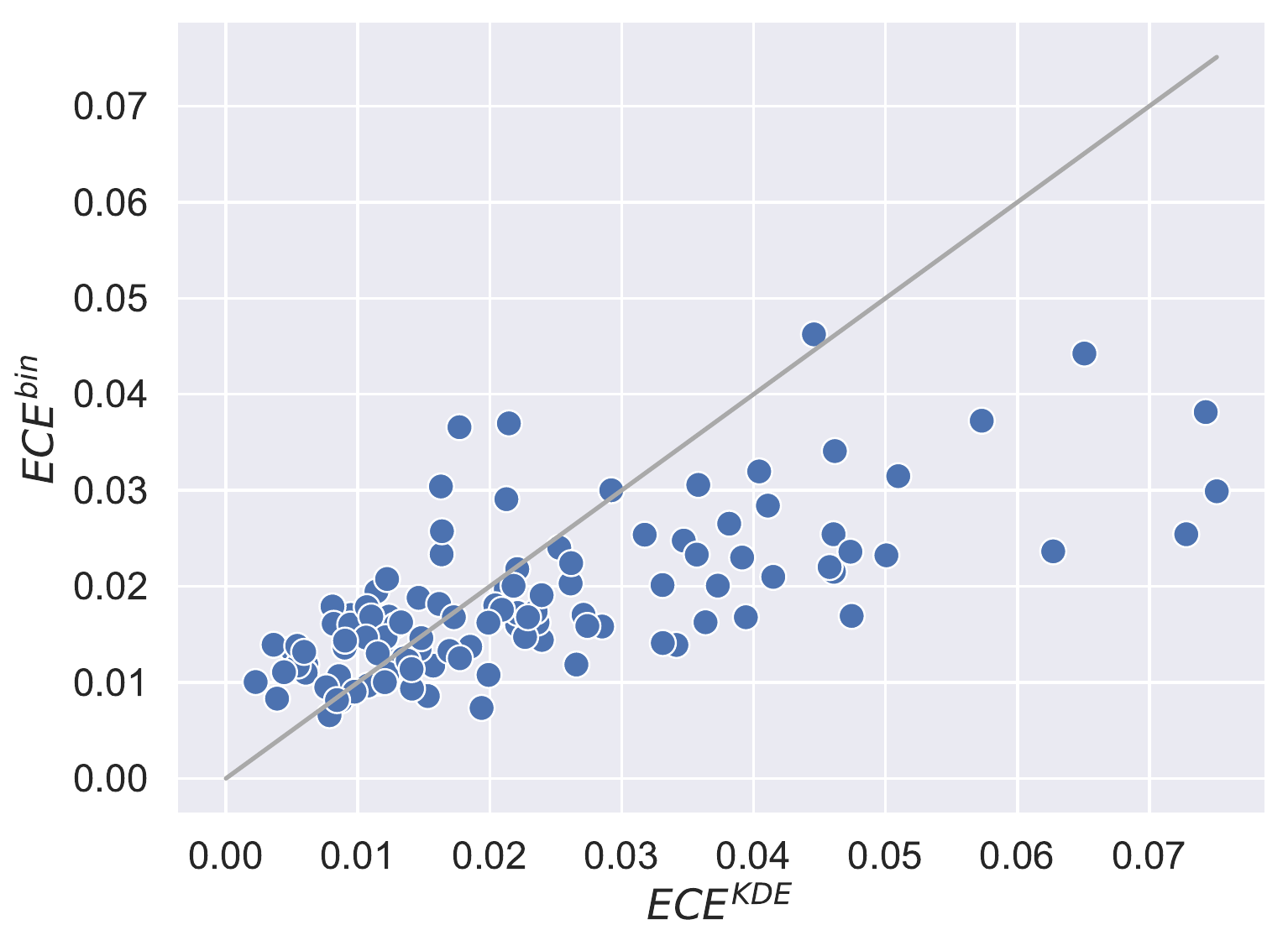}
    }
    \subfloat{
    \includegraphics[width=.31\linewidth]{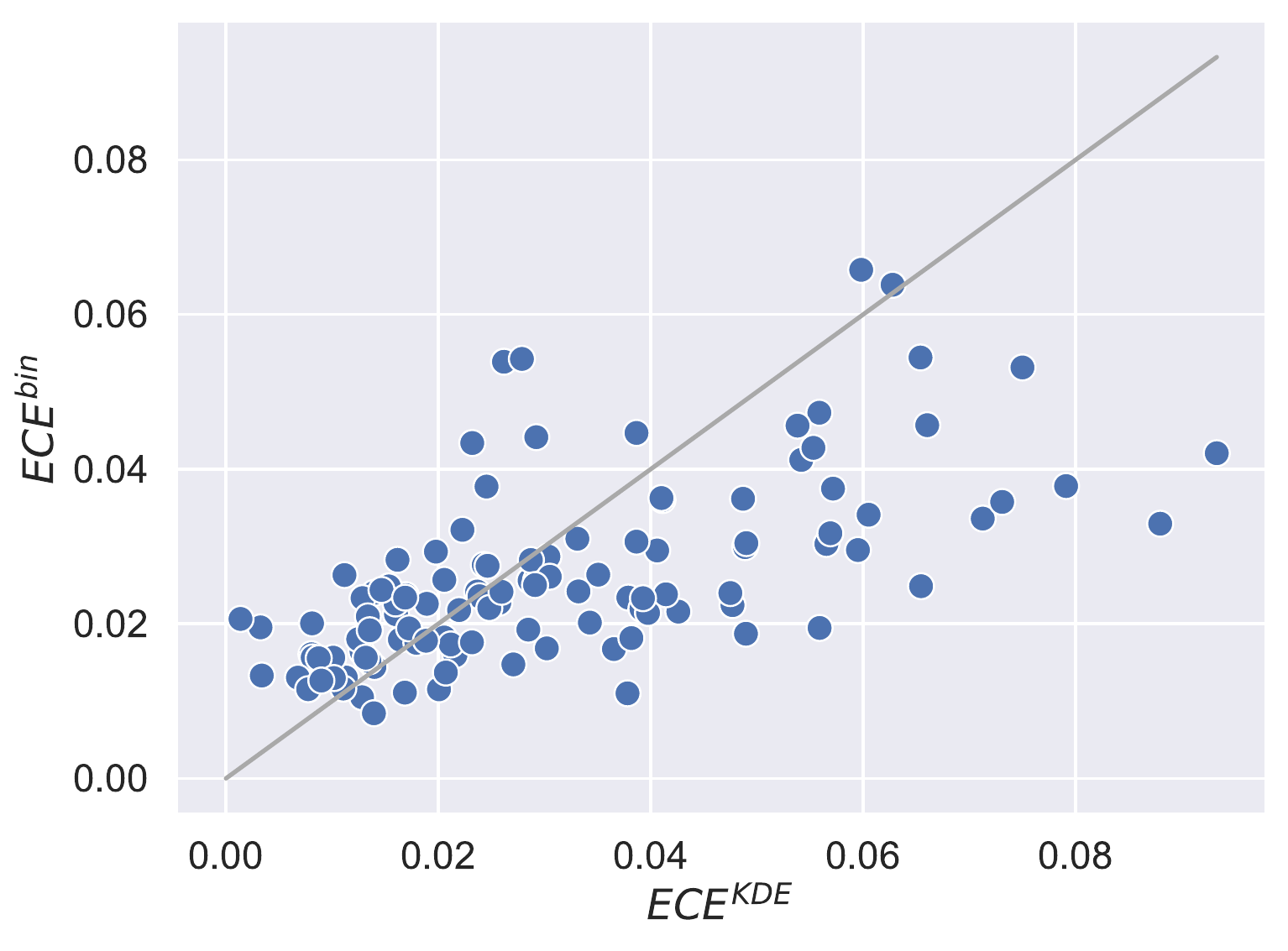}
    } \\[-2ex] 
    \subfloat{
    \includegraphics[width=.31\linewidth]{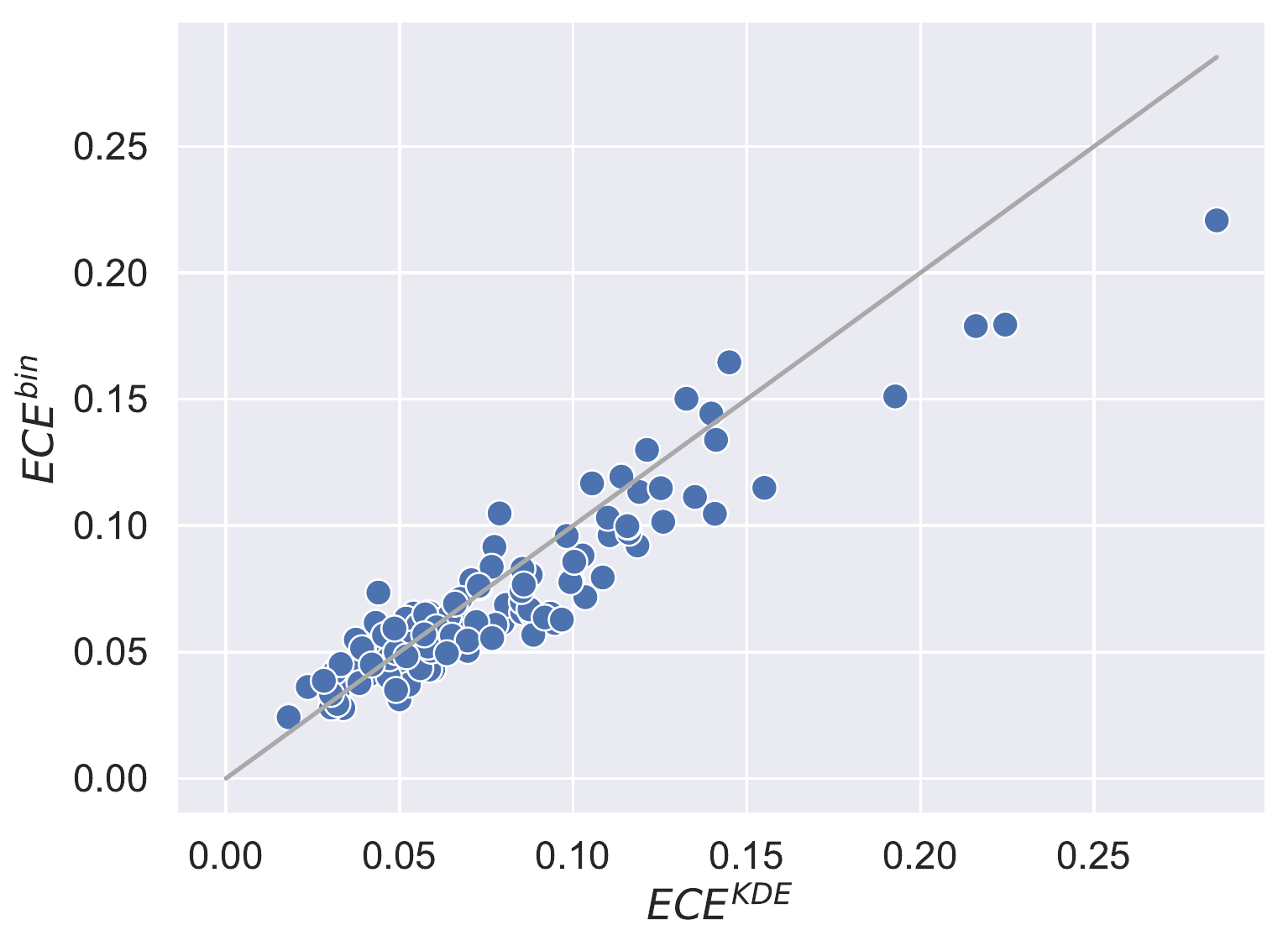}
    }
    \subfloat{
    \includegraphics[width=.31\linewidth]{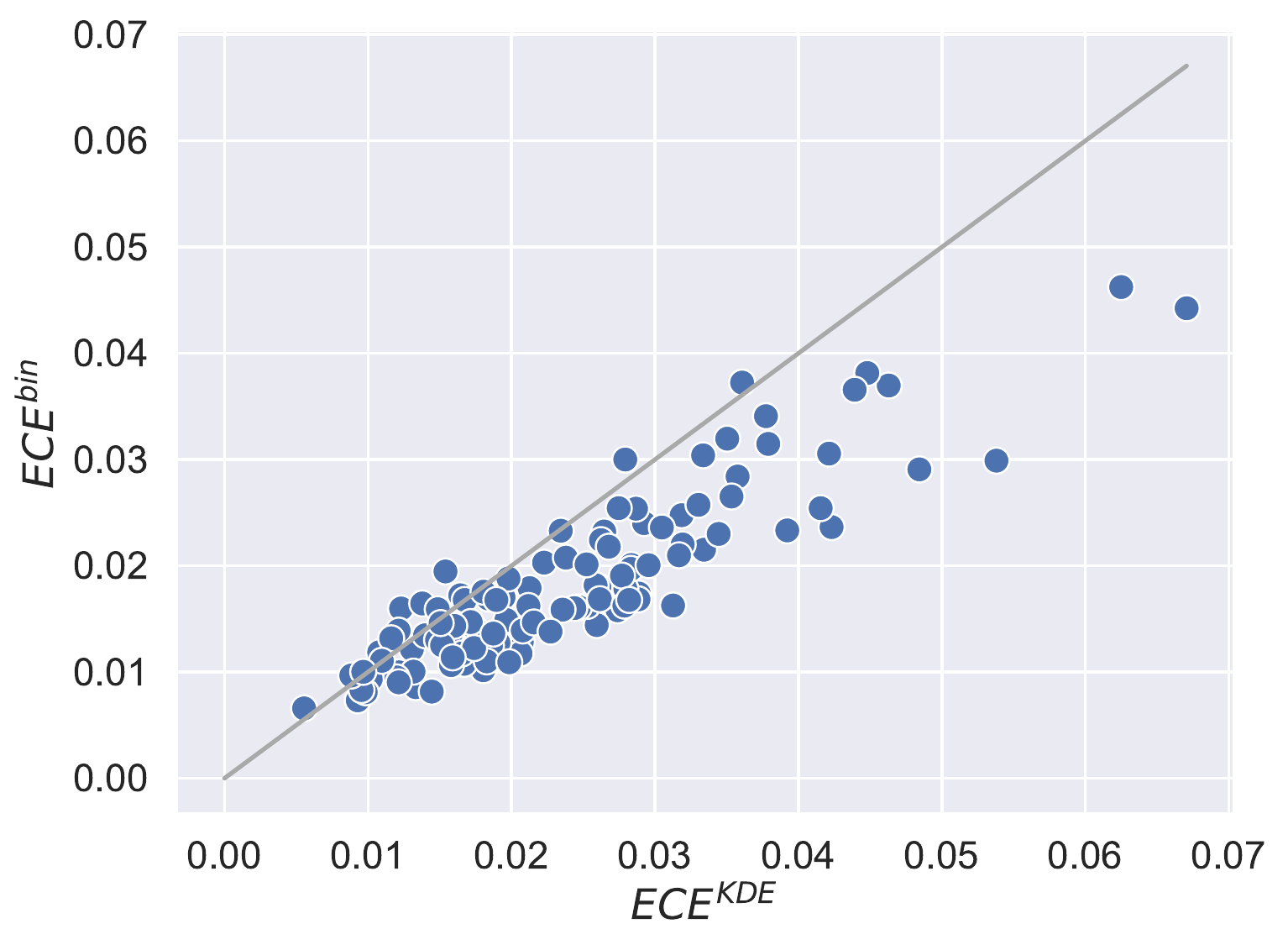}
    }
    \subfloat{
    \includegraphics[width=.31\linewidth]{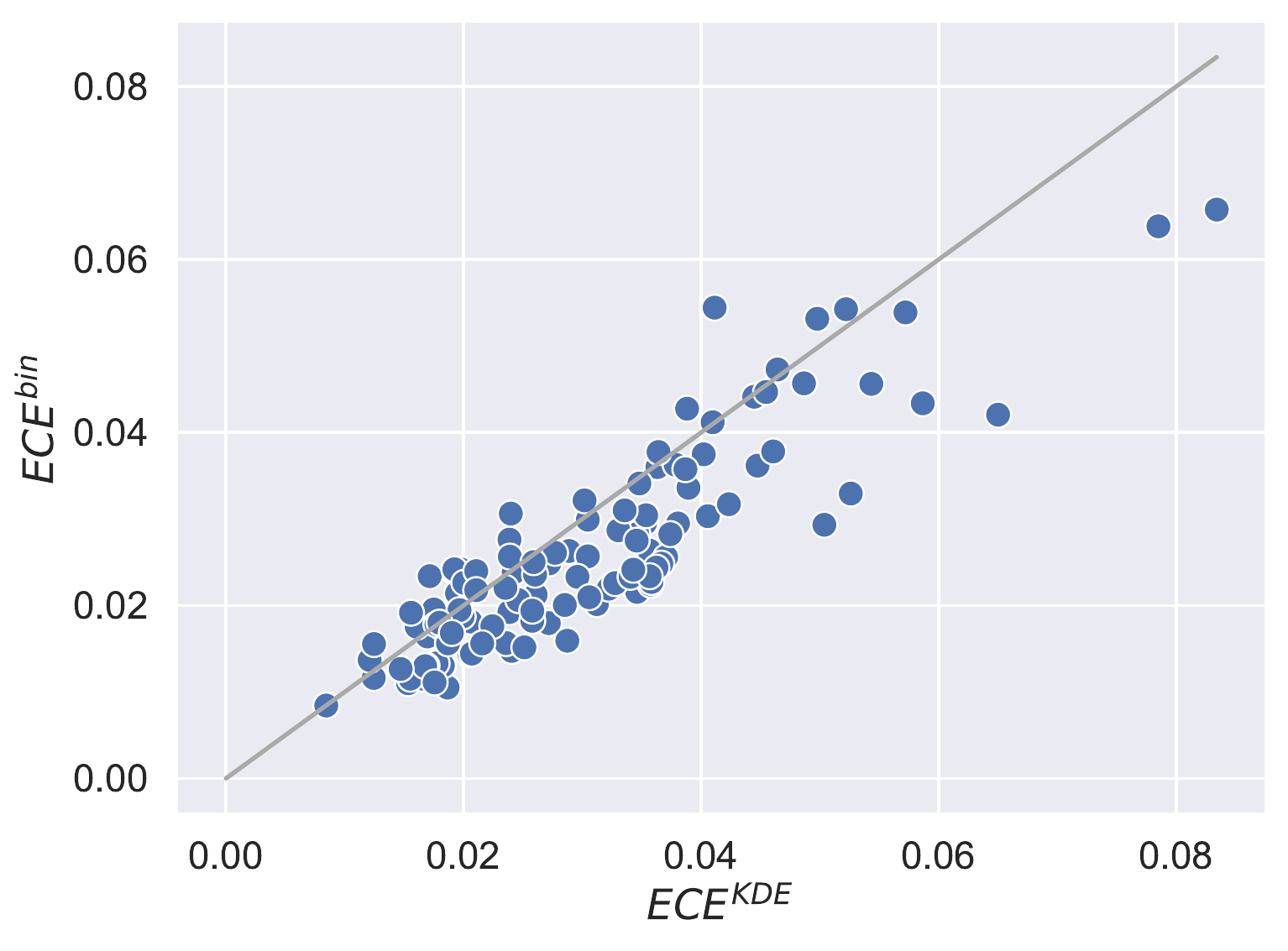}
    } \\[-2ex] 
    \subfloat{
    \includegraphics[width=.31\linewidth]{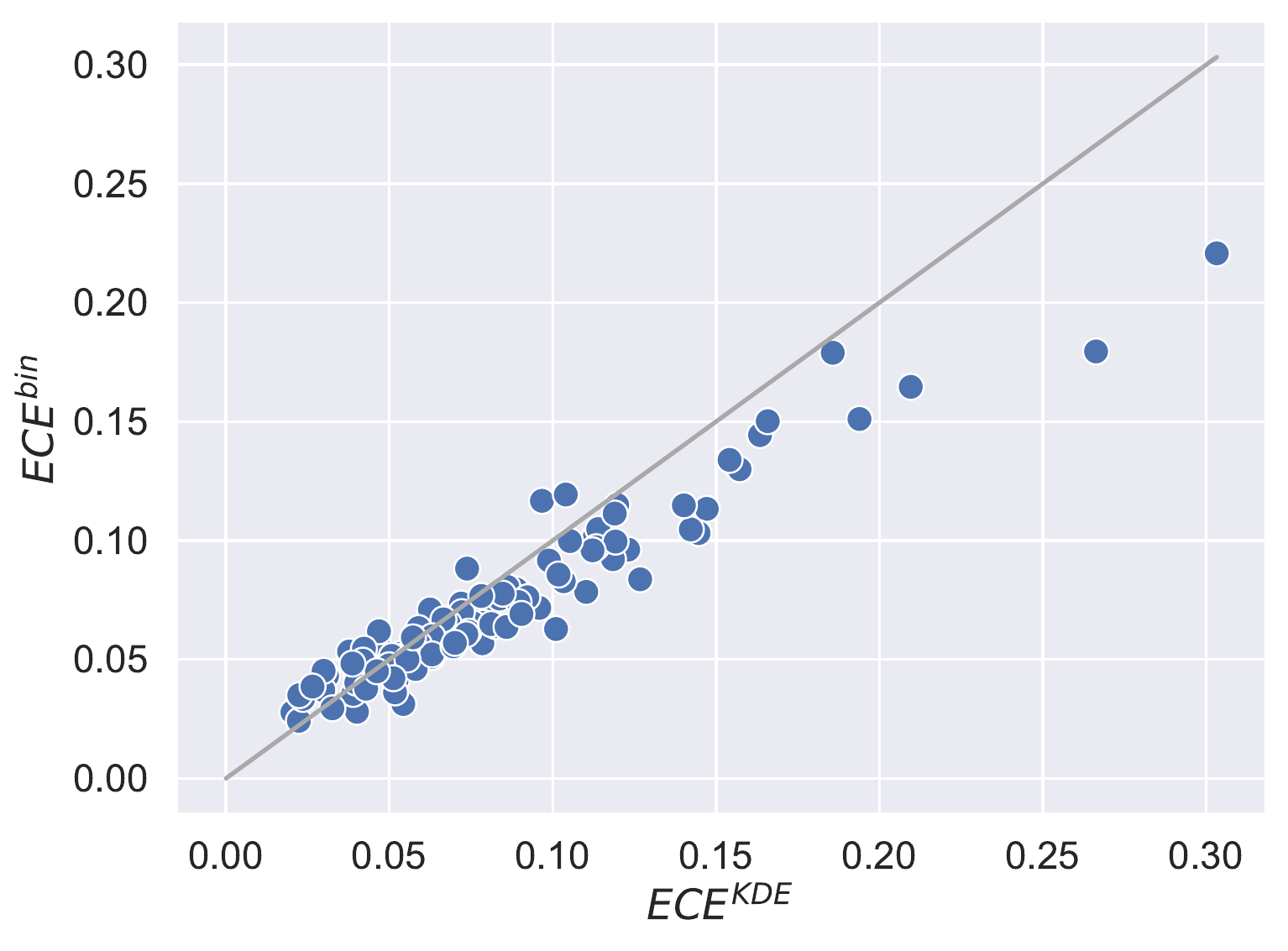}
    }
    \subfloat{
    \includegraphics[width=.31\linewidth]{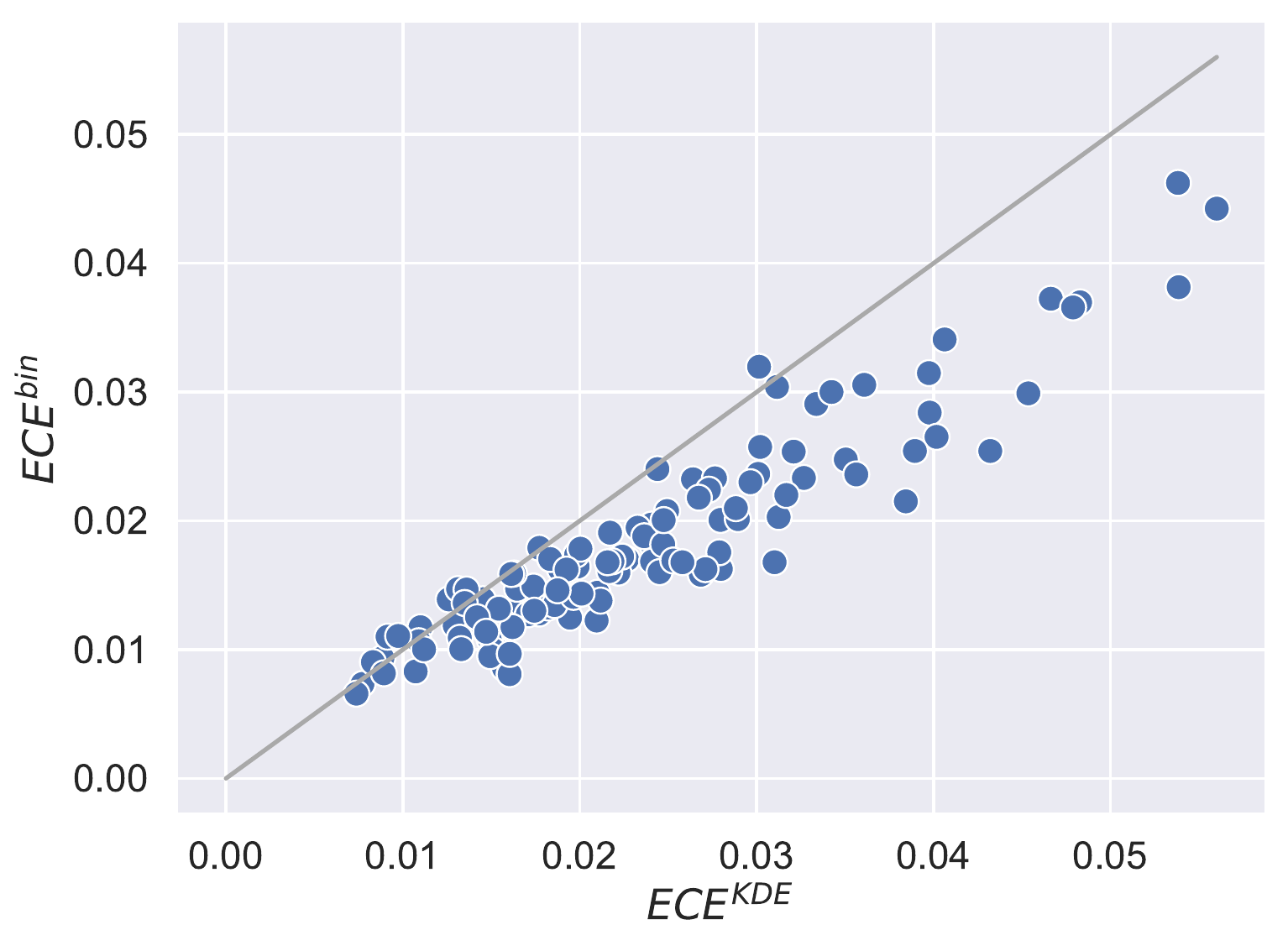}
    }
    \subfloat{
    \includegraphics[width=.31\linewidth]{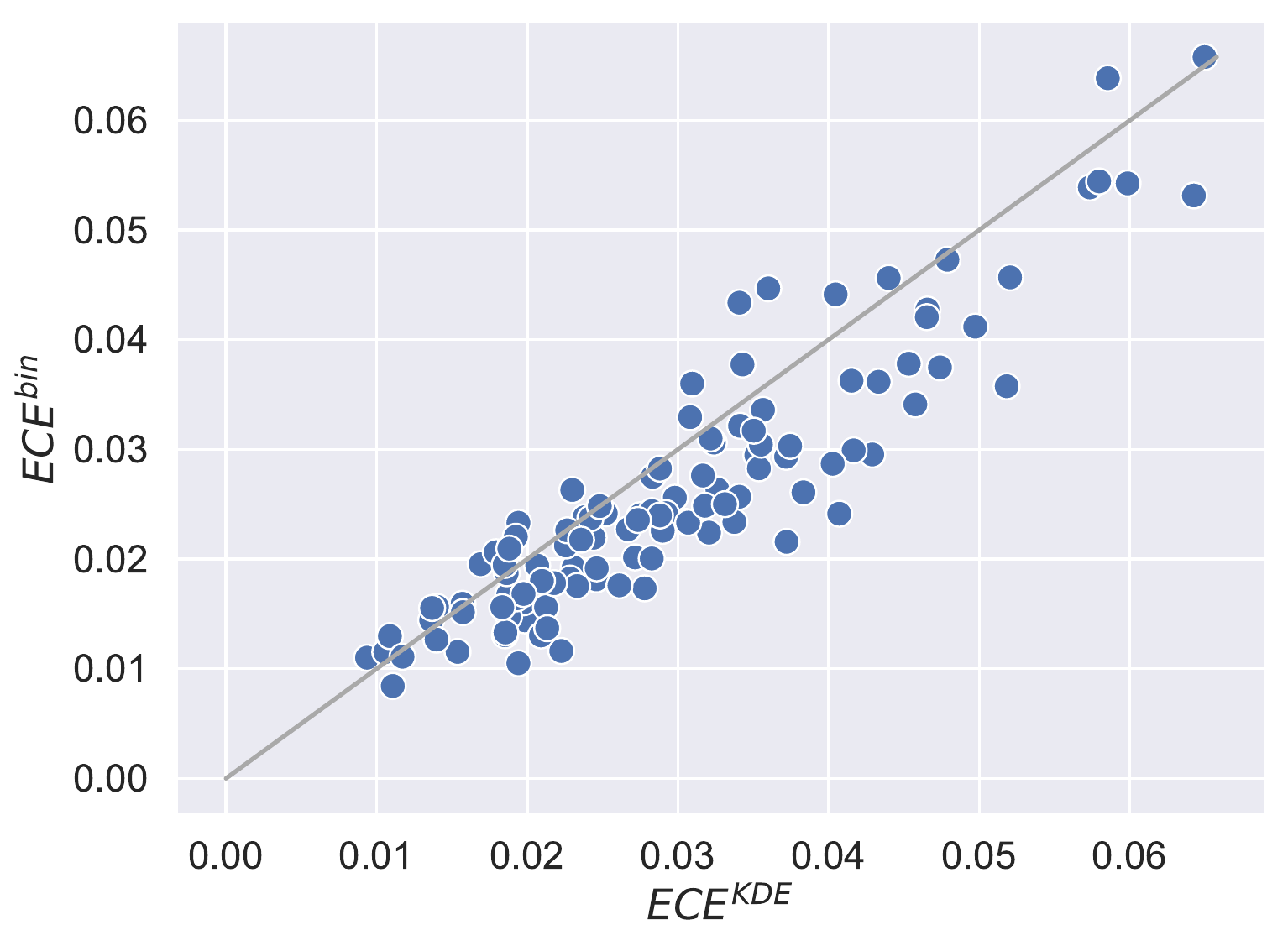}
    } \\[-2ex]
    \subfloat[Canonical]{
    \includegraphics[width=.31\linewidth]{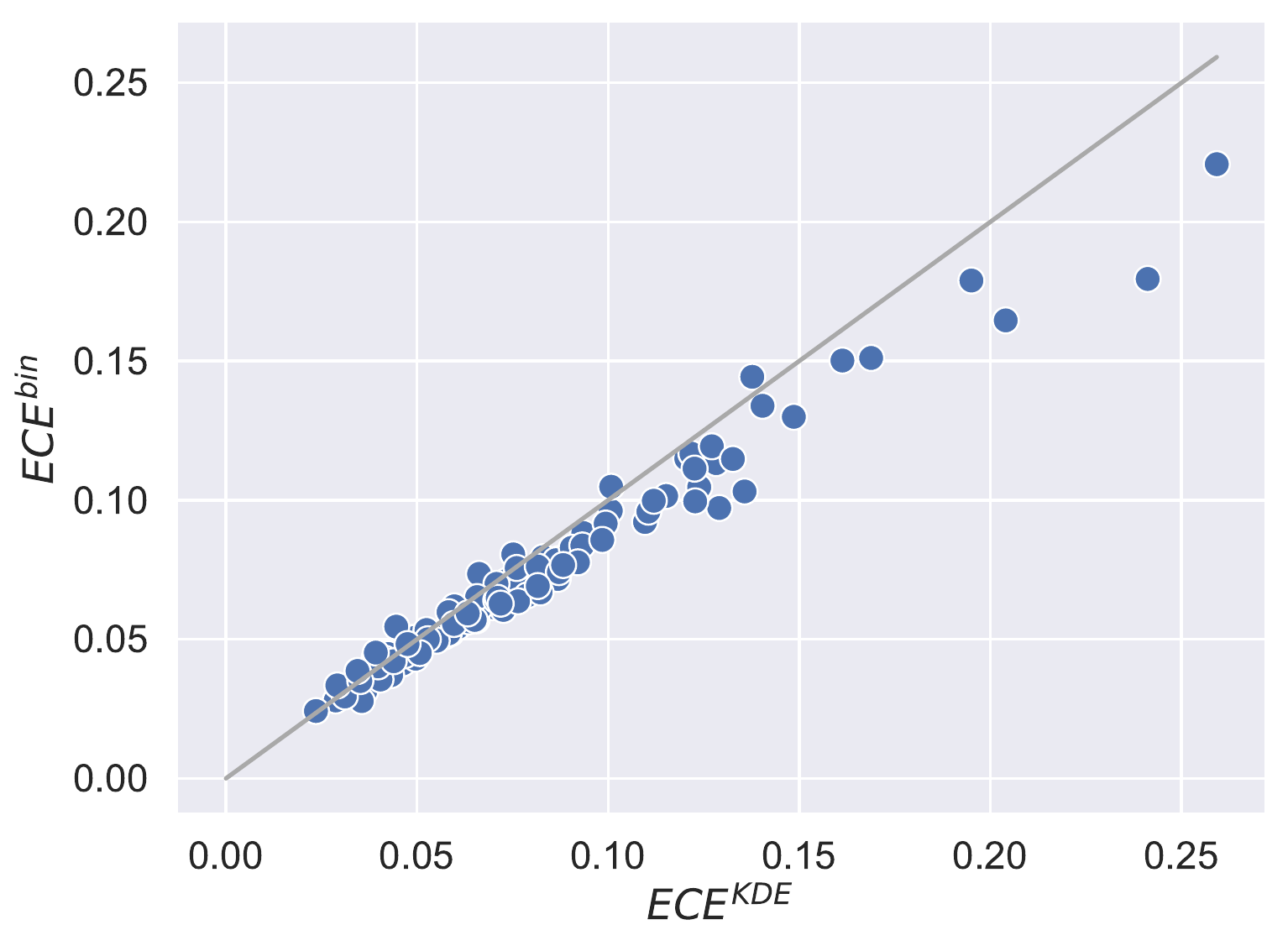}
    }
    \subfloat[Marginal]{
    \includegraphics[width=.31\linewidth]{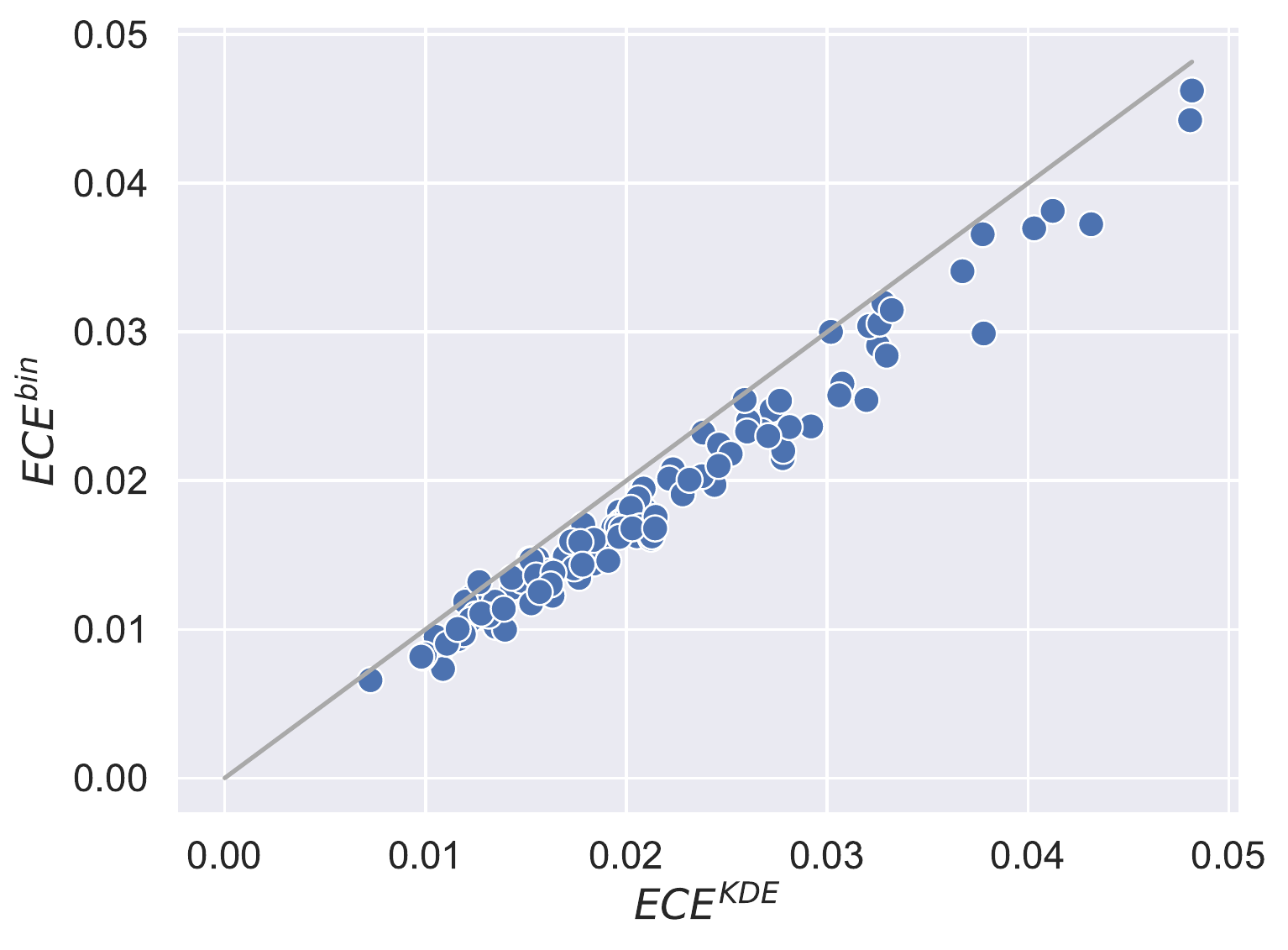}
    }
    \subfloat[Top-label]{
    \includegraphics[width=.31\linewidth]{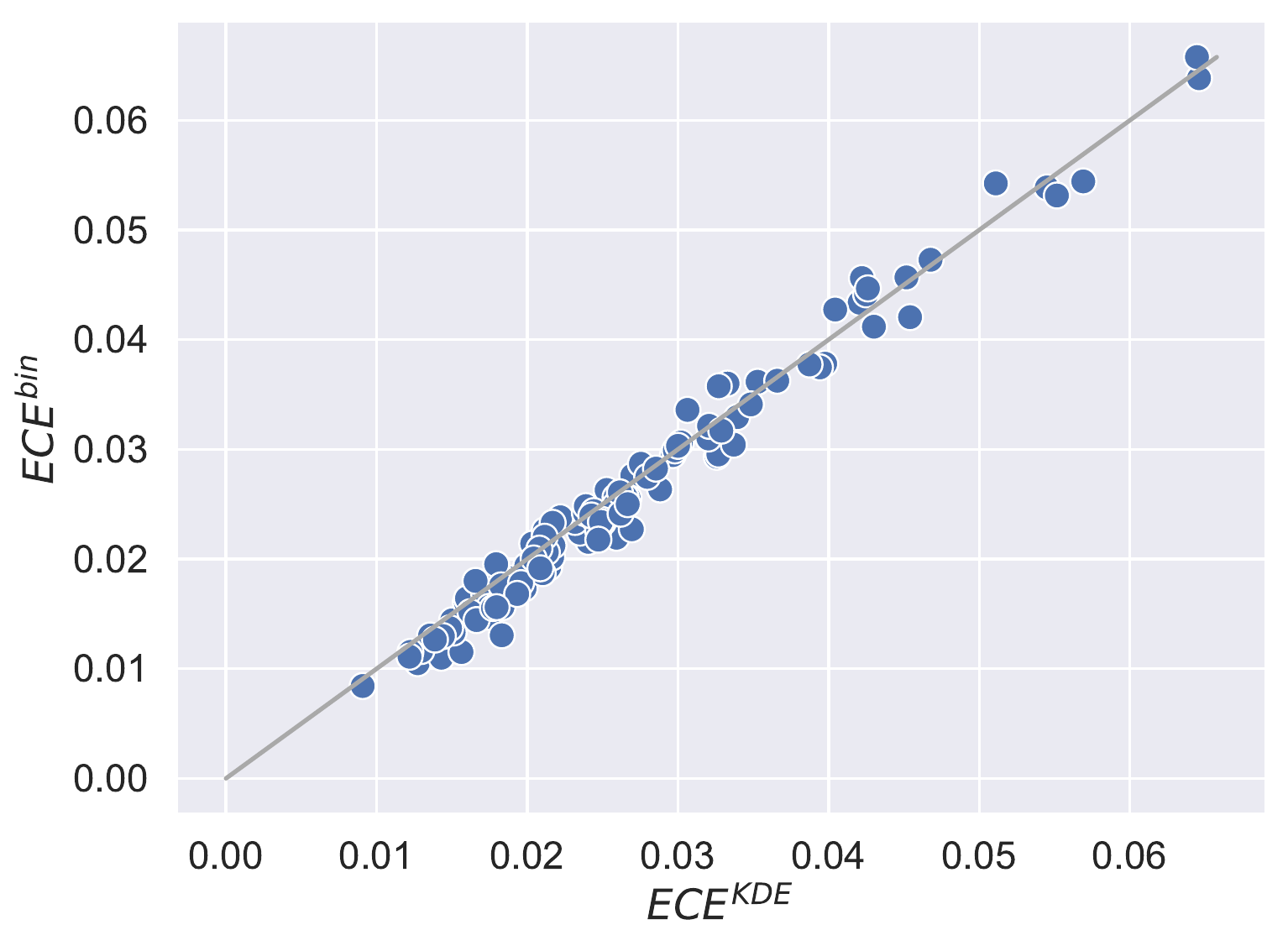}
    } 
    \caption{Relationship between $ECE^{bin}$ and $ECE^{KDE}$ for the three types of calibration: canonical (first column), marginal (second column) and top-label (third column). In every row top to bottom, different number of points (100, 500, 1000 and all points, respectively) are used to approximate $ECE^{KDE}$. Each point represents a ResNet-56 model trained on a subset of three classes from CIFAR-10. The number of bins per class (13) is selected using Doane's formula \citep{doanesformula1976}, while the bandwidth is selected using a leave-one-out maximum likelihood procedure (typical chosen values are 0.001 for 100 points and 0.0001 otherwise).
    }
    \label{fig:binned_vs_kde_points}
\end{figure}

\section{Bias and convergence rates}
\label{subsection:bias_and_convergece_rates}

Figure~\ref{fig:rebuttal_binned_vs_kde_scatter} shows a comparison of $ECE^{KDE}$ and $ECE^{bin}$ estimated with a varying number of points. The ground truth is computed from 3000 test points with $ECE^{KDE}$. The used model is a ResNet-56, trained on a subset of three classes from CIFAR-10. The figure shows that the two estimates are comparable and both are doing a reasonable job in a three-class setting. 

\begin{figure}[ht!]
    \centering
    \subfloat[Canonical]{
    \includegraphics[width=.31\linewidth]{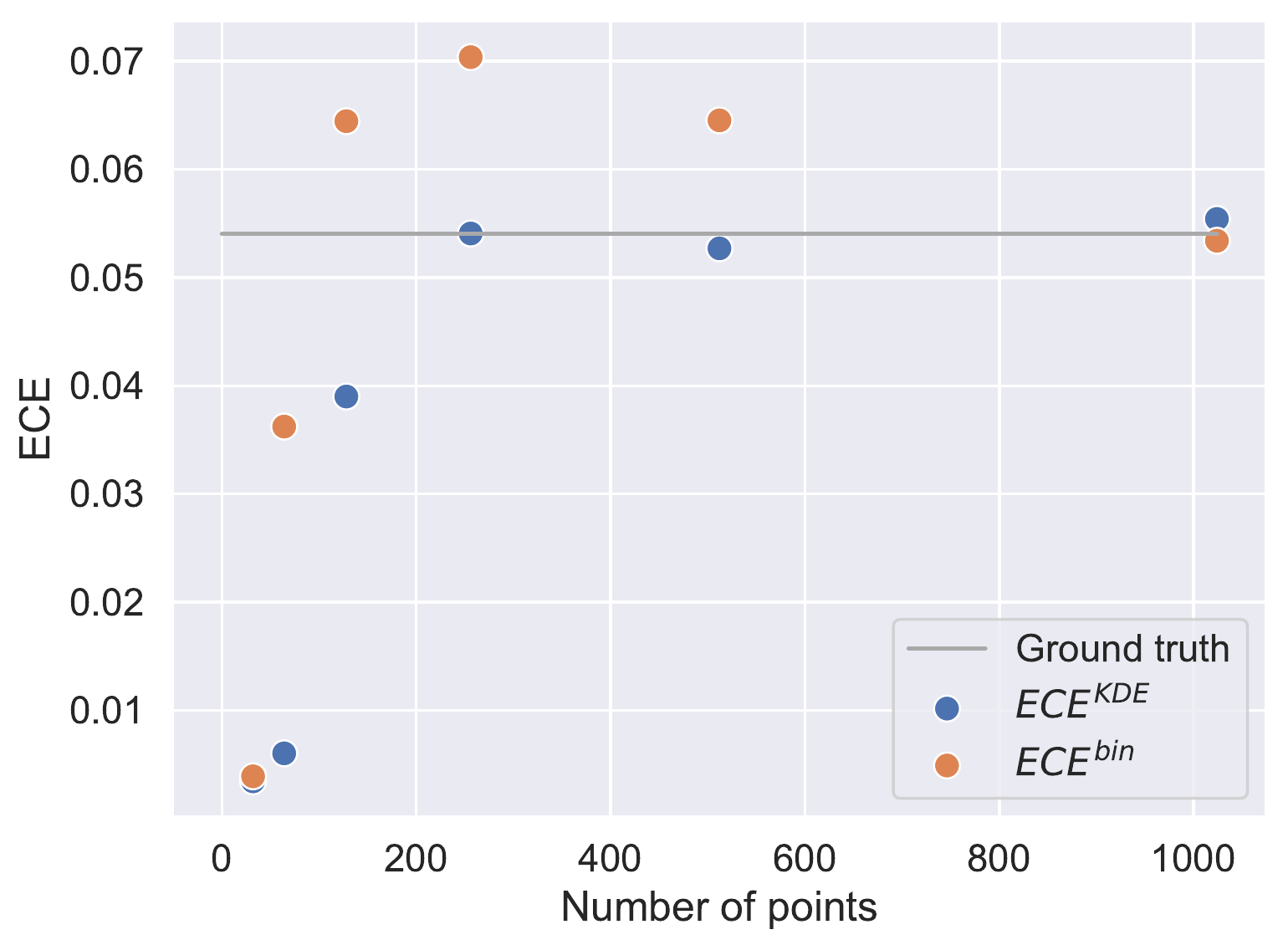}
    } \hfill
    \subfloat[Marginal]{
    \includegraphics[width=.31\linewidth]{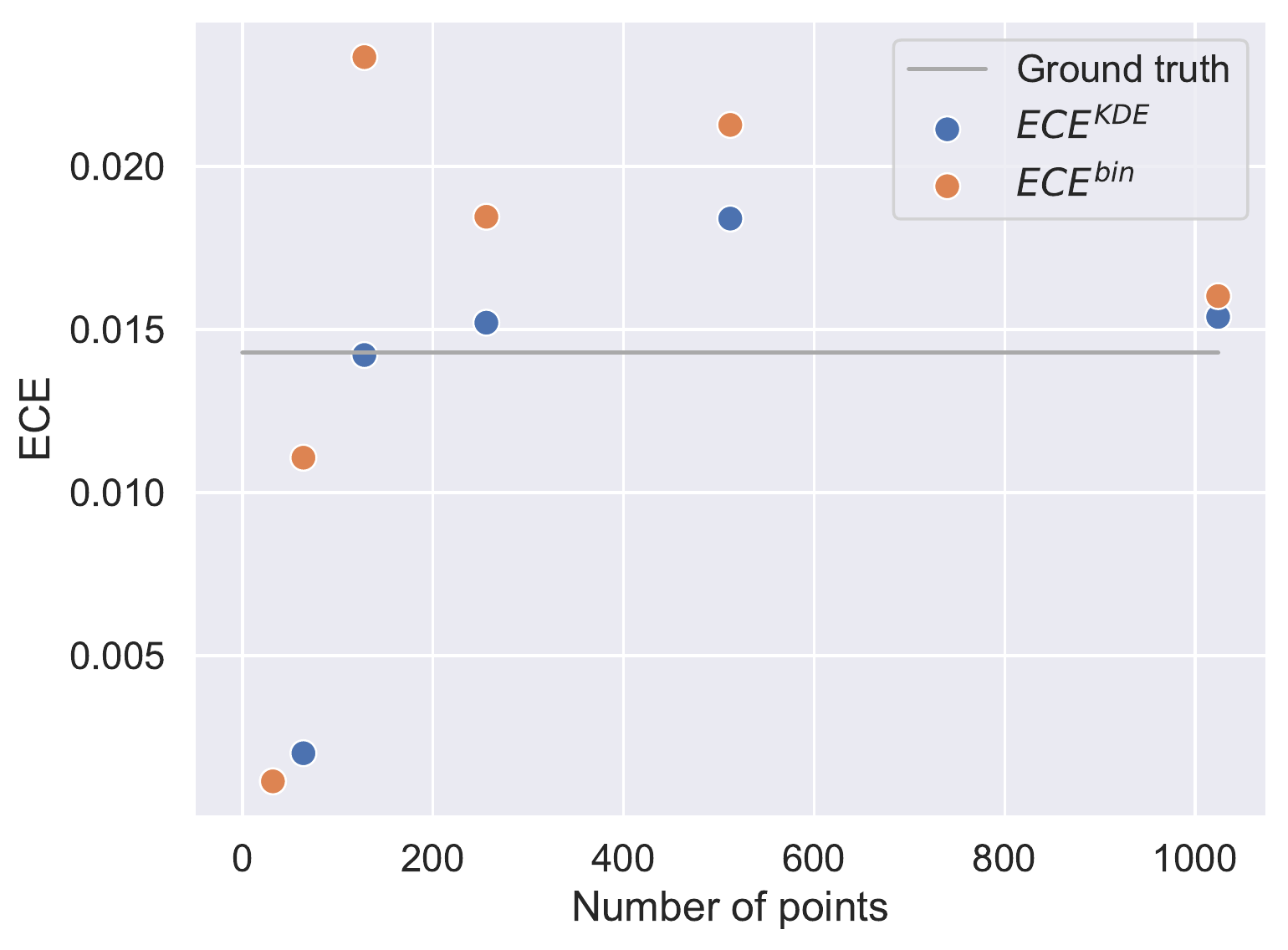}
    } \hfill
    \subfloat[Top-label]{
\includegraphics[width=.31\linewidth]{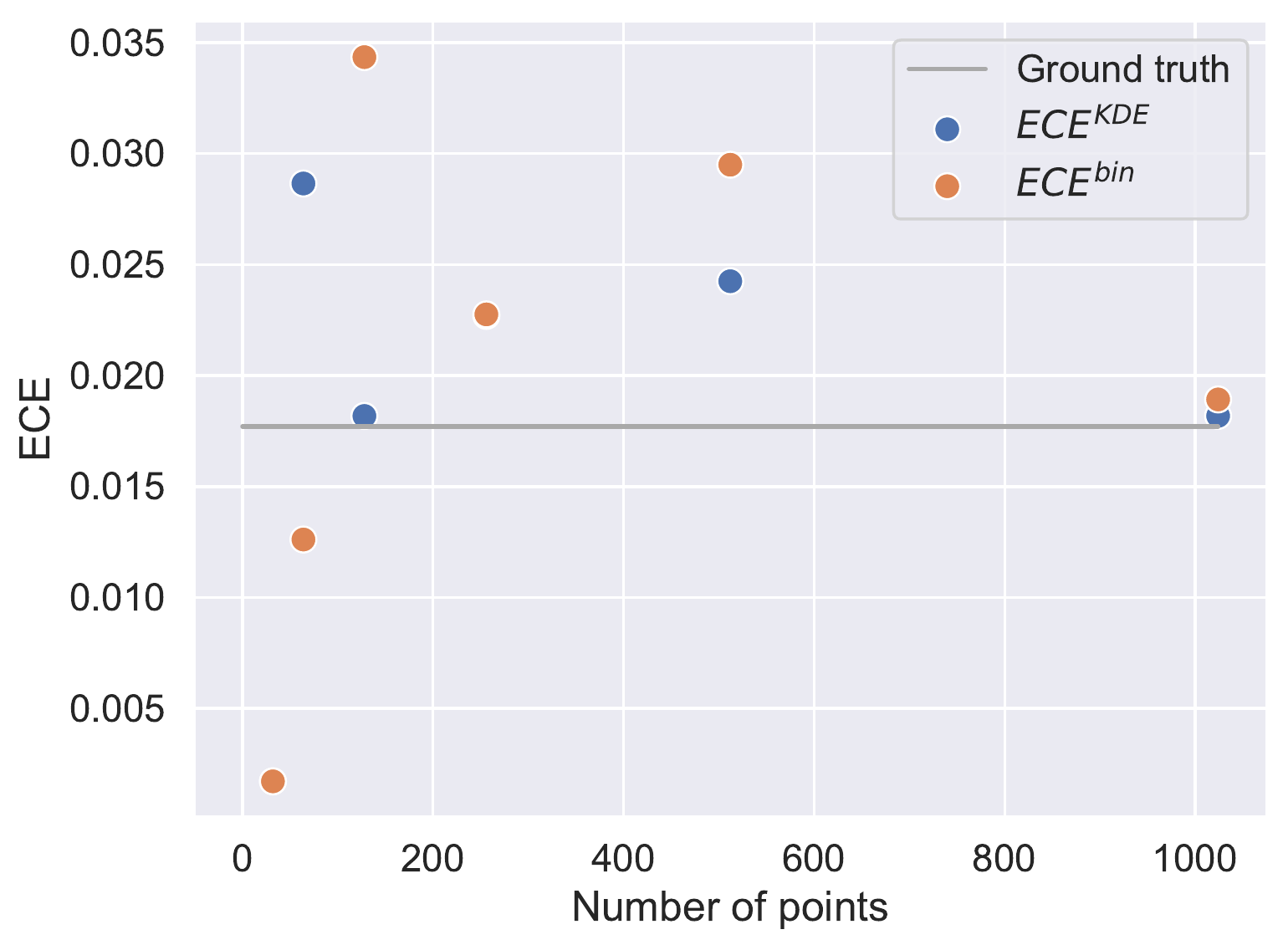}
    } \hfill
    \caption{$ECE^{KDE}$ estimates and their corresponding binned approximations on the three types of calibration for varying number of points used for the estimation. The ground truth is calculated using 3000 probability scores of the test set using $ECE^{KDE}$. Optimal number of bins and bandwidth are chosen with Doane's formula and LOO MLE, respectively. Typical chosen number of bins is 6-11, and common values for the bandwidth are 0.0001 and 0.001.}
    \label{fig:rebuttal_binned_vs_kde_scatter}
\end{figure}

Figure~\ref{fig:rebuttal_binned_vs_kde_scatter_diff} shows the absolute difference between the ground truth and estimated ECE using $ECE^{KDE}$ and a  $ECE^{bin}$ with varying number of points. The results are averaged over 120 ResNet-56 models trained on a subset of three classes from CIFAR-10.  Both estimators are biased and have some variance, and the plot shows that the combination of the two is in the same order of magnitude. The empirical convergence rates (slope of the log-log plot) is given in the legend and is shown to be close to the theoretically expected value of -0.5. We observe that $ECE^{KDE}$ has similar statistical properties in terms of bias and convergence as $ECE^{bin}$.

\begin{figure}[ht!]
    \centering
    \subfloat[Canonical]{
    \includegraphics[width=.31\linewidth]{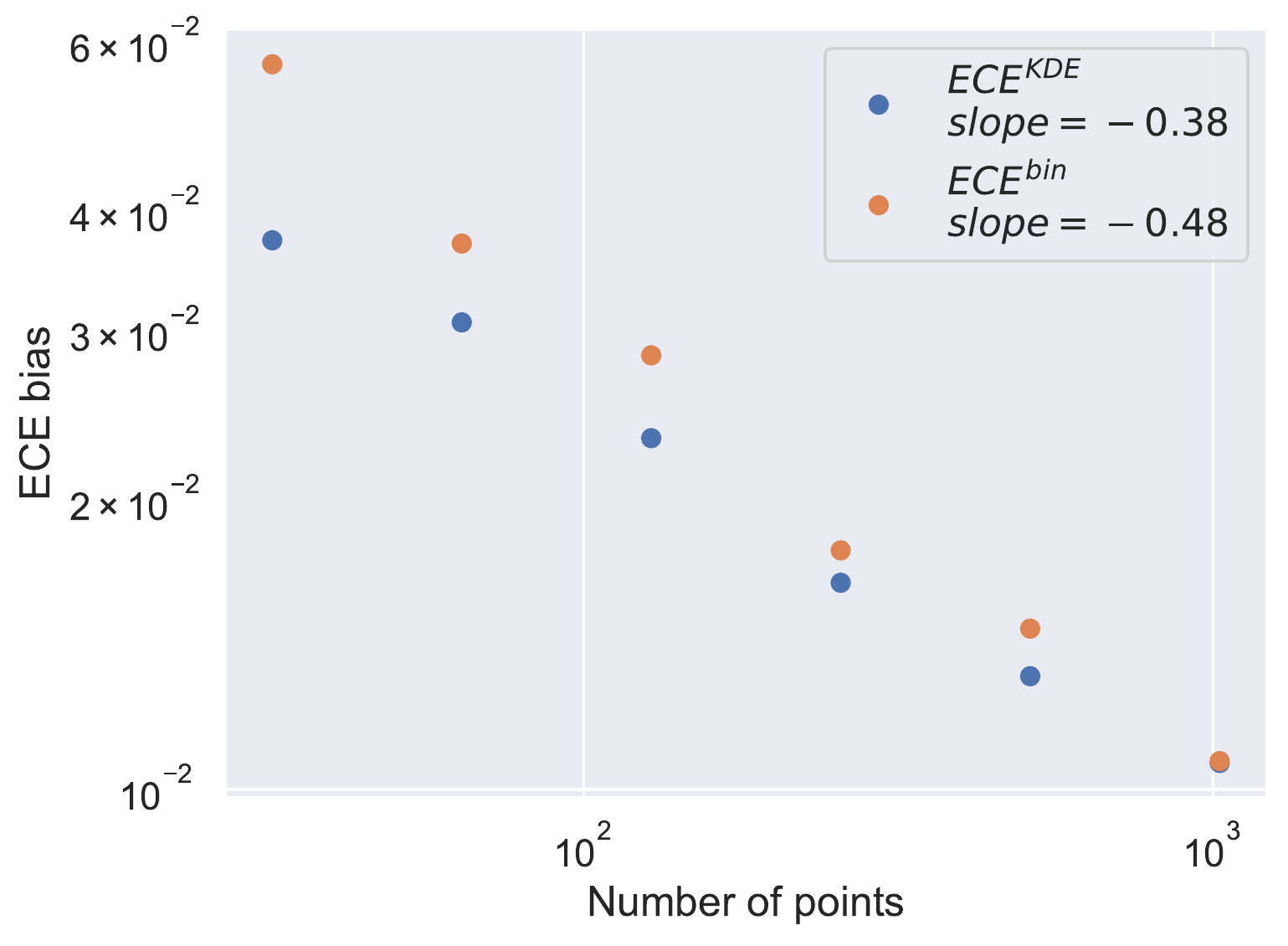}
    } \hfill
    \subfloat[Marginal]{
    \includegraphics[width=.31\linewidth]{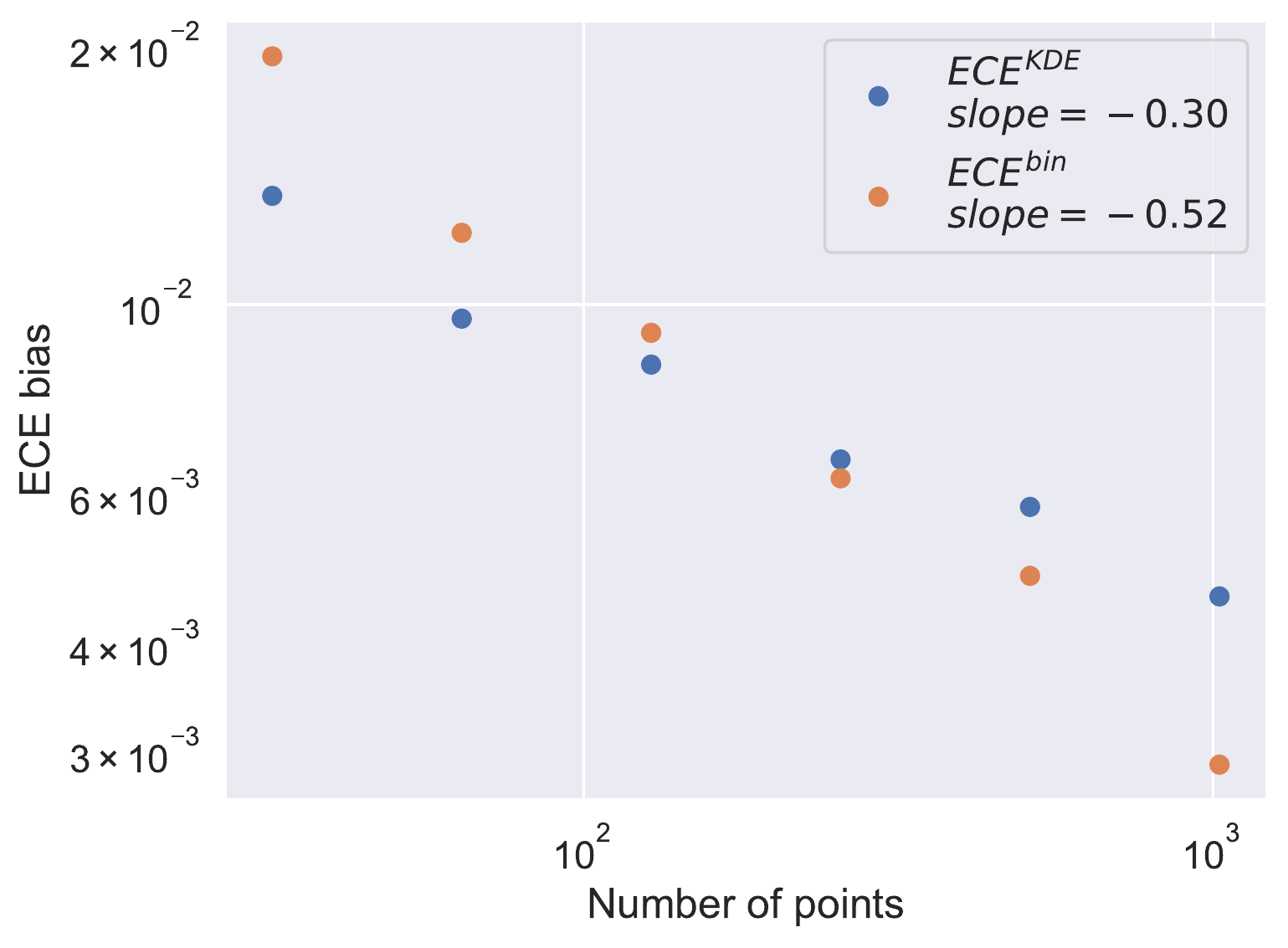}
    } \hfill
    \subfloat[Top-label]{
\includegraphics[width=.31\linewidth]{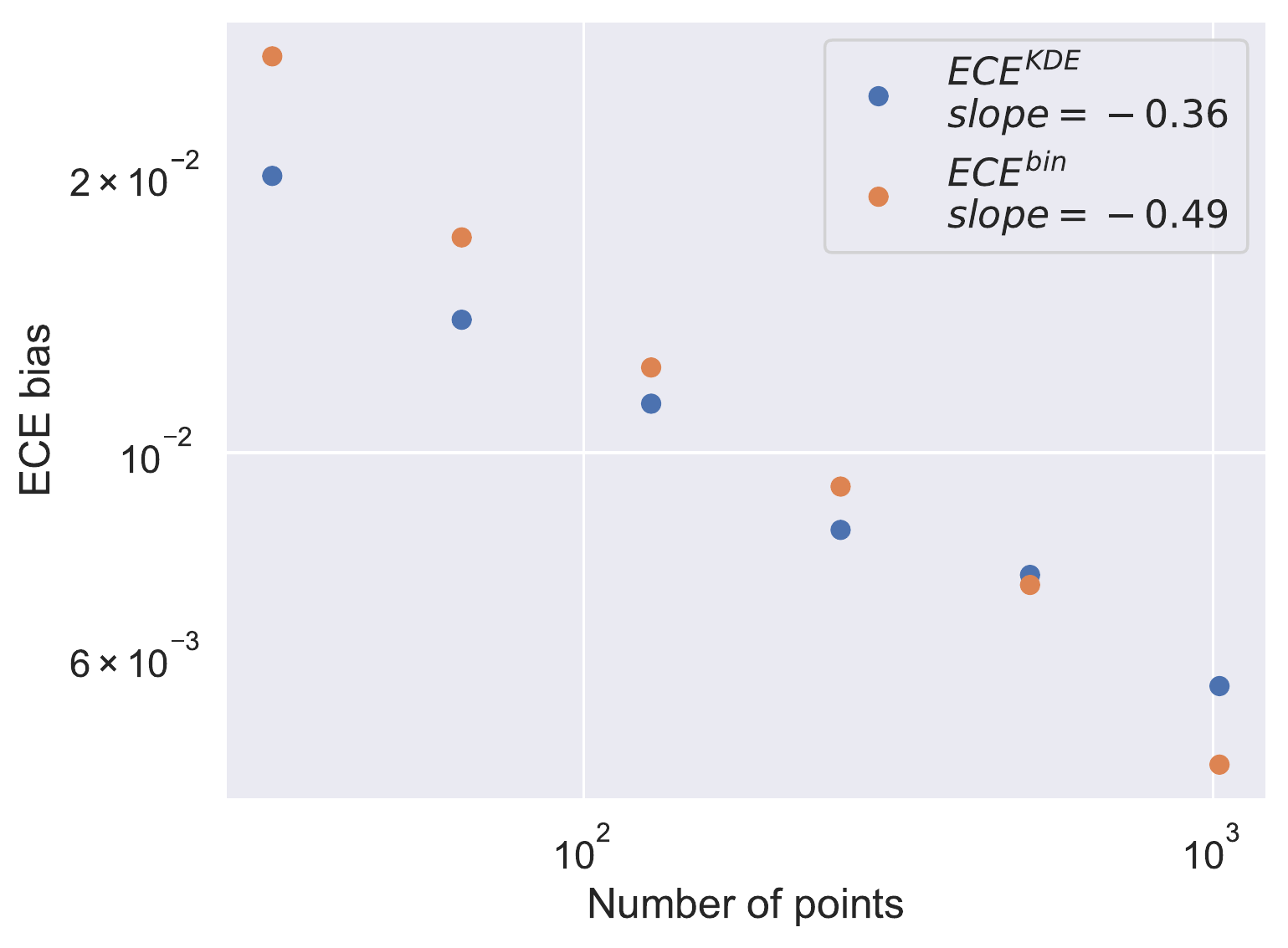}
    } \hfill
    \caption{Absolute difference between ground truth and estimated ECE  for varying number of points used for the estimation. The ground truth is calculated using 3000 probability scores of the test set. Note that the axes are on a log scale. }
    \label{fig:rebuttal_binned_vs_kde_scatter_diff}
\end{figure}

\end{document}